\documentclass[twoside,11pt]{article}

\usepackage{jmlr2e}
\usepackage{graphicx}
\usepackage{amssymb}
\usepackage{epstopdf}
\usepackage{subfigure}
\usepackage{algorithm}
\usepackage{algorithmic}
\usepackage{amsmath}
\usepackage{enumerate}
\usepackage{color}
\usepackage{textcomp}
\usepackage{relsize}

\newcommand{\half}{\ensuremath{ \frac{1}{2} }}
\newcommand{\tr}{\ensuremath{ \mathrm{tr}}}

\newcommand{\lambdamin}{\ensuremath{ \lambda_{\min}}}
\newcommand{\lambdamax}{\ensuremath{ \lambda_{\max}}}
\newcommand{\Hs}{\ensuremath{H}}
\newcommand{\R}{\ensuremath{\mathbb{R}}}
\newcommand{\M}{\ensuremath{\mathcal{M}}}
\newcommand{\E}{\ensuremath{\mathbb{E}}}
\newcommand{\given}{\ensuremath{\, | \, }}
\newcommand{\Rn}{\ensuremath{\mathbb{R}^n}}
\newcommand{\tA}{\ensuremath{\tilde A}}
\newcommand{\tU}{\ensuremath{\tilde U}}
\newcommand{\tV}{\ensuremath{\tilde V}}
\newcommand{\tu}{\ensuremath{\tilde u}}
\newcommand{\tv}{\ensuremath{\tilde v}}
\newcommand{\tLambda}{\ensuremath{\tilde \Lambda}}
\newcommand{\numsamp}{\ensuremath{N}}
\newcommand{\numclass}{\ensuremath{M}}
\newcommand{\numcat}{\ensuremath{Q}}
\newcommand{\icat}{\ensuremath{q}}
\newcommand{\jcat}{\ensuremath{r}}
\newcommand{\class}{\ensuremath{C}}
\newcommand{\classi}{\ensuremath{C_i}}
\newcommand{\classj}{\ensuremath{C_j}}
\newcommand{\X}{\ensuremath{\mathcal{X}}}
\newcommand{\x}{\ensuremath{x}}
\newcommand{\xj}{\ensuremath{x_j}}
\newcommand{\xii}{\ensuremath{x_i}}
\newcommand{\Y}{\ensuremath{Y}}
\newcommand{\y}{\ensuremath{y}}

\newcommand{\yi}{\ensuremath{y_i}}
\newcommand{\Z}{\ensuremath{Z}}
\newcommand{\z}{\ensuremath{z}}
\newcommand{\zj}{\ensuremath{z_j}}
\newcommand{\zi}{\ensuremath{z_i}}
\newcommand{\ii}{\ensuremath{i}}
\newcommand{\jj}{\ensuremath{j}}
\newcommand{\yonemax}{\ensuremath{y_{1,max}}}
\newcommand{\ytwomin}{\ensuremath{y_{2,min}}}
\newcommand{\zonemax}{\ensuremath{z_{1,max}}}
\newcommand{\ztwomin}{\ensuremath{z_{2,min}}}
\newcommand{\zonemin}{\ensuremath{z_{1,min}}}
\newcommand{\ztwomax}{\ensuremath{z_{2,max}}}
\newcommand{\ionemax}{\ensuremath{i_{1,max}}}
\newcommand{\ionemin}{\ensuremath{i_{1,min}}}
\newcommand{\itwomax}{\ensuremath{i_{2,max}}}
\newcommand{\itwomin}{\ensuremath{i_{2,min}}}
\newcommand{\suppone}{\ensuremath{S_1}}
\newcommand{\supptwo}{\ensuremath{S_2}}
\newcommand{\suppovr}{\ensuremath{S}}
\newcommand{\G}{\ensuremath{G}}
\newcommand{\Gw}{\ensuremath{G_w}}
\newcommand{\Gb}{\ensuremath{G_b}}
\newcommand{\wminbyD}{\ensuremath{\overline{w}_{min}}}
\newcommand{\simw}{\ensuremath{\sim_w}}
\newcommand{\simb}{\ensuremath{\sim_b}}
\newcommand{\Ww}{\ensuremath{W_w}}
\newcommand{\Wb}{\ensuremath{W_b}}
\newcommand{\Lw}{\ensuremath{L_w}}
\newcommand{\Lb}{\ensuremath{L_b}}
\newcommand{\Dw}{\ensuremath{D_w}}
\newcommand{\Db}{\ensuremath{D_b}}

\newcommand{\Lcat}{\ensuremath{L^c}}
\newcommand{\Licat}{\ensuremath{L^{c, \icat}}}
\newcommand{\Lwc}{\ensuremath{L_w^c}}
\newcommand{\Lbc}{\ensuremath{L_b^c}}
\newcommand{\Lwicat}{\ensuremath{L_w^{\icat}}}
\newcommand{\Lbicat}{\ensuremath{L_b^{\icat}}}
\newcommand{\Lall}{\ensuremath{ L}}
\newcommand{\PerL}{\ensuremath{ L^{nc}}}
\newcommand{\Ycat}{\ensuremath{Y^c}}
\newcommand{\dw}{\ensuremath{d_w}}
\newcommand{\db}{\ensuremath{d_b}}
\newcommand{\rdeg}{\ensuremath{\beta}}
\newcommand{\rdegmax}{\ensuremath{\beta_{max}}}
\newcommand{\rdegmin}{\ensuremath{\beta_{min}}}
\newcommand{\dwmin}{\ensuremath{d_{w, min}}}

\newcommand{\vol}{\ensuremath{V}}
\newcommand{\volmax}{\ensuremath{V_{max}}}
\newcommand{\volbetmax}{\ensuremath{V^b_{max}}}
\newcommand{\volbetkl}{\ensuremath{V^b_{kl}}}
\newcommand{\diam}{\ensuremath{D}}
\newcommand{\w}{\ensuremath{w}}
\newcommand{\wij}{\ensuremath{w_{ij}}}
\newcommand{\wmink}{\ensuremath{w_{min, k}}}
\newcommand{\wminone}{\ensuremath{w_{min, 1}}}
\newcommand{\wmintwo}{\ensuremath{w_{min, 2}}}

\newcommand{\yv}{\ensuremath{\xi}}

\newcommand{\mar}{\ensuremath{\gamma}}
\newcommand{\marc}{\ensuremath{\gamma^c}}
\newcommand{\maroned}{\ensuremath{\delta}}
\newcommand{\hyp}{\ensuremath{\omega}}
\newcommand{\pmes}{\ensuremath{\nu}}

\newcommand{\Lcon}{\ensuremath{L}}
\newcommand{\Lconf}{\ensuremath{L_\phi}}
\newcommand{\ci}{\ensuremath{c_i}}
\newcommand{\cik}{\ensuremath{c_i^k}}
\newcommand{\cvect}{\ensuremath{c}}
\newcommand{\cbnd}{\ensuremath{\mathcal{C}}}
\newcommand{\nrbf}{\ensuremath{R}}
\newcommand{\Demb}{\ensuremath{D}}
\newcommand{\cover}{\ensuremath{\mathcal{N}}}
\newcommand{\minm}{\ensuremath{\eta_{m,\delta}}}
\newcommand{\dimM}{\ensuremath{D}}

\newcommand{\nbdm}{\ensuremath{A}}
\newcommand{\Knb}{\ensuremath{Q}}
\newcommand{\Bx}{\ensuremath{B_{\delta}(x)}}
\newcommand{\ags}{\ensuremath{a}}
\newcommand{\epsac}{\ensuremath{\varepsilon}}
\newcommand{\delc}{\ensuremath{\zeta}}

\ShortHeadings{Classification with supervised manifold learning}{Vural and Guillemot}
\firstpageno{1}

\begin{document}

\title{A study of the classification of low-dimensional data with supervised manifold learning}

\author{\name Elif Vural \thanks{Most part of the work was performed while the first author was in INRIA.}
       \email velif@metu.edu.tr  \\      
       \addr 
       Middle East Technical University\\
       Ankara, Turkey
       \AND
       \name Christine Guillemot \email christine.guillemot@inria.fr \\
       \addr
       Centre de recherche INRIA Bretagne Atlantique\\
       Rennes, France}

\editor{}

\maketitle

\begin{abstract}
Supervised manifold learning methods learn data representations by preserving the geometric structure of data while enhancing the separation between data samples from different classes. In this work, we propose a theoretical study of supervised manifold learning for classification. We consider nonlinear dimensionality reduction algorithms that yield linearly separable embeddings of training data and present generalization bounds for this type of algorithms. A necessary condition for satisfactory generalization performance is that the embedding allow the construction of a sufficiently regular interpolation function in relation with the separation margin of the embedding. We show that for supervised embeddings satisfying this condition, the classification error decays at an exponential rate with the number of training samples. Finally, we examine the separability of  supervised nonlinear embeddings that aim to preserve the low-dimensional geometric structure of data based on graph representations. The proposed analysis is supported by experiments on several real data sets.


\end{abstract}

\begin{keywords}
Manifold learning, dimensionality reduction, classification, out-of-sample extensions, RBF interpolation
\end{keywords}

\section{Introduction}
\label{sec:intro}

In many data analysis problems, data samples have an intrinsically low-dimensional structure although they reside in a high-dimensional ambient space. The learning of low-dimensional structures in collections of data has been a well studied topic of the last two decades \citep{Tenenbaum00}, \citep{Roweis00}, \citep{Belkin03},  \citep{He04}, \citep{Donoho03}, \citep{Zhang2005}. Following these works, many classification methods have been proposed  in the recent years to apply such manifold learning techniques to learn classifiers that are adapted to the geometric structure of low-dimensional data \citep{Hua12}, \citep{Yang11}, \citep{Zhang12}, \citep{Sugiyama07}, \citep{Raducanu12}. The common approach in such works is to learn a data representation that enhances the between-class separation while preserving the intrinsic low-dimensional structure of data. While many efforts have focused on the practical aspects of learning such supervised embeddings for training data, the generalization performance of these methods as supervised classification algorithms has not been investigated much yet. In this work, we aim to study nonlinear supervised dimensionality reduction methods and present performance bounds based on the properties of the embedding and the interpolation function used for generalizing the embedding.

Several supervised manifold learning methods extend the Laplacian eigenmaps algorithm \citep{Belkin03}, or its linear variant LPP \citep{He04} to the classification problem. The algorithms proposed by \citet{Hua12}, \citet{Yang11}, \citet{Zhang12} provide a supervised extension of the LPP algorithm and learn a linear projection that preserves the proximity of neighboring samples from the same class, while increasing the distance between nearby samples from different classes. The method by \citet{Sugiyama07} proposes an adaptation of the Fisher metric for linear manifold learning, which is in fact shown to be equivalent to the above methods by \citet{Yang11}, \citet{Zhang12}. In \citep{Li13}, \citep{Cui12}, \citep{Wang09}, some other similar Fisher-based linear manifold learning methods are proposed. In \citep{Raducanu12} a method relying on a similar formulation as in \citep{Hua12}, \citep{Yang11}, \citep{Zhang12} is presented, which, however, learns a nonlinear embedding. The main advantage of linear dimensionality reduction methods over nonlinear ones is that the generalization of the learnt embedding to novel (initially unavailable) samples is straightforward. However, nonlinear manifold learning algorithms are more flexible as the possible data representations they can learn belong to a wider family of functions, e.g., one can always find a nonlinear embedding to make training samples from different classes linearly separable. On the other hand, when a nonlinear embedding is used, one must also determine a suitable interpolation function to generalize the embedding to new samples, and the choice of the interpolator is critical for the classification performance.




The common effort in all of these supervised dimensionality reduction methods is to learn an embedding that increases the separation between different classes, while preserving the geometric structure of data. It is interesting to note that supervised manifold learning methods achieve separability by reducing the dimension of data, while kernel methods in traditional classifiers achieve this by increasing the dimension of data. Meanwhile, making training data linearly separable in supervised manifold learning does not mean much only by itself. Assuming that the data are sampled from a continuous distribution (hence two samples coincide with 0 probability), it is almost always possible to separate a discrete set of samples from different classes with a nonlinear embedding, e.g., even with a simple embedding such as the one mapping each sample to a vector encoding its class label. What actually matters is how the embedding generalizes to test data, i.e., where the test samples will be mapped to in the low-dimensional domain of embedding and how well the performance will be. The generalization for test data is straightforward for kernel methods, it is determined by the underlying main algorithm. However, in nonlinear supervised manifold learning, this question has rather been overlooked so far. In this work we aim to fill this gap and look into the generalization capabilities of supervised manifold learning algorithms. We study the conditions that must be satisfied by the embedding of the training samples and the interpolation function for satisfactory generalization of the classifier. We then examine the rates of convergence of supervised manifold learning algorithms that satisfy these conditions.

In Section \ref{sec:class_anly_supml}, we consider arbitrary supervised manifold learning algorithms that compute a linearly separable embedding of training samples. We study the generalization capability of such algorithms for two types of out-of-sample interpolation functions. We first consider arbitrary interpolation functions that are Lipschitz-continuous on the support of each class, and then focus on out-of-sample extensions with radial basis function (RBF) kernels, which is a popular family of  interpolation functions. For both types of interpolators, we derive  conditions that must be satisfied by the embedding of the training samples and the regularity of the interpolation function that generalizes the embedding to test samples, when a nearest neighbor or linear classifier is used in the low-dimensional domain of embedding. These conditions enforce the Lipschitz constant of the interpolator to be sufficiently small, in comparison with the separation margin between training samples from different classes in the low-dimensional domain of embedding. The practical value of these results resides in their implications about what must really be taken into account when designing a supervised dimensionality reduction algorithm: Achieving a good separation margin does not suffice by itself; the geometric structure must also be preserved so as to ensure that a sufficiently regular interpolator can be found to generalize the embedding to the whole ambient space. We then particularly consider Gaussian RBF kernels and show the existence of an optimal value for the kernel scale by studying the condition in our main result that links the separation with the Lipschitz constant of the kernel.

Our results in Section \ref{sec:class_anly_supml} also provide bounds on the rate of convergence of the classification error of supervised embeddings. We show that the misclassification error probability decays at an exponential rate with the number of samples, provided that the interpolation function is sufficiently regular with respect to the separation margin of the embedding. These  convergence rates are higher than those reported in previous results on RBF networks \citep{NiyogiG96}, \citep{LinLRX14}, \citep{AguirreKB02}, and regularized least-squares regression algorithms \citep{CaponnettoV07}, \citep{SteinwartHS09}. The essential difference between our results and such previous works is that those assume a general setting and do not focus on a particular data model, whereas our results are rather relevant to settings where the support of each class admits some certain structure, so as to allow the existence of an interpolator that is sufficiently regular on the support of each class. Moreover, in contrast with these previous works, our bounds are independent of the ambient space dimension and vary only with the intrinsic dimensions of the class supports as they characterize the error in terms of the covering numbers of the supports. 

The results in Section \ref{sec:class_anly_supml} assume an embedding that makes training samples from different classes linearly separable. Even if most nonlinear dimensionality reduction methods are observed to yield separable embeddings in practice, we aim to verify this theoretically in Section \ref{sec:sep_analysis}. In particular, we focus on the nonlinear version of the supervised Laplacian eigenmaps embeddings \citep{Raducanu12}, \citep{Hua12}, \citep{Yang11}, \citep{Zhang12}. Supervised Laplacian eigenmaps methods embed the data with the eigenvectors of the linear combination of two graph Laplacian matrices that encode the links between neighboring samples from the same class and different classes. In such a data representation, the coordinates of neighboring data samples change slowly within the same class and rapidly across different classes. We study the conditions for the linear separability of these embeddings and characterize their separation margin in terms of some graph and algorithm parameters.

In Section \ref{sec:exp_results}, we evaluate our results with experiments on several object and face data sets. We study the implications of the condition derived in Section \ref{sec:class_anly_supml} on the separability margin - interpolator regularity tradeoff. The experimental comparison of several supervised dimensionality reduction algorithms shows that this compromise between the separation and interpolator regularity can indeed be related to the practical classification performance of a supervised manifold learning algorithm. This suggests that, one can possibly improve the accuracy of supervised dimensionality reduction algorithms by considering more carefully the generalization capability of the embedding during the learning. We then study the variation of the classification performance with parameters such as the sample size, the RBF kernel scale, and the dimension of the embedding, in view of the generalization bounds presented in Section  \ref{sec:class_anly_supml}. Finally, we conclude in Section \ref{sec:conc}.

\section{Performance bounds for supervised manifold learning methods}
\label{sec:class_anly_supml}


\subsection{Notation and Problem Formulation}
\label{ssec:oos_notation}

Consider a setting with $M$ data classes where the samples of each class $m \in \{ 1, \dots, M \}$ are drawn from a probability measure $\pmes_m$ in a Hilbert space $\Hs$ such that $\pmes_m$ has a bounded support $\M_m \subset \Hs$. Let $\X=\{ \x_i \}_{i=1}^{\numsamp} \subset \Hs $ be a set of $\numsamp$ training samples such that each $\x_i$ is drawn from one of the probability measures $\pmes_m$, and the samples drawn from each $\pmes_m$ are independent and identically distributed. We denote the class label of $\xii$ by $\classi \in \{1, 2, \dots, M\}$.

Let $\Y=\{ \y_i \}_{i=1}^{\numsamp} \subset \R^d$ be a $d$-dimensional embedding of $\X$, where each $\y_i $ corresponds to $\x_i$. We consider supervised embeddings such that $\Y$ is linearly separable. Linear separability is defined as follows:

\begin{definition}
\label{def:lin_sepbilty}
The data representation $\Y$ is linearly separable with a margin of $\mar>0$, if for any two classes $k, l \in \{1, 2, \dots, M \}$, there exists a separating hyperplane defined by $\hyp_{kl} \in \R^d$, $\| \hyp_{kl} \|=1$ and $b_{kl} \in \R$ such that 
\begin{equation}
\label{eq:defn_sepbilty}
\begin{split}
\hyp_{kl}^T \, \yi + b_{kl} &\geq \mar/2 \quad \,\,\,\, \text{     if  } \classi=k\\
\hyp_{kl}^T \, \yi + b_{kl} &\leq - \mar/2 \quad \text{  if  } \classi=l.
\end{split}
\end{equation}
\end{definition}

\begin{figure}[t]
 \centering
  \includegraphics[height=5cm]{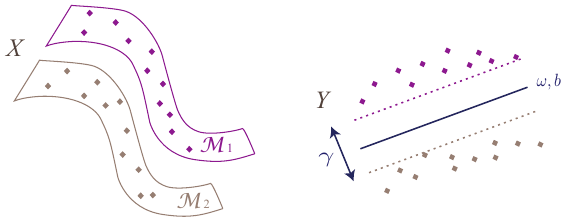}
  \caption{Illustration of a linearly separable embedding. Data in $X$ are sampled from two different classes with supports $\M_1$, $\M_2$. The samples $X$ are mapped to the coordinates $Y$ with a low-dimensional embedding, where the two classes become linearly separable with margin $\mar$ with the hyperplane given by $\hyp$, $b$.}
  \label{fig:illus_embedding}
\end{figure}

The above definition of separability implies the following. For any given class $m$, there exists a set of hyperplanes $\{ \hyp_{mk} \}_{k\neq m} \subset \R^d$, $\| \hyp_{mk} \|=1$, and a set of real numbers $\{ b_{mk}  \}_{k \neq m} \subset \R$ that separate class $m$ from other classes, such that for all $\yi$ of class $\classi=m$
\begin{equation}
\label{eq:hyp_m_pos}
\hyp_{mk}^T \, \yi + b_{mk}  > \mar/2,  \quad  \forall k \neq m
\end{equation}
and for all $\yi$ of class $\classi \neq m$, there exists a $k$ such that
\begin{equation}
\label{eq:hyp_m_neg}
\hyp_{mk}^T \, \yi + b_{mk}  < - \mar/2. 
\end{equation}
These hyperplanes are obtained by setting $\hyp_{km}=-\hyp_{mk}$, $b_{km}= - b_{mk}$.


Figure \ref{fig:illus_embedding} shows an illustration of a linearly separable embedding of data samples from two classes. Manifold learning methods typically compute a low-dimensional embedding $\Y$ of training data $\X$ in a pointwise manner, i.e., the coordinates $\y_i$ are computed only for the initially available training samples $\x_i$. However, in a classification problem, in order to estimate the class label of a new data sample $\x$ of unknown class, $\x$ needs to be mapped to the low-dimensional domain of embedding as well. The construction of a function $f: \Hs \rightarrow \R^d$ that generalizes the learnt embedding to the whole space is known as the out-of-sample generalization problem. Smooth functions are commonly used for out-of-sample interpolation, e.g.~as in \citep{QiaoZWZ13}, \citep{PeherstorferPB11}.

Now let $\x$ be a test sample drawn from the probability measure $\pmes_m$, hence, the true class label of $\x$ is $m$. In our study, we consider two basic classification schemes in the domain of embedding:

\textit{Linear classifier. }  The embeddings of the training samples are used to compute the separating hyperplanes, i.e., the classifier parameters $\{ \hyp_{mk} \}$ and $\{ b_{mk}  \}$. Then, mapping $x$ to the low-dimensional domain as $f(x) \in \R^d$, the class label of $\x$ is estimated as $\hat {\class} (x) = l$ if there exists $l \in \{ 1, \dots, M \}$ such that
\begin{equation}
\label{eq:lin_classifier}
\hyp_{lk}^T \, f(x) + b_{lk}  > 0,  \quad  \forall k \in \{1, \dots, M \} \setminus \{ l \}.
\end{equation}
Note that the existence of such an $l$ is not guaranteed in general for any $\x$, but for a given $\x$ there cannot be more than one $l$ satisfying the above condition. Then $\x$ is classified correctly if the estimated class label agrees with the true class label, i.e., $ \hat {\class} (x) = l = m$.

\textit{Nearest neighbor classification.} The test sample $x$ is assigned the class label of the closest training point in the domain of embedding, i.e.,  $\hat {\class} (x) = C_{i'}$, where
\[
i' = \arg \min_{i=1, \dots, \numsamp} \| y_i - f(x) \|
\]

In the rest of this section, we study the generalization performance of supervised dimensionality reduction methods. We first consider in Section \ref{ssec:oos_reg_func} interpolation functions that vary regularly on each class support and we search for a lower bound on the probability of correctly classifying a new data sample in terms of the regularity of $f$, the separation of the embedding, and the sampling density. Then in Section \ref{ssec:oos_rbf}, we study the classification performance for a particular type of interpolation functions, namely RBF interpolators, which is one of the most popular ones  \citep{PeherstorferPB11}, \citep{ChinS08}. We focus particularly on Gaussian RBF interpolators in Section \ref{ssec:scale_optim} and derive some results regarding the existence of an optimal kernel scale parameter. Lastly, we discuss our results in comparison with previous literature in Section \ref{ssec:disc_prev_res}.

In the results in Sections \ref{ssec:oos_reg_func}-\ref{ssec:scale_optim}, we keep a generic formulation and simply treat the supports $\{ \M_m \}$ as arbitrary bounded subsets of $\Hs$, each of which represents a different data class. Nevertheless, from the perspective of manifold learning, our results are of interest especially when the data is assumed to have an underlying low-dimensional structure. In Section \ref{ssec:disc_prev_res}, we study the implications of our results for the setting where $\M_m $ are low-dimensional manifolds. We then examine how the proposed bounds vary in relation to the intrinsic dimensions of $\{ \M_m  \}$.


\subsection{Out-of-sample interpolation with regular functions} 
\label{ssec:oos_reg_func}
Let $f: \Hs \rightarrow \R^d$ be an out-of-sample interpolation function such that $f(\x_i)=\y_i$ for each training sample $\x_i$, $i=1, \dots, \numsamp$. Assume that $f$ is Lipschitz continuous with constant $\Lcon>0$ when restricted to any one of the supports $\M_m$; i.e., for any $m \in \{1, \dots, M  \}$ and any $u, v \in \M_m$
\begin{equation*}
\| f(u) - f(v) \| \leq \Lcon \, \| u - v \|
\end{equation*}
where $\| \cdot \|$ denotes above the $\ell_2$-norm if the argument is in $\R^d$, and the norm induced from the inner product in $\Hs$ if the argument is in $\Hs$. 

We will find a relation between the classification accuracy and the number of training samples via the covering number of the supports $\M_m$. Let $B_\epsilon(x) \subset \Hs$ denote an open ball of radius $\epsilon$ around $x$
\begin{equation*}
B_\epsilon(x)=\{ u \in \Hs: \| x-u \| < \epsilon  \}.
\end{equation*}
The covering number $\cover(\epsilon, A)$ of a set $A \subset \Hs$ is defined as the smallest number of open balls $B_\epsilon$ of radius $\epsilon$ whose union contains $A$ \citep{KulkarniP95}   
\begin{equation*}
\cover(\epsilon, A) = \inf \{ k: \exists \, u_1, \dots, u_k \in \Hs \, \text{s.t.} \, A \subset \bigcup_{i=1}^k B_\epsilon(u_i)   \}.
\end{equation*}
We assume that the supports $\M_m$ are totally bounded, i.e., $\M_m$ has a finite covering number $\cover(\epsilon, \M_m)$ for any $\epsilon > 0$. 

We state below a lower bound for the probability of correctly classifying a sample $\x$ drawn from $\pmes_m$, in terms of the number of training samples drawn from $\pmes_m$, the separation of the embedding and the regularity of $f$. 
\begin{theorem}
\label{thm:emb_smooth_intp}
For some $\epsilon$ with $0 < \epsilon \leq \mar/(2\Lcon) $, let the training set $\X$ contain at least $\numsamp_m$ samples drawn i.i.d.~according to a probability measure $\pmes_m$ such that 
\begin{equation*}
\numsamp_m \geq \cover(\epsilon/2, \M_m).
\end{equation*}
Let $\Y$ be an embedding of the training samples $\X$ that is linearly separable with margin larger than $\mar$, and let $f$ be an interpolation function that is Lipschitz continuous with constant $\Lcon$ on the support $\M_m$. Then the probability of correctly classifying a test sample $\x$ drawn from $\pmes_m$ independently from the training samples with the linear classifier  \eqref{eq:lin_classifier} is lower bounded as
\begin{equation*}
P\left( \hat{\class}(\x) = m \right) \geq  1 - \frac{\cover(\epsilon/2, \M_m)}{2 \numsamp_m}.
\end{equation*}
\end{theorem}

The proof of the theorem is given in Appendix \ref{pf:thm:emb_smooth_intp}. Theorem \ref{thm:emb_smooth_intp} establishes a link between the classification performance and the separation of the embedding of the training samples. In particular, due to the condition $\epsilon  \leq \mar/(2 \Lcon)$, the increase in the separation $\mar$ allows a larger value for $\epsilon$, provided that the interpolator regularity is not affected much. This reduces the covering number $\cover(\epsilon/2, \M_m)$ in return and increases the probability of correct classification. Similarly, from the condition $\epsilon  \leq \mar/(2 \Lcon)$, one can also observe that at a given separation $\mar$, a smaller Lipschitz constant $\Lcon$ for the interpolation function allows the parameter $\epsilon$ to take a larger value. This reduces the covering number $\cover(\epsilon/2, \M_m)$ and therefore increases the correct classification probability. Thus, choosing a more regular interpolator at a given separation helps improve the classification performance. If the $\epsilon$ parameter is fixed, the Lipschitz constant of the interpolator is allowed to increase only proportionally to the separation margin. The condition that the interpolator must be sufficiently regular in comparison with the separation suggests that increasing the separation too much at the cost of impairing the interpolator regularity may degrade the classifier performance.  In the case that the supports $\M_m$ are low-dimensional manifolds, the covering number $\cover(\epsilon/2, \M_m)$ increases at a geometric rate with the intrinsic dimension $\dimM$ of the manifold, since a $\dimM$-dimensional manifold is locally homeomorphic to $\mathbb{R}^\dimM$. Therefore, from the condition on the number of samples, $\numsamp_m$ should increase at a geometric rate with $\dimM$.

In Theorem \ref{thm:emb_smooth_intp} the probability of misclassification decreases with the number $\numsamp_m$ of training samples at a rate of $O(\numsamp_m^{-1})$. In the rest of this section, we show that it is in fact possible to obtain an exponential convergence rate with linear and NN-classifiers under certain assumptions. We first present the following lemma.
\begin{lemma}
\label{lem:fdev_neigh_genf}
Let $\X=\{ \x_i \}_{i=1}^{\numsamp} \subset \Hs $ be a set of training samples such that each $\x_i$ is drawn i.i.d.~from one of the probability measures $\{\pmes_m \}_{m=1}^M$. Let $\x$ be a test sample randomly drawn according to the probability measure $\pmes_m$ of class $m$. Let 
\begin{equation}
\label{eq:defn_NNx_lem_genf}
\nbdm = \{ \x_i \in \X : \x_i \in B_\delta(x), \x_i \sim \pmes_m \}
\end{equation}
be the set of samples in $\X$ that are in a $\delta$-neighborhood of $\x$ and  also drawn from the measure $\pmes_m$. Assume that $\nbdm$ contains $|  \nbdm|=\Knb$ samples. Then 
\begin{equation}
\label{eq:lemma_bnd_dev_genf}
P\left( \| f(x) -  \frac{1}{\Knb} \sum_{\x_j \in \nbdm} f(\x_j)  \|  \leq  \Lcon \delta + \sqrt{d}  \epsilon \right) 
\geq
1- 2 d  \exp  \left( - \frac{  \Knb \, \epsilon^2 }{ 2 \Lcon^2 \delta^2} \right).
\end{equation}

\end{lemma}

Lemma \ref{lem:fdev_neigh_genf} is proved in Appendix \ref{pf:lem_fdev_neigh_genf}. The inequality in \eqref{eq:lemma_bnd_dev_genf} shows that as the number $\Knb$ of training samples falling in a neighborhood of a test point $\x$ increases, the probability of the deviation of $f(x)$ from its average within the neighborhood decreases. The parameter $\epsilon$ captures the relation between the amount and the probability of deviation.

When studying the classification accuracy in the main result below, we will use the following generalized definition of the linear separation. 
\begin{definition}
Let $\Y$ be a linearly separable embedding with margin $\mar$ such that each pair $(k,l)$ of classes are separated with the hyperplanes given by $\hyp_{kl}$, $b_{kl}$ as defined in Definition \ref{def:lin_sepbilty}. We say that the linear classifier given by  $\{ \hyp_{kl} \}$, $\{ b_{kl} \}$ has a $Q$-mean separability margin of $\mar_Q>0$ if any choice of $Q$ samples $\{y_{k,i}\}_{i=1}^Q \subset \Y$ from class $k$ and $Q$ samples $\{y_{l,i}\}_{i=1}^Q \subset \Y$ from class $l$, $l \neq k$, satisfies
\begin{equation}
\begin{split}
\hyp_{kl}^T \, \left( \frac{1}{Q} \sum_{i=1}^Q y_{k,i} \right) + b_{kl} &\geq \mar_Q/2 \\
\hyp_{kl}^T \, \left( \frac{1}{Q} \sum_{i=1}^Q y_{l,i} \right) + b_{kl} &\leq - \mar_Q/2 .
\end{split}
\end{equation}
\end{definition}
The above definition of separability is more flexible than the one in Definition \ref{def:lin_sepbilty}. Clearly, an embedding that is linearly separable with margin $\mar$ has a $Q$-mean separability margin of $\mar_Q \geq \mar$ for any $Q$. As in the previous section, we consider that the test sample $\x$ is classified with the linear classifier \eqref{eq:lin_classifier} in the low-dimensional domain, defined with respect to the set of hyperplanes given by $\{ \hyp_{mk} \} $ and $\{ b_{mk}  \}$ as in \eqref{eq:hyp_m_pos} and \eqref{eq:hyp_m_neg}.

In the following result, we show that an exponential convergence rate can be obtained with linear classifiers in supervised manifold learning. We define beforehand a parameter depending on $\delta$, which gives the smallest possible measure of the $\delta$-neighborhood $\Bx$ of a point $\x$ in support $\M_m$.
\begin{equation*}
\minm := \inf_{\x \in \M_m} \pmes_m(\Bx).
\end{equation*}
\begin{theorem}
\label{thm:linclassif_genf}
Let $\X =\{ \x_i \}_{i=1}^{\numsamp} \subset \Hs$ be a set of training samples such that each $\x_i$ is drawn i.i.d.~from one of the probability measures $\{\pmes_m \}_{m=1}^M$. Let $\Y$ be an embedding of $\X$ in $\R^d$ that is linearly separable with a $Q$-mean separability margin larger than $\mar_Q$. For a given $\epsilon>0$ and $\delta>0$, let $f$ be a Lipschitz-continuous interpolator such that 
\begin{equation}
\label{eq:cond_sep_genf_lin}
\Lcon \delta +  \sqrt{d} \epsilon  \leq  \frac{\mar_Q}{ 2}.
\end{equation}
Consider a test sample $\x$ randomly drawn according to the probability measure $\pmes_m$ of class $m$. If $\X$ contains at least $\numsamp_m$ training samples drawn i.i.d.~from $\pmes_m$ such that 
\begin{equation*}
\numsamp_m > \frac{Q}{\minm}
\end{equation*}
then the probability of correctly classifying $\x$ with the linear classifier given in \eqref{eq:lin_classifier} is lower bounded as
\begin{equation}
\label{eq:prob_lb_genf_lin}
P\left( \hat{\class}(\x) = m \right) \geq 
1 - \exp \left( - \frac{2 \, (\numsamp_m \, \minm - Q)^2 }{\numsamp_m} \right)
 - 2 d \exp \left( - \frac{  \Knb \, \epsilon^2 }{ 2 \Lcon^2 \delta^2} \right).
\end{equation}
\end{theorem}

Theorem \ref{thm:linclassif_genf} is proved in Appendix \ref{pf:thm_linclassif_genf}. The theorem shows how the classification accuracy is influenced by the separation of the classes in the embedding, the smoothness of the out-of-sample interpolant, and the number of training samples drawn from the density of each class. 
The condition in \eqref{eq:cond_sep_genf_lin} points to the tradeoff between the separation and the regularity of the interpolation function. As the Lipschitz constant $\Lcon$ of the interpolation function $f$ increases, $f$ becomes less ``regular'', and a higher separation $\mar_Q$ is needed to meet the condition. This is coherent with the expectation that, when $f$ becomes irregular, the classifier becomes more sensitive to the perturbations of the data, e.g., due to noise. The requirement of a higher separation is then for ensuring a larger margin in the linear classifier, which compensates for the irregularity of $f$. From  \eqref{eq:cond_sep_genf_lin}, it is also observed that the separation should increase with the dimension $d$ as well, and also with $\epsilon$, whose  increase improves the confidence of the bound \eqref{eq:prob_lb_genf_lin}. Note that the condition in \eqref{eq:cond_sep_genf_lin} implies also the following: When computing an embedding, it is not advisable to increase the separation of training data unconditionally. In particular, increasing the separation too much may violate the preservation of the geometry and yield an irregular interpolator. Hence, when designing a supervised dimensionality reduction algorithm, one must pay attention to the regularity of the resulting interpolator as much as the enhancement of the separation margin.

Next, we discuss the roles of the parameters $Q$ and $\delta$. The term $\exp ( - \Knb \, \epsilon^2 / ( 2 \Lcon^2 \delta^2) )$ in the correct classification probability bound \eqref{eq:prob_lb_genf_lin} shows that, for fixed $\delta$, the confidence increases with the value of $Q$. Meanwhile, due to the numerator of the term $\exp ( -2 \, (\numsamp_m \, \minm - Q)^2 / \numsamp_m )$, for a high confidence, the number of samples $\numsamp_m$ should also be relatively big with respect to $Q$ to have a high overall confidence. Similarly, at fixed $Q$, $\delta$ should be made smaller to increase the confidence due to the term $\exp ( - (  \Knb \, \epsilon^2 ) / ( 2 \Lcon^2 \delta^2 )  )$, which then reduces the parameter $\minm$ and eventually requires the number of samples $\numsamp_m$ to take a sufficiently large value in order to make the term $\exp ( -2 \, (\numsamp_m \, \minm - Q)^2 / \numsamp_m )$ small and have a high confidence. Therefore, these two parameters $Q$ and $\delta$ behave in a similar way, and determine the relation between the number of samples and the correct classification probability, i.e., they indicate how large $\numsamp_m $ should be in order to have a certain confidence of correct classification.

Theorem \ref{thm:linclassif_genf} studies the setting where the class labels are estimated with a linear classifier in the domain of embedding. We also provide another result below that analyses the performance when a nearest-neighbor classifier is used in the domain of embedding.

\begin{theorem}
\label{thm:error_genf_nnclass}
Let $\X =\{ \x_i \}_{i=1}^{\numsamp} \subset \Hs$ be a set of training samples such that each $\x_i$ is drawn i.i.d.~from one of the probability measures $\{\pmes_m \}_{m=1}^M$. Let $\Y$ be an embedding of $\X$ in $\R^d$ such that 
\begin{equation*}
\begin{split}
\| \y_i - \y_j \| &< \Demb_\delta, \text{ if } \| \x_i - \x_j  \| \leq \delta \text{ and } \classi = \classj \\
\| \y_i - \y_j \| &> \mar,  \  \text{ if }  \classi \neq \classj,
\end{split}
\end{equation*}
hence, nearby samples from the same class are mapped to nearby points, and samples from different classes are separated by a distance of at least $\mar$ in the embedding.

For given $\epsilon>0$ and $\delta>0$, let $f$ be a Lipschitz-continuous interpolation function such that 
\begin{equation}
\label{eq:cond_sep_genf_nn}
\Lcon \delta + \sqrt{d} \epsilon +   \Demb_{2 \delta} \leq \frac{ \mar}{ 2}.
\end{equation}

Consider a test sample $\x$ randomly drawn according to the probability measure $\pmes_m$ of class $m$. If $\X$ contains at least $\numsamp_m$ training samples drawn i.i.d.~from $\pmes_m$ such that 
\begin{equation*}
\numsamp_m > \frac{Q}{\minm}
\end{equation*}
then the probability of correctly classifying $\x$ with nearest-neighbor classification in  $\R^d$ is lower bounded as
\begin{equation}
P\left( \hat{\class}(\x) = m \right) \geq 
1 - \exp \left( - \frac{2 \, (\numsamp_m \, \minm - Q)^2 }{\numsamp_m} \right)
 - 2 d \exp \left( - \frac{  \Knb \, \epsilon^2 }{ 2 \Lcon^2 \delta^2} \right).
\end{equation}

\end{theorem}

Theorem \ref{thm:error_genf_nnclass} is proved in Appendix \ref{pf:thm_error_genf_nnclass}. Theorem \ref{thm:error_genf_nnclass} is quite similar to Theorem \ref{thm:linclassif_genf} and can be interpreted similarly. Unlike in the previous result, the separability condition of the embedding is based on the pairwise distances of samples from different classes here. The condition \eqref{eq:cond_sep_genf_nn} suggests that the result is useful when the parameter $\Demb_{2 \delta}$ is sufficiently small, which requires the embedding to map nearby samples from the same class in the ambient space to nearby points.


In this section, we have characterized the regularity of the interpolation functions via their rates of variation when restricted to the supports $\M_m$. While the results of this section are generic in the sense that they are valid for any interpolation function with the described regularity properties, we have not examined the construction of such functions. In a practical classification problem where one uses a particular type of interpolation functions, one would also be interested in the adaptation of these results to obtain performance guarantees for the particular type of function used. Hence, in the following section we focus on a popular family of smooth functions; radial basis function (RBF) interpolators, and study the classification performance of this particular type of interpolators.

\subsection{Out-of-sample interpolation with RBF interpolators} 
\label{ssec:oos_rbf}

Here we consider an RBF interpolation function $f: \Hs \rightarrow \R^d$ of the form
\[
f(x) = [f^1(x) \, f^2(x) \, \dots f^d(x)]
\]
such that each component $f^k$ of $f$ is given by
\begin{equation*}
f^k(\x) = \sum_{i=1}^{\numsamp} \cik \,  \phi(\| \x - \xii \|)
\end{equation*}
where $\phi: \R \rightarrow \R^{+}$ is a kernel function, $\cik \in \R$ are coefficients, and $\xii$ are kernel centers. In interpolation with RBF functions, it is common to choose the set of kernel centers as the set of available data samples. Hence, we assume that the set of kernel centers $\{ \xii \}_{i=1}^N$ is selected to be the same as the set of training samples $\X$. We consider a setting where the coefficients $\cik$ are set such that $f(\xii)= \yi$, i.e., $f$ maps each training point in $\X$ to its embedding previously computed with supervised manifold learning.

We consider the RBF kernel $\phi$ to be a Lipschitz continuous function with constant $\Lconf>0$, hence, for any $u, v \in \R$
\begin{equation*}
| \phi(u) - \phi(v) | \leq \Lconf \,  | u - v |.
\end{equation*}
Also, let $\cbnd$ be an upper bound on the coefficient magnitudes such that for all $k=1, \dots, d$
\begin{equation*}
\sum_{i=1}^N |\cik| \leq \cbnd.
\end{equation*}

In the following, we analyze the classification accuracy and extend the results in Section \ref{ssec:oos_reg_func} to the case of RBF interpolators. We first give the following result, which probabilistically bounds how much the value of the interpolator $f$ at a point $x$ randomly drawn from $\pmes_m$ may deviate from the average interpolator value of the training points of the same class within a neighborhood of $x$.
\begin{lemma}
\label{lem:bnd_dev_f_nborhd}

Let $\X=\{ \x_i \}_{i=1}^{\numsamp} \subset \Hs $ be a set of training samples such that each $\x_i$ is drawn i.i.d.~from one of the probability measures $\{\pmes_m \}_{m=1}^M$. Let $\x$ be a test sample randomly drawn according to the probability measure $\pmes_m$ of class $m$. Let 
\begin{equation}
\label{eq:defn_NNx_X}
\nbdm = \{ \x_i \in \X : \x_i \in B_\delta(x), \x_i \sim \pmes_m \}
\end{equation}
be the set of samples in $\X$ that are in a $\delta$-neighborhood of $\x$ and  also drawn from the measure $\pmes_m$. Assume that $\nbdm$ contains $|  \nbdm|=\Knb$ samples. Then 
\begin{equation}
\label{eq:lemma_bnd_devf}
P\left( \| f(x) -  \frac{1}{\Knb} \sum_{\x_j \in \nbdm} f(\x_j)  \|  \leq \sqrt{d}  \cbnd (\Lconf \delta + \epsilon) \right) 
\geq
1- 2 \numsamp \exp \left( - \frac{(\Knb-1) \, \epsilon^2 }{ 2 \Lconf^2 \delta^2} \right).
\end{equation}
\end{lemma}
%


The proof of Lemma \ref{lem:bnd_dev_f_nborhd} is given in Appendix \ref{pf:lem:bnd_dev_f_nborhd}. The lemma states a result similar to the one in Lemma \ref{lem:fdev_neigh_genf}; however, is specialized to the case where $f$ is an RBF interpolator.

We are now ready to present the following main result.

\begin{theorem}
\label{thm:acc_cl_rbfint}
Let $\X =\{ \x_i \}_{i=1}^{\numsamp} \subset \Hs$ be a set of training samples such that each $\x_i$ is drawn i.i.d.~from one of the probability measures $\{\pmes_m \}_{m=1}^M$. Let $\Y$ be an embedding of $\X$ in $\R^d$ that is linearly separable with a $Q$-mean separability margin larger than $\mar_Q$. For a given $\epsilon>0$ and $\delta>0$, let $f$ be an RBF interpolator such that 
\begin{equation}
\label{eq:cond_sep_interp}
\sqrt{d} \, \cbnd \, (\Lconf \delta + \epsilon) \leq  \frac{\mar_Q}{ 2}.
\end{equation}
Consider a test sample $\x$ randomly drawn according to the probability measure $\pmes_m$ of class $m$. If $\X$ contains at least $\numsamp_m$ training samples drawn i.i.d.~from $\pmes_m$ such that 
\begin{equation*}
\numsamp_m > \frac{Q}{\minm}
\end{equation*}
then the probability of correctly classifying $\x$ with the linear classifier given in \eqref{eq:lin_classifier} is lower bounded as
\begin{equation}
\label{eq:prob_lb_rbfint}
P\left( \hat{\class}(\x) = m \right) \geq 
1 - \exp \left( - \frac{2 \, (\numsamp_m \, \minm - Q)^2 }{\numsamp_m} \right)
 - 2 \numsamp \exp \left( - \frac{(\Knb-1) \, \epsilon^2 }{ 2 \Lconf^2 \delta^2} \right).
\end{equation}
\end{theorem}

The theorem is proved in Appendix \ref{pf:thm:acc_cl_rbfint}. The theorem bounds the classification accuracy in terms of the smoothness of the RBF interpolation function and the number of samples. The condition in  \eqref{eq:cond_sep_interp} characterizes the compromise between the separation and the regularity of the interpolator, which depends on the Lipschitz constant of the RBF kernels and the coefficient magnitude. As the Lipschitz constant $\Lconf$ and the coefficient magnitude parameter $\cbnd$ increase (i.e., $f$ becomes less ``regular''), a higher separation $\mar_Q$ is required to provide a performance guarantee. When the separation margin of the embedding and the interpolator satisfy the condition in  \eqref{eq:cond_sep_interp}, the misclassification probability decays exponentially as the number of training samples increases, similarly to the results in Section \ref{ssec:oos_reg_func}.


Theorem \ref{thm:acc_cl_rbfint} studies the misclassification probability when the class labels in the low-dimensional domain are estimated with a linear classifier. We  also present below a bound on the misclassification probability when the nearest-neighbor classifier is used in the low-dimensional domain. 

\begin{theorem}
\label{thm:acc_cl_nn_rbfint}

Let $\X =\{ \x_i \}_{i=1}^{\numsamp} \subset \Hs$ be a set of training samples such that each $\x_i$ is drawn i.i.d.~from one of the probability measures $\{\pmes_m \}_{m=1}^M$. Let $\Y$ be an embedding of $\X$ in $\R^d$ such that 
\begin{equation*}
\begin{split}
\| \y_i - \y_j \| &< \Demb_\delta, \text{ if } \| \x_i - \x_j  \| \leq \delta \text{ and } \classi = \classj \\
\| \y_i - \y_j \| &> \mar,  \  \text{ if }  \classi \neq \classj.
\end{split}
\end{equation*}
%

For given $\epsilon>0$ and $\delta>0$, let $f$ be an RBF interpolator such that 
\begin{equation}
\label{eq:cond_sep_nn_interp}
\sqrt{d} \, \cbnd \, (\Lconf \delta + \epsilon) +   \Demb_{2 \delta} \leq \frac{ \mar}{ 2}.
\end{equation}

Consider a test sample $\x$ randomly drawn according to the probability measure $\pmes_m$ of class $m$. If $\X$ contains at least $\numsamp_m$ training samples drawn i.i.d.~from $\pmes_m$ such that 
\begin{equation*}
\numsamp_m > \frac{Q}{\minm}
\end{equation*}
then the probability of correctly classifying $\x$ with nearest-neighbor classification in  $\R^d$ is lower bounded as
\begin{equation}
\label{eq:prob_lb_nn_rbfint}
P\left( \hat{\class}(\x) = m \right) \geq 
1 - \exp \left( - \frac{2 \, (\numsamp_m \, \minm - Q)^2 }{\numsamp_m} \right)
 - 2 \numsamp \exp \left( - \frac{(\Knb-1) \, \epsilon^2 }{ 2 \Lconf^2 \delta^2} \right).
\end{equation}
\end{theorem}

Theorem \ref{thm:acc_cl_nn_rbfint} is proved in Appendix \ref{pf:thm:acc_cl_nn_rbfint}. While it provides the exact convergence rate as in Theorem \ref{thm:acc_cl_rbfint}, the necessary condition in \eqref{eq:cond_sep_nn_interp} includes also the parameter $\Demb_{2 \delta}$. Hence, if the embedding maps nearby samples from the same class to nearby points, and a compromise is achieved between the separation and the interpolator regularity, the misclassification probability can be upper bounded.

\subsection{Optimizing the scale of Gaussian RBF kernels}
\label{ssec:scale_optim}

In data interpolation with RBFs, it is known that the accuracy of interpolation is quite sensitive to the choice of the shape parameter for several kernels including the Gaussian kernel \citep{Baxter92}. The relation between the shape parameter and the performance of interpolation has been an important problem of interest \citep{Piret07}. In this section, we focus on the Gaussian RBF kernel, which is a popular choice for RBF interpolation due to its smoothness and good spatial localization properties. We study the choice of the scale parameter of the kernel within the context of classification. 

We consider the RBF kernel given by
\begin{equation*}
 \phi(r) = e^{-\frac{r^2}{\sigma^2}}
\end{equation*}
where $\sigma$ is the scale parameter of the Gaussian function. We focus on the condition \eqref{eq:cond_sep_interp} in Theorem \ref{thm:acc_cl_rbfint}
\begin{equation*}
\sqrt{d} \, \cbnd \, (\Lconf \delta + \epsilon) \leq \mar_Q / 2,
\end{equation*}
(or equivalently the condition \eqref{eq:cond_sep_nn_interp} if the nearest neighbor classifier is used), which relates the interpolation function properties with the separation. In particular, for a given separation margin, this condition is satisfied more easily when the term on the left hand side of the inequality is smaller. Thus, in the following, we derive an expression for the left hand side of the above inequality by deriving the Lipschitz constant $\Lconf$ and the coefficient bound $\cbnd$ in terms of the scale parameter $\sigma$ of the Gaussian kernel. We then study the scale parameter that minimizes $\sqrt{d} \, \cbnd \, (\Lconf \delta + \epsilon)$.

Writing the condition $f(\xii)= \yi$ in a matrix form for each dimension $k= 1, \dots, d$, we have
\begin{equation}
\label{eq:lin_sys_coefk}
\Phi \cvect^k = \y^k
\end{equation}
where $\Phi \in \R^{\numsamp \times \numsamp}$ is a matrix whose $(i,j)$-th entry is given by $\Phi_{ij} = \phi (\| \xii - \xj \|)$, $\cvect^k \in \R^{\numsamp \times 1}$ is the coefficient vector whose $i$-th entry is $\cik$, and $\y^k \in \R^{\numsamp \times 1}$ is the data coordinate vector giving the $k$-th dimensions of the embeddings of all samples, i.e., $\yi^k= \Y_{ik}$. Assuming that the embedding is computed with the usual scale constraint $\Y^T \Y = I$, we have $\| \y^k \|=1$. The norm of the coefficient vector can then be bounded as
\begin{equation}
\label{eq:bnd_coefv_Phi}
\| \cvect^k \| \leq \| \Phi^{-1} \|  \| \y^k \| = \| \Phi^{-1} \|.
\end{equation}

In the rest of this section, we assume that the data $\X$ are sampled from the Euclidean space, i.e., $\Hs = \R^n$. We first use a result by \citet{NarcowichSW94} in order to bound the norm $\| \Phi^{-1} \|$ of the inverse matrix. From \citep[Theorem 4.1]{NarcowichSW94} we get\footnote{The result stated in \cite[Theorem 4.1]{NarcowichSW94} is adapted to our study by taking the measure as $\beta(\rho)=\delta(\rho - \rho_0)$ so that the RBF kernel defined in \citep[(1.1)]{NarcowichSW94} corresponds to a Gaussian function as $F(r)=\exp(-\rho_0 \, r^2)$. The scale of the Gaussian kernel is then given by $\sigma= {\rho_0}^{-1/2}$.}
\begin{equation}
\label{eq:norm_bnd_invphi}
\| \Phi^{-1} \| \leq \beta \, \sigma^{-n} e^{\alpha \sigma^2}
\end{equation}
where $\alpha>0$ and $\beta>0$ are constants depending on the dimension $n$ and the minimum distance between the training points $\X$ (separation radius) \citep{NarcowichSW94}. As the $\ell_1$-norm of the coefficient vector can be bounded as $\| \cvect^k \|_1 \leq \sqrt{\numsamp} \| \cvect^k \|$, from \eqref{eq:bnd_coefv_Phi} one can set the parameter $\cbnd$ that upper bounds the coefficients magnitudes as
\begin{equation*}
\cbnd = \ags \sigma^{-n} e^{\alpha \sigma^2}
\end{equation*}
where $\ags= \beta \sqrt{\numsamp} $. 

Next, we derive a Lipschitz constant for the Gaussian kernel $\phi(r)$ in terms of $\sigma$. Setting the second derivative of $\phi$ to zero
\begin{equation*}
\frac{d^2\phi}{dr^2} =  e^{-\frac{r^2}{\sigma^2}} \left( \frac{4 r^2}{\sigma^4} - \frac{2}{\sigma^2} \right) = 0
\end{equation*}
we get that the maximum value of $| d\phi/dr |$ is attained at $r=\sigma/\sqrt{2}$. Evaluating $| d\phi/dr |$ at this value, we obtain
\begin{equation*}
\Lconf = \sqrt{2} e^{-\half} \sigma^{-1}.
\end{equation*}

Now rewriting the condition \eqref{eq:cond_sep_interp} of the theorem, we have
\begin{equation*}
\sqrt{d} \, \cbnd \, (\Lconf \delta + \epsilon) 
=  \ags_1 \sigma^{-n-1} e^{\alpha \sigma^2} \, 
+ \ags_2  \sigma^{-n} e^{\alpha \sigma^2} \leq \mar_Q / 2 
\end{equation*}
where $\ags_1= \sqrt{2d} \,  \ags \, e^{-1/2} \delta$ and $\ags_2=  \sqrt{d} \, \ags \, \epsilon $. We thus determine the Gaussian scale parameter $\sigma$ that minimizes 
\begin{equation*}
F(\sigma)=  \ags_1 \sigma^{-n-1} e^{\alpha \sigma^2} \, 
+ \ags_2  \sigma^{-n} e^{\alpha \sigma^2}. 
\end{equation*}
First, notice that as $\sigma \rightarrow 0$ and $\sigma \rightarrow \infty$, the function $F(\sigma) \rightarrow \infty$. Therefore, it has at least one minimum. Setting
\begin{equation*}
\frac{dF}{d \sigma}= e^{\alpha \sigma^2} \sigma^{-n-2} 
\big(
 2 \alpha \ags_2 \sigma^3 + 2 \alpha \ags_1 \sigma^2 -  \ags_2 n \sigma - \ags_1 (n+1) 
\big) = 0 
\end{equation*}
we need to solve 
\begin{equation}
\label{eq:cubic_eqn_optsigma}
 2 \alpha \ags_2 \sigma^3 + 2 \alpha \ags_1 \sigma^2 -  \ags_2 n \sigma - \ags_1 (n+1) =0.
\end{equation}
The leading and the second-degree coefficients are positive, while the first-degree and the constant coefficients are negative in the above cubic polynomial. Then, the sum of the roots is negative and the product of the roots is positive. Therefore, there is one and only one positive root $\sigma_{opt}$, which is the unique minimizer of $F(\sigma)$. 

The existence of an optimal scale parameter $0<\sigma_{opt}<\infty$ for the RBF kernel can be intuitively explained as follows. When $\sigma$ takes too small values, the support of the RBF function concentrated around the training points does not sufficiently cover the whole class supports $\M_m$. This manifests itself in \eqref{eq:cond_sep_interp} with the increase in the term $\Lconf$, which indicates that the interpolation function is not sufficiently regular. This weakens the guarantee that a test sample will be interpolated sufficiently close to its neighboring training samples from the same class and mapped to the correct side of the hyperplane in the linear classifier. On the other hand, when $\sigma$ increases too much, the stability of the linear system \eqref{eq:lin_sys_coefk} is impaired and the coefficients $\cvect$ increase too much. This results in an overfitting of the interpolator and, therefore, decreases the classification performance. Hence, the analysis in this section provides a theoretical justification of the common knowledge that $\sigma$ should be set to a sufficiently large value while avoiding overfitting.

\textbf{Remark:}  It is also interesting to observe how the optimal scale parameter changes with the number of samples $\numsamp$. In the study \citep{NarcowichSW94}, the constants $\alpha$ and $\beta$ in \eqref{eq:norm_bnd_invphi} are shown to vary with the separation radius $q$ at  rates $\alpha = O(q^{-2})$ and $\beta = O (q^n)$, where the separation radius $q$ is proportional to the smallest distance between two distinct training samples. Then a reasonable assumption is that the separation radius $q$ should typically decrease at  rate  $O(N^{-1/n})$ as $N$ increases. Using this relation, we get that $\alpha$ and $\beta$ should vary at rates $\alpha = O (N^{2/n})$ and $\beta = O(N^{-1})$ with $N$. It follows that $\ags= \beta \sqrt{N} = O (N^{-1/2})$, and the  parameters $\ags_1$, $\ags_2$ of the cubic polynomial in \eqref{eq:cubic_eqn_optsigma} also vary with $N$ at rates $\ags_1= O (N^{-1/2})$, $\ags_2 = O (N^{-1/2})$. The equation \eqref{eq:cubic_eqn_optsigma}  in $\sigma$ can then be rearranged as
\[
b_3 \sigma^3 + b_2 \sigma^2 - b_1 \sigma - b_0 = 0
\]
such that the constants vary with $N$ at rates $b_3 = O(N^{2/n})$, $b_2 = O(N^{2/n})$, $b_1=O(1)$, $b_0= O(1)$. We can then inspect how the roots of this equation change with $N$ as $N$ increases. Since $b_3$ and $b_2$ dominate the other coefficients for large $N$, three real roots will exist if $N$ is sufficiently large, two of which are negative and one is positive. The sum of the pairwise products of the roots is negative and it decays with $N$ at rate $O(N^{-2/n})$, and the product of the roots also decays with $N$. Then at least two of the roots must decay with $N$. Meanwhile, the sum of the three roots is $O(1)$ and negative. This shows that one of the negative roots is $O(1)$, i.e., does not decay with $N$. From the product of three roots, we then observe that the product of the two decaying roots is $O(N^{-2/n})$. However, their sum also decays at the same rate (from the sum of the pairwise products), which is possible if their dominant terms have the same rate and cancel each other. We conclude that both of the decaying roots vary at rate $O(N^{-1/n})$, one of which is the positive root and the optimal value $\sigma_{opt}$ of the scale parameter. This analysis shows that the scale parameter of the Gaussian kernel should be adapted to the number of training samples, and a smaller kernel scale must be preferred for a larger number of training samples. In fact, the relation $\sigma_{opt} = O(N^{-1/n})$ is quite intuitive, as the average or typical distance between two samples will also decrease at rate  $O(N^{-1/n})$ as the number of samples $N$ increases in an $n$-dimensional space. Then the above result simply suggests that the  kernel scale should be chosen as proportional to the average distance between the training samples.

\subsection{Discussion of the results in relation with previous results}
\label{ssec:disc_prev_res}

In Theorems \ref{thm:acc_cl_rbfint} and \ref{thm:acc_cl_nn_rbfint}, we have presented a result that characterizes the performance of classification with RBF interpolation functions. In particular, we have considered a setting where an RBF interpolator is fitted to each dimension of a low-dimensional embedding where different classes are separable. Our study has several links with RBF networks or least-squares regression algorithms. In this section, we interpret our findings in relation with previously established results.

Several previous works study the performance of learning by considering a probability measure $\rho$ defined on $X \times Y$, where $X$ and $Y$ are two sets. The ``label'' set $Y$ is often taken as an interval $[-L, L]$. Given a set of data pairs $\{ (x_j, y_j ) \}_{j=1}^\numsamp$ sampled from the distribution $\rho$, the RBF network estimates a function $\hat f$ of the form
\begin{equation}
\label{eq:hatf_rbf}
\hat f(x)=\sum_{i=1}^{\nrbf} \ci \,  \phi \left( \frac{ \| \x - t_i \| }{\sigma_i} \right).
\end{equation}
The number of RBF terms $\nrbf$ may be different from the number of samples $\numsamp$ in general. The function $\hat f$ minimizes the empirical error 
\begin{equation*}
\hat f = \arg \min_f \sum_{j=1}^\numsamp \left( f(x_j) - y_j \right)^2.
\end{equation*}

The function $\hat f$ estimated from a finite collection of data samples is often compared to the regression function \citep{CuckerS02}
\begin{equation*}
f_o (x) = \int_Y y \, d\rho(y | x) 
\end{equation*}
where $d\rho(y | x) $ is the conditional probability measure on $Y$. The regression function $f_o$ minimizes the expected risk as
\begin{equation*}
f_o = \arg \min_f \int_{X \times Y} \big( f(x) - y \big)^2 d\rho.
\end{equation*}
As the probability measure $\rho$ is not known in practice, the estimate $\hat f$ of $f_o$ is obtained from data samples. Several previous works have characterized the performance of learning by studying the approximation error \citep{NiyogiG96}, \citep{LinLRX14}
\begin{equation}
\label{eq:defn_app_err}
\E [ (f_o - \hat f)^2 ] = \int_X  (f_o(x) - \hat f(x))^2 d\rho_X(x)
\end{equation}
where $\rho_X$ is the marginal probability measure on $X$. This definition of the approximation error can be adapted to our setting as follows. In our problem the distribution of each class is assumed to have  a bounded support, which is a special case of modeling the data with an overall probability distribution $\rho$. If the supports $\M_m$ are assumed to be nonintersecting, the regression function $f_o$ is given by
\begin{equation*}
f_o(x) = \sum_{m=1}^M m \, I_m(x) 
\end{equation*}
which corresponds to the class labels $m=1, \dots, M$, where $I_m$ is the indicator function of the support $\M_m$. It is then easy to show that the approximation error $\E [ (f_o - \hat f)^2 ] $ can be bounded as a constant times the probability of misclassification $P( \hat{\class}(\x) \neq m )$. Hence, we can compare our misclassification probability bounds in Section \ref{ssec:oos_rbf} with the approximation error in other works.


The study in \citep{NiyogiG96} assumes that the regression function is an element of the Bessel potential space of a sufficiently high order and that the sum of the coefficients $|\ci|$ is bounded. It is then shown that for data sampled from $\Rn$, with probability greater than $1-\delta$ the approximation error in \eqref{eq:defn_app_err} can be bounded as 
\begin{equation}
\label{eq:res_niyogi}
\E [ (f_o - \hat f)^2 ] \leq O \left( \frac{1}{\nrbf} \right) + O\left( \sqrt{ \frac{ \nrbf n \log(\nrbf \numsamp) - \log(\delta) }{\numsamp} }\right)
\end{equation}
where $\nrbf$ is the number of RBF terms.

The analysis by \citet{LinLRX14} considers families of RBF kernels that include the Gaussian function. Supposing that the regression function $f_o$ is of Sobolev class $W_2^r$, and that the number of RBF terms is given by $\nrbf = \numsamp^{\frac{n}{n+2r}}$ in terms of the number of samples $\numsamp$, the approximation error is bounded as
\begin{equation}
\label{eq:res_lin}
\E [ (f_o - \hat f)^2 ] \leq O(\numsamp^{-\frac{2r}{n+2r} }   \log^2(\numsamp)).
\end{equation}

Next, we overview the study by \citet{AguirreKB02}, which studies the performance of  RBFs in a Probably Approximately Correct (PAC)-learning framework. For $X \subset \Rn$, a family $\mathcal{F}$ of measurable functions from $X$ to $ [0,1]$ is considered and the problem of approximating a target function $f_0$ known only through examples with a function in $\hat f \in \mathcal{F}$ is studied. The authors use a previous result from \citep{Vidyasagar97} that relates the accuracy of empirical risk minimization to the covering number of $\mathcal{F}$ and the number of samples. Combining this result with the bounds on covering number estimates of Lipschitz continuous functions \citep{KolmogorovT61}, the following result is obtained for PAC function learning with RBF neural networks with Gaussian kernel. Let the coefficients be bounded as $|\ci| \leq A$, a common scale parameter be chosen as $\sigma_i=\sigma$, and $\E [ |f_0 - \hat f| ] $ be computed under a uniform probability measure $\rho$. Then if the number of samples satisfies
\begin{equation}
\label{eq:res_agui}
\numsamp \geq \frac{8}{\epsac^2} \log \bigg(\frac{\sqrt{2} \nrbf n A}
{e^{-1/2}  \sigma \delc } \bigg)
\end{equation}
an approximation of the target function is obtained with accuracy parameter $\epsac$ and confidence parameter $\delc$:
\begin{equation}
\label{eq:ineq_pac}
P( \E [ |f_0 - \hat f| ] > \epsac ) \leq \delc.
\end{equation}
In the above expression, the expectation is over the test samples, whereas the probability is over the training samples; i.e., over all possible distributions of training samples, the probability of having the average approximation error larger than $\epsac$ is bounded. Note that, our results in Theorems \ref{thm:acc_cl_rbfint} and \ref{thm:acc_cl_nn_rbfint}, when translated into the above PAC-learning framework, correspond to a confidence parameter of $\delc=0$. This is because the misclassification probability bound of a test sample is valid for any choice of the training samples, provided that the condition \eqref{eq:cond_sep_interp} (or the condition \eqref{eq:cond_sep_nn_interp}) holds. Thus, in our result the probability running over the training samples in \eqref{eq:ineq_pac} has no counterpart. When we take $\delc=0$, the above result does not provide a useful bound since $\numsamp \rightarrow \infty$ as $\delc \rightarrow 0$. By contrast, our result is valid only if the conditions \eqref{eq:cond_sep_interp}, \eqref{eq:cond_sep_nn_interp} on the interpolation function holds. It is easy to show that, assuming nonintersecting class supports $\M_m$, the expression $ \E [ |f_0 - \hat f| ]$ is given by a constant times the probability of misclassification. The accuracy parameter $\epsac$ can then be seen as the counterpart of the misclassification probability upper bound given on the right hand sides of \eqref{eq:prob_lb_rbfint} and \eqref{eq:prob_lb_nn_rbfint} (the expression subtracted from 1). At fixed $\numsamp$, the dependence of the accuracy on the kernel scale parameter is monotonic in the bound \eqref{eq:res_agui}; $\epsac$ decreases as $\sigma$ increases. Therefore, this bound does not guide the selection of the scale parameter of the RBF kernel, while the discussion in Section \ref{ssec:scale_optim} (confirmed by the experimental results in Section \ref{ssec:exp_oos}) suggests the existence of an optimal scale. 

Finally, we mention some results on the learning performance of regularized least squares regression algorithms. In \citep{CaponnettoV07} optimal rates are derived for the regularized least squares method in a Reproducing Kernel Hilbert Space (RKHS) in the minimax sense. It is shown that, under some hypotheses concerning the data probability measure and the complexity of the family of learnt functions, the maximum error (yielded by the worst distribution) obtained with the regularized least squares method converges at a rate of $O(1/\numsamp)$. Next, the work in \citep{SteinwartHS09} shows that, in regularized least squares regression over a RKHS, if the eigenvalues of the kernel integral operator decay sufficiently fast, and if the $\ell_\infty$-norms of regression functions can be bounded, the error of the classifier converges at a rate of up to $O(1/\numsamp)$ with high probability. Steinwart et al. also examine the learning performance in relation with the exponent of the function norm in the regularization term and show that the learning rate is not affected by the choice of the exponent of the function norm.

We now overview the three bounds given in \eqref{eq:res_niyogi}, \eqref{eq:res_lin}, and \eqref{eq:res_agui} in terms of the dependence of the error on the number of samples. The results in \eqref{eq:res_niyogi} and \eqref{eq:res_lin} provide a useful bound only in the case where the number of samples $\numsamp$ is larger than the number of RBF terms $\nrbf$, contrary to our study where we treat the case $\nrbf=\numsamp$. If it is assumed that $\numsamp$ is sufficiently larger than $\nrbf$, the result in \eqref{eq:res_niyogi} predicts a rate of decay of only $O(\sqrt{ \log(\numsamp)/\numsamp})$ in the misclassification probability.  The bound in \eqref{eq:res_lin} improves with the Sobolev regularity of the regression function; however, the dependence of the error on the number of samples is of a similar nature to the one in \eqref{eq:res_niyogi}. Considering $\epsac$ as a misclassification error parameter in the bound in \eqref{eq:res_agui}, the error decreases at a rate of $O(\numsamp^{-1/2})$ as the number of samples increases. The analysis in \citep{CaponnettoV07} and \citep{SteinwartHS09} also provide the similar rates of convergence of $O(\numsamp^{-1})$. Meanwhile, our results in Theorems \ref{thm:acc_cl_rbfint} and \ref{thm:acc_cl_nn_rbfint} predict an exponential decay in the misclassification probability as the number of samples $\numsamp$ increases (under the reasonable assumption that $\numsamp_m=O(\numsamp)$ for each class $m$). The reason why we arrive at a more optimistic bound is the specialization of the analysis to the considered particular setting, where the support of each class is assumed to be restricted to a totally bounded region in the ambient space, as well as the assumed relations between the separation margin of the embedding and the regularity of the interpolation function.


Another difference between these previous results and ours is the dependence on the dimension. The results in \eqref{eq:res_niyogi},  \eqref{eq:res_lin}, and \eqref{eq:res_agui} predict an increase in the error at the respective rates of $O(\sqrt{n})$, $O(e^{-1/n})$, and $O(\sqrt{ \log n} )$ with the ambient space dimension $n$. While these results assume that the data $\X \subset \Rn$ is in an Euclidean space of dimension $n$, our study assumes the data $\X$ to be in a generic Hilbert space  $H$. The results in Theorems \ref{thm:linclassif_genf}-\ref{thm:acc_cl_rbfint} involve the dimension $d$ of the low-dimensional space of embedding and does not explicitly depend on the dimension of the ambient Hilbert space $H$ (which could be infinite-dimensional). However, especially in the context of manifold learning, it is interesting to analyze the dependence of our bound on the intrinsic dimension of the class supports $\M_m$. 

In order to put the expressions \eqref{eq:prob_lb_rbfint}, \eqref{eq:prob_lb_nn_rbfint} in a more convenient form, let us reduce one parameter by setting $Q=\numsamp_m \minm /2$. Then the misclassification probability is of
\begin{equation*}
O \left(\exp(- \numsamp_m \minm^2)  
+ \numsamp \exp \left( -\frac{\numsamp_m \, \minm \, \epsilon^2}{\Lconf^2 \, \delta^2} \right)
\right). 
\end{equation*}

We can relate the dependence of this expression on the intrinsic dimension as follows. Since the supports $\M_m$ are assumed to be totally bounded, one can define a parameter $\Theta$ that represents the ``diameter'' of $\M_m$, i.e., the largest distance between any two points on $\M_m$. Then the measure $\minm$ of the minimum ball of radius $\delta$ in $\M_m$ is of $O((\delta/\Theta)^\dimM)$, where $\dimM$ is the intrinsic dimension of $\M_m$. Replacing this in the above expression gives the probability of misclassification as
\begin{equation*}
O \left(\exp \left(-  \frac{ \numsamp_m \, \delta^{2\dimM} }{\Theta^{2\dimM}}  \right)  
+ \numsamp \exp \left( -\frac{\numsamp_m \,  \delta^{\dimM-2}  \, \epsilon^2}{\Lconf^2 \,  \Theta^\dimM} \right)
\right).
\end{equation*}
This shows that in order to retain the correct classification guarantee, as the intrinsic dimension $\dimM$ grows, the number of samples $\numsamp_m$ should increase at a geometric rate with $\dimM$. In supervised manifold learning problems, data sets usually have a low intrinsic dimension, therefore, this geometric rate of increase can often be tolerated. Meanwhile the dimension of the ambient space is typically high, so that performance bounds independent of the ambient space dimension are of particular interest. Note that generalization bounds in terms of the intrinsic dimension have been proposed in some previous works as well \citep{BickelL07}, \citep{Kpotufe11}, for the local linear regression and the K-NN regression problems.


\section{Separability of supervised nonlinear embeddings}
\label{sec:sep_analysis}

In the results in Section \ref{sec:class_anly_supml}, we have presented  generalization bounds for classifiers based on linearly separable embeddings. One may wonder if the separability assumption is easy to satisfy when computing structure-preserving nonlinear embeddings of data. In this section, we try to answer this question by focusing on a particular family of supervised dimensionality reduction algorithms, i.e., supervised Laplacian eigenmaps embeddings, and analyze the conditions of separability. We first discuss the supervised Laplacian eigenmaps embeddings in Section \ref{ref:sec_prob_formul} and then present results in Section \ref{ssec:anly_twoClass} about the linearly separability of these embeddings.

\subsection{Supervised Laplacian eigenmaps embeddings}
\label{ref:sec_prob_formul}

Let $\X=\{ \x_i \}_{i=1}^{\numsamp} \subset \Hs $ be a set of training samples, where each $\x_i$ belongs to one of $\numclass$ classes. Most manifold learning algorithms rely on a graph representation of data. This graph can be a complete graph in some works, in which case an edge exists between each pair of samples. Meanwhile, in some manifold learning algorithms, in order to better capture the intrinsic geometric structure of data, each data sample is connected only to its nearest neighbors in the  graph. In this case, an edge exists only between neighboring data samples.

In our analysis, we consider a weighted data graph $\G$ each vertex of which represents a point $\x_i$. We write $\xii \sim \xj$, or simply $\ii \sim \jj $ if the graph contains an edge between the data samples $\xii$, $\xj$. We denote the edge weight as $\wij>0$. The weights $\wij$ are usually determined as a positive and monotonically decreasing function of the distance between $\xii$ and $\xj$ in $\Hs$, where the Gaussian function is a common choice. Nevertheless, we maintain a generic formulation here without making any assumption on the neighborhood or weight selection strategies.

Now let $\Gw$ and $\Gb$ represent two subgraphs of $\G$, which contain the edges of $G$ that are respectively within the same class and between different classes. Hence, $\Gw$ contains an edge $\ii \simw \jj$ between samples $\xii$ and $\xj$, if $\ii \sim \jj$ and $\classi=\classj$. Similarly, $\Gb$ contains an edge $\ii \simb \jj$ if $\ii \sim \jj$ and $\classi \neq \classj$. We assume that all vertices of $\G$ are contained in both $\Gw$ and $\Gb$; and that $\Gw$ has exactly $M$ connected components such that the training samples in each class form a connected component\footnote{The straightforward application of common graph construction strategies, like connecting each training  sample to its K-nearest neighbors or to its neighbors within a given distance, may result in several disconnected components in a single class in the graph if there is much diversity in that class. However, this difficulty can be easily overcome by introducing extra edges to bridge between graph components that are originally disconnected.}. We also assume that $\Gw$ and $\Gb$ do not contain any isolated vertices; i.e., each data sample $\xii$ has at least one neighbor in both graphs.

The $\numsamp \times \numsamp$ weight matrices $\Ww$ and $\Wb$ of $\Gw$ and $\Gb$ have entries as follows.
\begin{equation*}
\Ww(i,j)= \bigg\{ 
\begin{array} {l}
\wij \, \, \text{    if   } \ii \sim \jj \text{ and }  \classi=\classj \\ 
0 \, \, \text{    otherwise}  \end{array}
\end{equation*}
\begin{equation*}
\Wb(i,j)= \bigg\{ 
\begin{array} {l}
\wij \, \, \text{    if   } \ii \sim \jj \text{ and }  \classi \neq \classj \\ 
0 \, \, \text{    otherwise}  \end{array}
\end{equation*}
Let $\dw(i)$ and $\db(i)$ denote the degrees of $\xii$ in $\Gw$ and $\Gb$ 
\begin{eqnarray*}
\dw(i)= \sum_{j\simw i } \wij, 
\qquad \qquad
\db(i)= \sum_{j\simb i } \wij
\end{eqnarray*}
and $\Dw$, $\Db$ denote the $\numsamp \times \numsamp$ diagonal degree matrices given by $\Dw(i,i)=\dw(i)$, $\Db(i,i)=\db(i)$. The normalized graph Laplacian matrices $\Lw$ and $\Lb$ of $\Gw$ and $\Gb$ are then defined as
\begin{equation*}
\Lw := \Dw^{-1/2} (\Dw-\Ww) \Dw^{-1/2}, 
\qquad \qquad
\Lb := \Db^{-1/2} (\Db-\Wb) \Db^{-1/2}. 
\end{equation*}

Supervised extensions of the Laplacian eigenmaps and LPP algorithms seek a $d$-dimensional embedding of the data set $\X$, such that each $\xii$ is represented by a vector $\yi \in \R^{d\times 1}$. Denoting the new data matrix as $\Y = [\y_1 \, \y_2 \, \dots  \, \y_\numsamp]^T \in \R^{\numsamp \times d} $, the  coordinates of data samples are computed by solving the problem
\begin{equation}
\text{``Minimize } \tr(\Y^T \Lw \Y) \text{ while maximizing }  \tr(\Y^T \Lb \Y) \text{.''}
\label{eq:supLap_informal}
\end{equation}
%
The reason behind this formulation can be explained as follows. For a graph Laplacian matrix $L=D^{-1/2} (D-W) D^{-1/2}$, where $D$ and $W$ are respectively the degree and the weight matrices, defining the coordinates $\Z=D^{-1/2} \Y$ normalized with the vertex degrees, we have 
\begin{equation}
\label{eq:form_withinvar}
\tr(\Y^T L \, \Y) = \tr(Z^T (D-W) Z) = \sum_{i \sim j} \| \zi - \zj \|^2 \wij
\end{equation}
where $\zi$ is the $i$-th row of $\Z$ giving the normalized coordinates of the embedding of the data sample $\xii$. Hence, the problem in (\ref{eq:supLap_informal}) seeks a representation $\Y$ that maps nearby samples in the same class to nearby points, while mapping nearby samples from different classes to distant points. In fact, when the samples $\xii$ are assumed to come from a manifold $\M$, the term $y^T L y$ is the discrete equivalent of
\begin{equation*}
\int_{\M} \| \nabla f(x) \|^2 dx
\end{equation*}
where $f: \M \rightarrow \R$ is a continuous function on the manifold that extends the one-dimensional coordinates $y$ to the whole manifold. Hence, the term $\tr(\Y^T L \Y)$ captures the rate of change of the learnt coordinate vectors $\Y$ over the underlying manifold.
%
Then, in a setting where the samples of different classes come from $M$ different manifolds $\{ \M_m \}_{m=1}^M$, the formulation in \eqref{eq:supLap_informal} looks for a function that has a slow variation on each manifold $\M_m$, while having a fast variation ``between'' different manifolds.

The supervised learning problem in (\ref{eq:supLap_informal}) has so far been studied by several authors with slight variations in their problem formulations. \citet{Raducanu12} minimize a weighted difference of the within-class and between-class similarity terms in  (\ref{eq:supLap_informal}) in order to learn a nonlinear embedding. Meanwhile, linear dimensionality reduction methods pose the manifold learning problem as the learning of a linear projection matrix $P\in \R^{d\times n}$; therefore, solve the problem  in \eqref{eq:supLap_informal} under the constraint $\yi= P \, \xii$, where $\xii \in \R^{n\times 1}$ and $d<n$. \citet{Hua12} formulate the problem as the minimization of the difference of the within-class and the between-class similarity terms in (\ref{eq:supLap_informal}) as well. Thus, their algorithm can be seen as the linear version of the method by \citet{Raducanu12}. \citet{Sugiyama07} proposes an adaptation of the Fisher discriminant analysis algorithm to preserve the local structures of data. Data sample pairs are weighted with respect to their affinities in the construction of the within-class and the between-class scatter matrices in Fisher discriminant analysis. Then the trace of the ratio of the between-class and the within-class scatter matrices is maximized to learn a linear embedding. Meanwhile, the within-class and the between-class local scatter matrices are closely related to the two terms in (\ref{eq:supLap_informal}) as shown by \citet{Yang11}. The terms $\Y^T \Lw \Y$ and $\Y^T \Lb \Y$, when evaluated under the constraint $\yi= P \, \xii$, become equal to the locally weighted within-class and between-class scatter matrices of the projected data. \citet{Cui12} and \citet{Wang09} propose to maximize the ratio of the between-class and the within-class local scatters in the learning. \citet{Yang11} optimize the same objective function, while they construct the between-class graph only on the centers of mass of the classes. \citet{Zhang12} similarly optimize a Fisher metric to maximize the ratio of the between- and within-class scatters; however, the total scatter is also taken into account in the objective function in order to preserve the overall manifold structure.

All of the above methods use similar formulations of the supervised manifold learning problem and give comparable results. In our study, we base our analysis on the following formal problem definition  
\begin{equation}
\label{eq:supLap_formal}
\min_{Y} \tr(\Y^T \Lw \Y) - \mu \,  \tr(\Y^T \Lb \Y) \text{ subject to } \Y^T \Y = I
\end{equation}
which minimizes the difference of the within-class and the between-class similarity terms as in works such as \citep{Raducanu12} and \citep{Hua12}. Here $I$ is the $d \times d$ identity matrix and $\mu>0$ is a parameter adjusting the weights of the two terms. The condition $\Y^T \Y = I$ is a commonly used constraint to remove the scale ambiguity of the coordinates. The solution of the problem (\ref{eq:supLap_formal}) is given by the first $d$ eigenvectors of the matrix
\begin{equation*}
\Lw - \mu \Lb
\end{equation*}
corresponding to its smallest eigenvalues.

Our purpose in this section is then to theoretically study the linear separability of the  learnt coordinates of training data, with respect to the definition of linear separability given in \eqref{def:lin_sepbilty}. In the following, we determine some conditions on the graph properties and the weight parameter $\mu$ that ensure the linear separability. We derive lower bounds on the margin $\mar$ and study its dependence on the model parameters. Let us give beforehand the following definitions about the graphs $\Gw$ and $\Gb$.

\begin{definition} The volume of the subgraph of $\Gw$ that corresponds to the connected component containing samples from class $k$ is
\[
\vol_k := \sum_{\ii: \, \classi=k} \dw(i).
\]
We define the maximal within-class volume as
\[
\volmax := \max_{k=1, \dots, M} \vol_k.
\]
The volume of the component of $\Gb$ containing the edges between the samples of classes $k$ and $l$ is \footnote{In order to keep the analogy with the definition of $\vol_k $, a $2$ factor is introduced in this expression as each edge is counted only once in the sum.}
\[
\volbetkl := 
\sum_{\begin{subarray} {c} \ii \simb \jj \\ \classi=k, \classj=l \end{subarray}}
2 \, \wij.
\]
We then define the maximal pairwise between-class volume as
\[
\volbetmax := \max_{k \neq l} \volbetkl.
\]
\end{definition}

In a connected graph, the distance between two vertices $\xii$ and $\xj$ is the number of edges in a shortest path joining $\xii$ and $\xj$. The diameter of the graph is then given by the maximum distance between any two vertices in the graph \citep{Chung96}. We define the diameter of the connected component of $\Gw$ corresponding to class $k$ as follows. 

\begin{definition}
\label{def:diameter_subg}
For any two vertices $\xii$ and $\xj$ such that $\classi=\classj=k$, consider a within-class shortest path joining $\xii$ and $\xj$, which contains samples only from class $k$. Then the diameter $\diam_k$ of the connected component of $\Gw$ corresponding to class $k$ is the maximum number of edges in the within-class shortest path joining any two vertices $\xii$ and $\xj$ from class $k$.
\end{definition}

\begin{definition}
The minimum edge weight within class $k$ is defined as
\[
\wmink := 
\min_{\begin{subarray} {c} \ii \simw \jj \\ \classi=\classj=k \end{subarray}}
\wij.
\]

\end{definition}

\subsection{Separability bounds for two classes}
\label{ssec:anly_twoClass}


We now present a lower bound for the linear separability of the embedding obtained by solving (\ref{eq:supLap_formal}) in a setting with two classes $\classi \in \{1, 2 \}$. We first show that an embedding of dimension $d=1$ is sufficient to achieve linear separability for the case of two classes. We then derive a lower bound on the separation in terms of the graph parameters and the algorithm parameter $\mu$.

Consider a one-dimensional embedding $\Y=\y=[\y_1 \, \y_2 \, \dots \, \y_N]^T \in \R^{\numsamp \times 1}$, where $\y_i \in \R$ is the coordinate of the data sample $\xii$ in the one-dimensional space. The coordinate vector $\y$ is given by the eigenvector of $\Lw - \mu \Lb$ corresponding to its smallest eigenvalue. We begin with presenting the following result, which states that the samples from the two classes are always mapped to different halves (nonnegative or nonpositive) of the real line.

\begin{lemma}
\label{lem:twoclass_diffsign}
The learnt embedding $\y$ of dimension $d=1$ satisfies
\begin{equation*}
\begin{split}
\yi &\leq 0 \qquad \text{if } \, \classi=1 \text{ (or respectively \classi=2)}\\
\yi &\geq 0 \qquad \text{if } \, \classi=2  \text{ (or respectively \classi=1)}
\end{split}
\end{equation*}
for any $\mu>0$ and for any choice of the graph parameters.
\end{lemma}



Lemma \ref{lem:twoclass_diffsign} is proved in Appendix \ref{pf:lem_twoclass_diffsign}. The lemma states that in one-dimensional embeddings of two classes, samples from different classes always have coordinates with different signs. Therefore, the hyperplane given by $\hyp=1$, $b=0$ separates the data as $\hyp^T \yi \leq 0$ for $\classi=1$ and $\hyp^T \yi \geq 0$ for $\classi=2$ (since the embedding is one dimensional, the vector $\hyp$ is a scalar in this case). However, this does not guarantee that the data is separable with a positive margin $\mar>0$. In the following result, we show that a positive margin exists and give a lower bound on it. In the rest of this section, we assume without loss of generality that classes $1$ and $2$ are respectively mapped to the negative and positive halves of the real axis.

\begin{theorem}
\label{thm:sep2class}
Defining the normalized data coordinates $\z = \Dw^{-1/2} y$, let 
\[ \zonemax:=\max_{i: \, \classi=1} \zi
\qquad
\ztwomin:=\min_{i: \, \classi=2} \zi
\]
denote the maximum and minimum coordinates that classes $1$ and $2$ are respectively mapped to with a one-dimensional embedding learnt with supervised Laplacian eigenmaps. We also define the parameters
\begin{equation*}
\wminbyD=\min_{k\in \{1,2\}} \frac{\wmink}{\diam_k} \, , 
\qquad \qquad
\rdeg_i = \frac{\dw(i)}{\db(i)} \, ,
\qquad \qquad
\rdegmax= \max_{i} \rdeg_i \, , 
\end{equation*}
where $\diam_k$ is the diameter of the graph corresponding to class $k$ as defined in Definition \ref{def:diameter_subg}. Then, if the weight parameter is chosen such that $0 < \mu < \wminbyD/(\rdegmax \volbetmax)$, any supervised Laplacian embedding of dimension $d \geq 1$ is linearly separable with a positive margin lower bounded as below: 
\begin{equation}
\label{eq:lbsep2class}
\ztwomin - \zonemax  \geq  \frac{1}{\sqrt{\volmax}} \left(  1 -  \sqrt{ \frac{ \mu \rdegmax \volbetmax }{\wminbyD}}  \right).
\end{equation}
\end{theorem}

The proof of Theorem \ref{thm:sep2class} is given in Appendix \ref{pf:thm_sep2class}. The proof is based on a variational characterization of the eigenvector of $\Lw - \mu \Lb$ corresponding to its smallest eigenvalue, whose elements are then bounded in terms of the parameters of the graph such as the diameters and volumes of its connected components.  

Theorem \ref{thm:sep2class} states that an embedding learnt with the supervised Laplacian eigenmaps method makes two classes linearly separable if the weight parameter $\mu$ is chosen sufficiently small. In particular, the theorem shows that, for any $0<\maroned<{\volmax}^{-1/2}$, a choice of the weight parameter $\mu$ satisfying
\begin{equation*}
0<\mu \leq \frac{\wminbyD}{\rdegmax \, \volbetmax} \left(1-\sqrt{\volmax} \, \maroned \right)^2
\end{equation*}
guarantees a separation of $ \ztwomin - \zonemax \geq \maroned$ between classes $1$ and $2$ at $d=1$. Here, we use the symbol $\maroned$ to denote the separation in the normalized coordinates $\z$. In practice, either one of the normalized eigenvectors $\z$ or the original eigenvectors $\y$ can be used for embedding the data. If the original eigenvectors $\y$ are used, due to the relation $\y = \Dw^{1/2} \z$, we can lower bound the separation as $\ytwomin - \yonemax \geq \sqrt{\dwmin} (\ztwomin - \zonemax)$ where $\dwmin = \min_i \dw(i) $. Thus, for any embedding of dimension $d\geq 1$, there exists a hyperplane that results in a linear separation with a margin $\mar$ of at least
\begin{equation*}
\mar \geq  \sqrt{ \frac{\dwmin}{\volmax} } \left(  1 -  \sqrt{ \frac{ \mu \rdegmax \volbetmax }{\wminbyD}}  \right).
\end{equation*}

Next, we comment on the dependence of the separation on $\mu$. The inequality in  \eqref{eq:lbsep2class} shows that the lower bound on the separation $\ztwomin - \zonemax$ has a variation of
%
$
O(1-\sqrt{\mu }) 
$
%
with the weight parameter $\mu$. The fact that the separation decreases with the increase in $\mu$ seems counterintuitive at first; this parameter weights the between-class dissimilarity in the objective function. This can be explained as follows. When $\mu$ is high, the algorithm tries to increase the distance between neighboring samples from different classes as much as possible by moving them away from the origin (remember that different classes are mapped to the positive and the negative sides of the real line). However,  since the normalized coordinate vector $\z$ has to respect the equality $\z^T \Dw \z=1$, the total squared norm of the coordinates cannot be arbitrarily large. Due to this constraint, setting $\mu$ to a high value causes the algorithm to map non-neighboring samples from different classes to nearby coordinates close to the origin. This occurs since the increase in $\mu$ reduces the impact of the first term $\y^T \Lw \y$ in the overall objective and results in an embedding with a weaker link between the samples of the same class. This causes a polarization of the data and eventually reduces the separation. Hence, the $\mu$ parameter should be carefully chosen and should not take too large values.



Theorem \ref{thm:sep2class} characterizes the separation at $d=1$ in terms of the distance between the supports of the two classes. Meanwhile, it is also interesting to determine the individual distances of the supports of the two classes to the origin. In the following corollary, we present a lower bound on the distance between the coordinates of any sample and the origin. 

\begin{corollary}
\label{cor:dist_orig_2class}

The distance between the supports of the first and the second classes and the origin in a one-dimensional embedding is lower bounded in terms of the separation between the two classes as 
\begin{equation*}
\min \{ | \zonemax | , \, \,  | \ztwomin |  \} \geq \half  \frac{ \rdegmin  }{ \rdegmax }  (\ztwomin - \zonemax) 
\end{equation*}
where 
\begin{equation*}
\begin{split}
\rdegmin &= \min_i \rdeg_i, 
\qquad
\rdegmax = \max_i \rdeg_i .
\end{split}
\end{equation*}

\end{corollary}

Corollary \ref{cor:dist_orig_2class} is proved in Appendix \ref{pf:cor_dist_orig_2class}. The proof is based on a Lagrangian formulation of the embedding as a constrained optimization problem, which then allows us to establish a link between the separation and the individual distances of class supports to the origin. The corollary states a lower bound on the portion of the overall separation lying in the negative or the positive sides of the real line. In particular, if the vertex degrees are equal for all samples in $\Gw$ and $\Gb$ (which is the case, for instance, if all vertices have the same number of neighbors and a constant weight of $\wij=1$ is assigned to the edges), since $\rdegmin = \rdegmax$, the portions of the overall separation in the positive and negative sides of the real line will be equal. Although the statement of Theorem \ref{thm:sep2class} is sufficient to show the existence of separating hyperplanes with positive margins for the embeddings of two classes, we will see in Section \ref{ssec:anly_multClass} that the separability with a hyperplane passing through the origin as in Corollary \ref{cor:dist_orig_2class} is a desirable property for the extension of these results to a multi-class setting.

\subsection{Separability bounds for multiple classes}
\label{ssec:anly_multClass}

In this section, we study the separability of the embeddings of multiple classes with the supervised Laplacian eigenmaps algorithm. In particular, we focus on a setting with multiple classes that can be grouped into several categories. The classes in each category are assumed to bear a relatively high resemblance within themselves, whereas the resemblance between classes from different categories is weaker. This is a scenario that is likely to be encountered in several practical data classification problems. 

In the following, we study the embeddings of multiple categorizable classes. The objective matrix $\Lw - \mu \Lb$ defining the embedding is close to a block-diagonal matrix if the between-category similarities are relatively low. Building on this observation, we present a result that links the separability of the overall embedding to the separability of the embeddings of each individual category with the same algorithm. Especially in a setting with many classes, this simplifies the problem for multiple classes and makes it possible to deduce information for the overall separation by studying the separation of the individual categories, which is easier to analyze. 

We consider data samples $\X = \{ \xii \}_{i=1}^\numsamp$ belonging to $\numclass$ different classes that can be categorized into $\numcat$ groups. For the purpose of our theoretical analysis, let us focus for a moment on the individual categories and consider the embedding of the samples in each category $\icat$ with the supervised Laplacian eigenmaps algorithm if the data graph was constructed only within the category $\icat$. Let $\Y^\icat$ be the $d^\icat$-dimensional embedding of category $\icat$. Assume that $\Y^\icat$ is separable with margin $\marc$. Then for any two classes $k, l$ in category $\icat$, there exists a hyperplane $\hyp_{kl}$ such that
\begin{equation}
\label{eq:cond_sep_nooff}
\begin{split}
\hyp_{kl}^T \, \y^\icat_i  &\geq \marc/2 \quad \,\,\,\, \text{     if  } \classi=k\\
\hyp_{kl}^T \, \y^\icat_i  &\leq - \marc/2 \quad \text{  if  } \classi=l
\end{split}
\end{equation}
where $\y^\icat_i$ is the $i$-th row of $\Y^\icat$ defining the coordinates of the $i$-th data sample in category $\icat$. Note that an offset of $b_{kl}=0$ is assumed here, i.e., the classes in each category are assumed to be separable with hyperplanes passing through the origin. While this is mainly for simplifying the analysis, the studied supervised Laplacian eigenmaps algorithm in fact computes embeddings having this property in practice (the theoretical guarantee for the two-class setting being provided in Corollary \ref{cor:dist_orig_2class}). 

Now let $\Lall = \Lw - \mu \Lb$ denote the $\numsamp \times \numsamp$ objective matrix defining the embedding of $\X$ with supervised Laplacian eigenmaps. Also, let $\Lcat = \Lwc - \mu \Lbc$ denote the block-diagonal objective matrix where the within-class and the between-class Laplacians $\Lwc$ and $\Lbc$ are obtained by restricting the graph edges to the ones within the categories. In other words, $\Lcat$ is obtained by removing the edges between all pairs of data samples belonging to different categories.

Let $\PerL=\Lall - \Lcat$ denote the component of $\Lall$ arising from the between-category data connections. In our analysis, we will treat this component $\PerL$ as a perturbation on the block-diagonal matrix $\Lcat$ and analyze the eigenvectors of $\Lall$ accordingly in order to study the separability of the embedding obtained with $\Lall$.

We will need a condition on the separation of the eigenvalues of $\Lcat$. Let $\eta$ denote the minimal separation (the smallest difference) between the eigenvalues of $\Lcat$
\begin{equation}
\label{defn:eigsep}
\eta := \min_{i\neq j}  | \lambda_i - \lambda_j |
\end{equation}
where $ \lambda_i $ are the eigenvalues of $\Lcat$ for $i=1, \dots, \numsamp$. For $\mu >0$ and a random sampling of data, the eigenvalues of $\Lcat$ are expected to be distinct\footnote{Note that the within-class and the between-class Laplacians $\Lwc$ and $\Lbc$ are normalized Laplacians; therefore, the constant vector is not an eigenvector and 0 is not a repeating eigenvalue.}; therefore, one can reasonably assume the minimal eigenvalue separation to be positive. The characterization of the behavior of the minimal separation of the eigenvalues depending on the graph properties is not within the scope of this study and remains as future work. 

We state below our main result about the separability of the embeddings of multiple categorizable classes.

\begin{theorem}
\label{thm:sep_mult_class_categ}

Let $\Lall = \Lw - \mu \Lb  \in \R^{\numsamp \times \numsamp} $ be the matrix representing the objective function of the supervised Laplacian eigenmaps algorithm with $\numclass$ classes categorizable into $\numcat$ groups. Assume that $\Lall$ is close to a block-diagonal objective matrix $\Lcat$ containing only within-category edges such that the perturbation $\PerL=\Lall - \Lcat$ is bounded as 
\begin{equation*}
\| \PerL \| < \frac{\eta }{2}
\end{equation*}
in terms of the minimal eigenvalue separation $\eta$ of the matrix $\Lcat$ defined in \eqref{defn:eigsep}. Let each category $\icat$ have a $d^\icat$-dimensional embedding $\Y^\icat$  separable with margin $\marc$ as in \eqref{eq:cond_sep_nooff}. We define the parameters
\begin{equation*}
\xi= \left(  1 - \frac{4  \| \PerL  \|^2 }{\eta^2}  \right)^{1/2}
\end{equation*}
and
\begin{equation*}
\zeta = \left( 2 - 2 \xi + 2 \sqrt{\numsamp(1-\xi^2) }  \right)^{1/2}.
\end{equation*}
Then there exists an embedding $\Y$ of dimension
$
d=\sum_{\icat=1}^\numcat d_\icat
$
consisting of the eigenvectors of the overall objective matrix $\Lall$ that is separable with a margin of at least
\begin{equation*}
\mar= \marc/\sqrt{2} - 2 \zeta
\end{equation*}
provided that $\zeta < \marc/(2\sqrt{2})$.
\end{theorem}

The proof of Theorem \ref{thm:sep_mult_class_categ} is given in Appendix \ref{pf:thm:sep_mult_class_categ}. The proof is based on first analyzing the separation of the embedding corresponding to the block-diagonal component $\Lcat$ of the objective matrix, and then lower bounding the separation of the original embedding in terms of the perturbation and the separation of the eigenvalues. In brief, the theorem says that if the classes are categorizable with sufficiently low between-category edge weights, and if the individual embedding of each category makes all classes in that category linearly separable, then in the embedding computed for the overall data graph with the supervised Laplacian eigenmaps algorithm, all pairs of classes (from same and different categories) are also linearly separable. This extends the linear separability of individual categories to the separability of all classes. The margin of the overall separation decreases at a rate of $O(\sqrt{1 -  \| \PerL \|})$ as the magnitude of the non-block-diagonal component $\| \PerL \|$ of the objective matrix increases.\footnote{From the definition of the parameters $\xi$,$\zeta$, and $\mar$ in Theorem \ref{thm:sep_mult_class_categ}, we have $\xi=O(\sqrt{1-\| \PerL\|^2})$, $\zeta=O(\sqrt{1-\xi})$, $\mar=O(1-\zeta)$. It follows that $\mar = O(1-( 1- \sqrt{1-\|  \PerL\|^2} \,)^{1/2}) \approx O(\sqrt{1 - \| \PerL \|})$.}

The dimension of the separable embedding is given by the sum of the dimensions of the individual embeddings of the categories that ensure the linear separability within each category, hence, the dimension required for linear separability must be linearly proportional to the number of categories. In order to compute the exact value of the number of dimensions required for linear separability, one needs the knowledge of the number of dimensions that ensures the separability within each category. Nevertheless, the provided result is still interesting as the theoretical or numerical analysis of individual categories is often easier than the analysis of the whole data set, since the number of classes in a particular category is more limited.\footnote{For instance, we have theoretically shown in Section \ref{ssec:anly_twoClass} that one dimension is sufficient for obtaining a linearly separable embedding of two classes. While we do not provide a theoretical analysis for more than two classes, we have experimentally observed that data becomes linearly separable at two dimensions when the number of classes is three or four.} Note that, one can also interpret the theorem by considering each class as a different category. However, in this case the edges between samples of different classes must have sufficiently low weights for the applicability of the theorem, i.e., the non-block diagonal component $L^{nc}$ of the Laplacian must be sufficiently small. The examination of the general problem of embedding data with multiple non-categorizable classes and no assumptions of the edge weights between different classes seems to be a more challenging problem and remains as a future direction to study.


\section{Experimental Results}
\label{sec:exp_results}

In this section, we present results on synthetical and real data sets. We compare several supervised manifold learning methods and study their performances in relation with our theoretical results.

\begin{figure}[t]
\begin{center}
     \subfigure[Quadratic surfaces]
       {\label{fig:quadsurf_two_class}\includegraphics[height=3.5cm]{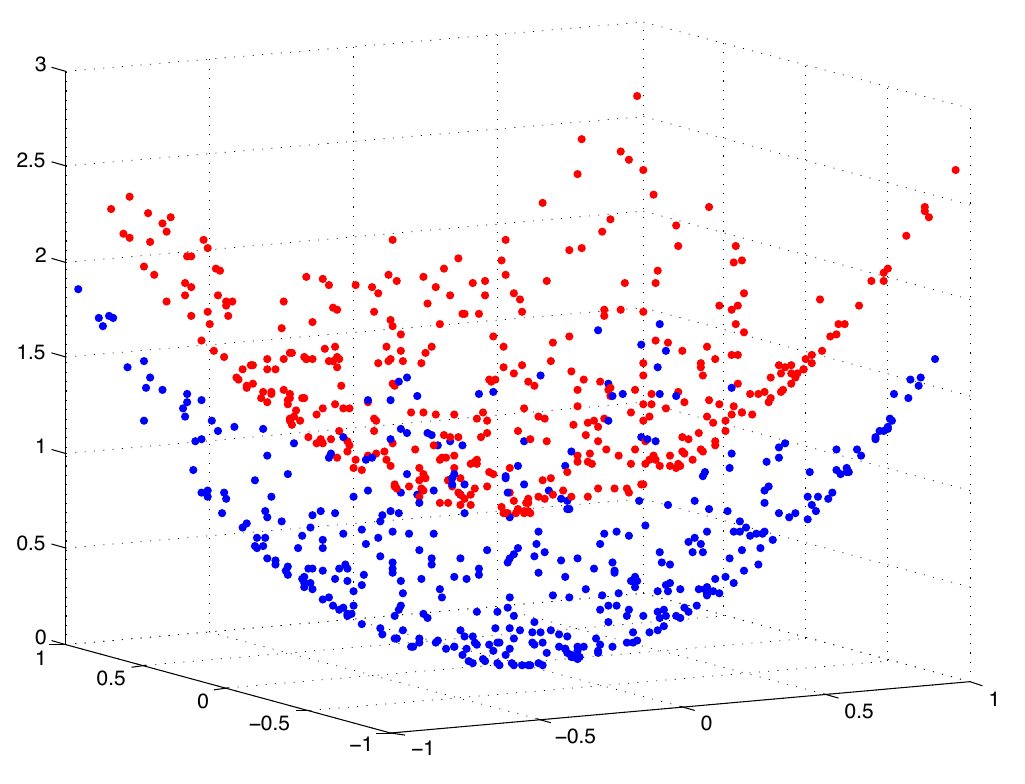}}
     \subfigure[Swissrolls]
       {\label{fig:swisssurf_two_class}\includegraphics[height=3.5cm]{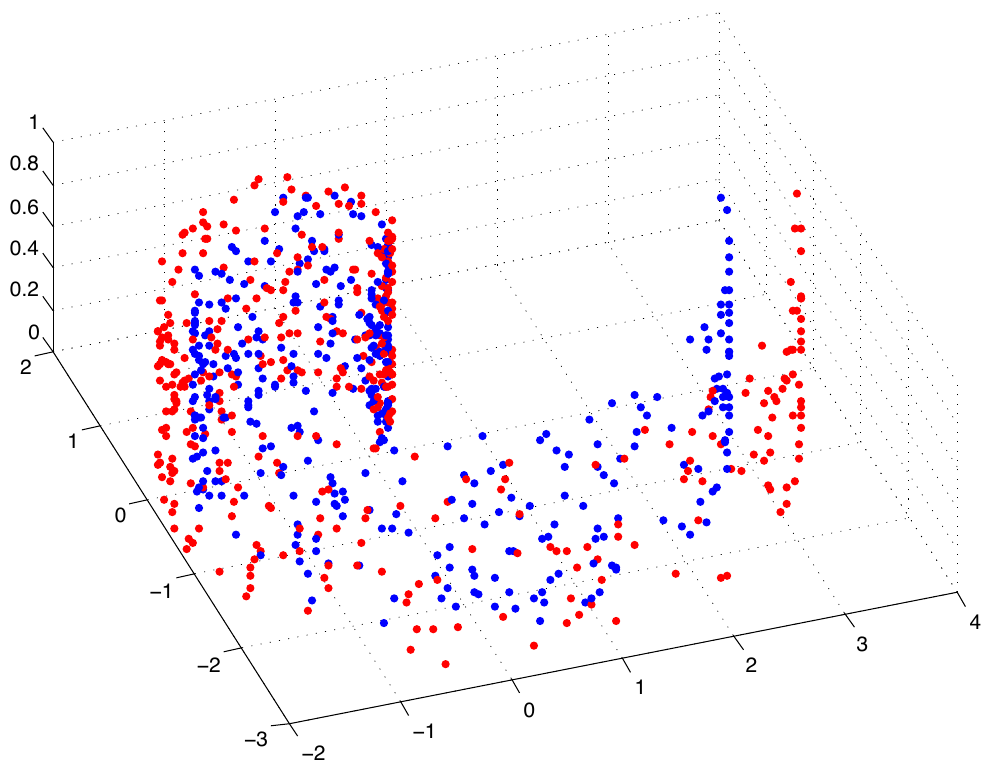}}
     \subfigure[Spheres]
       {\label{fig:sphsurf_two_class}\includegraphics[height=3.5cm]{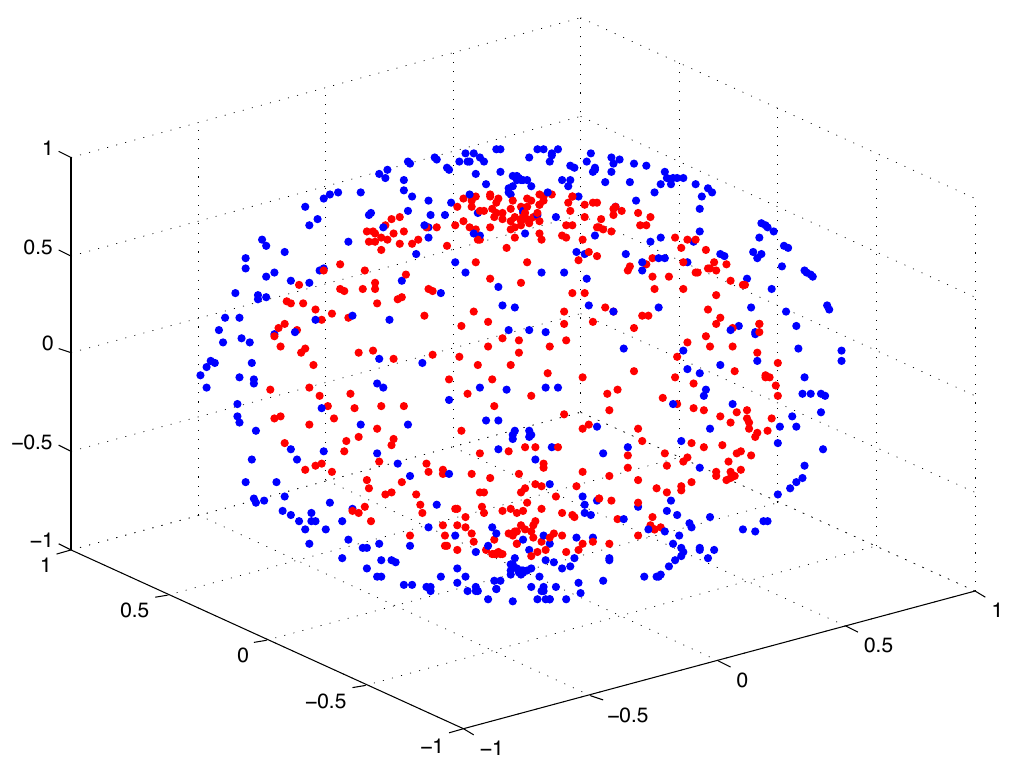}}
 \end{center}
 \caption{Data sampled from two-dimensional synthetical surfaces. Red and blue colors represent two different classes.}
 \label{fig:surfaces_two_class}
\end{figure}

\begin{figure}[]
\begin{center}
     \subfigure[1-D embedding]
       {\label{fig:quademb_1d_twoclass}\includegraphics[height=3.5cm]{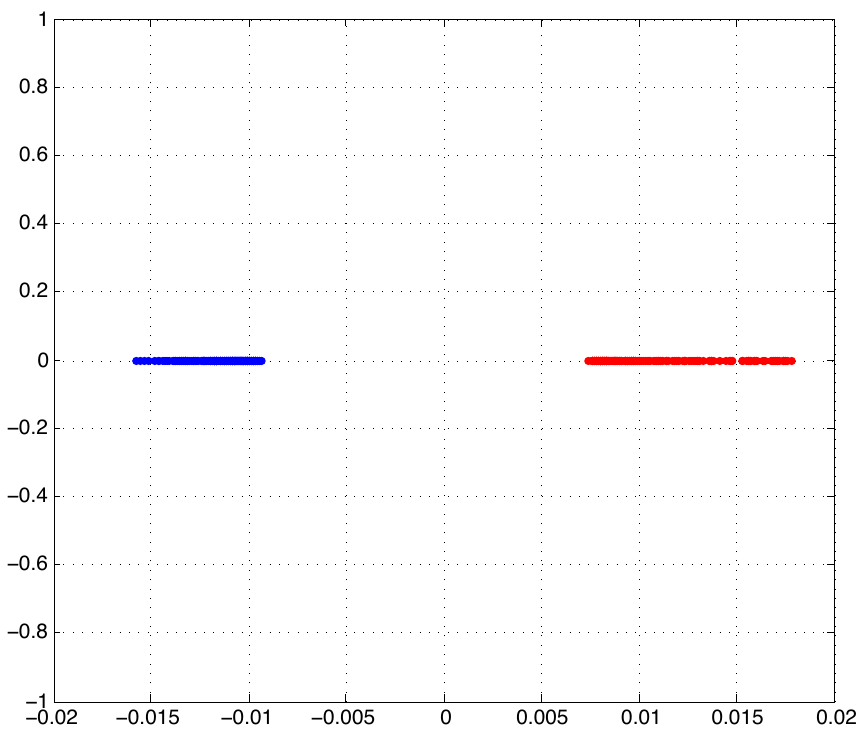}}
     \subfigure[2-D embedding]
       {\label{fig:quademb_2d_twoclass}\includegraphics[height=3.5cm]{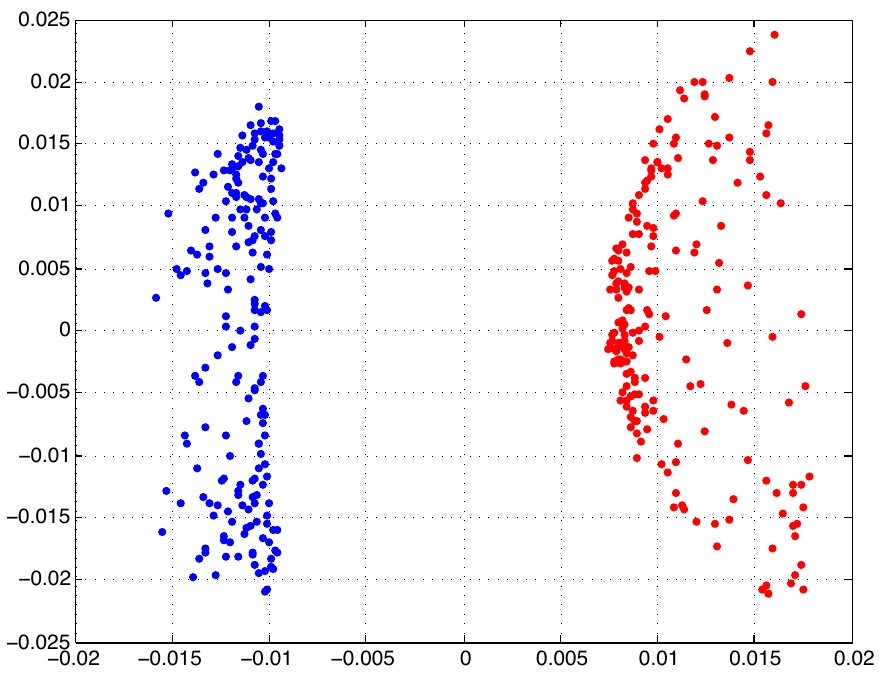}}
     \subfigure[3-D embedding]
       {\label{fig:quademb_3d_twoclass}\includegraphics[height=3.5cm]{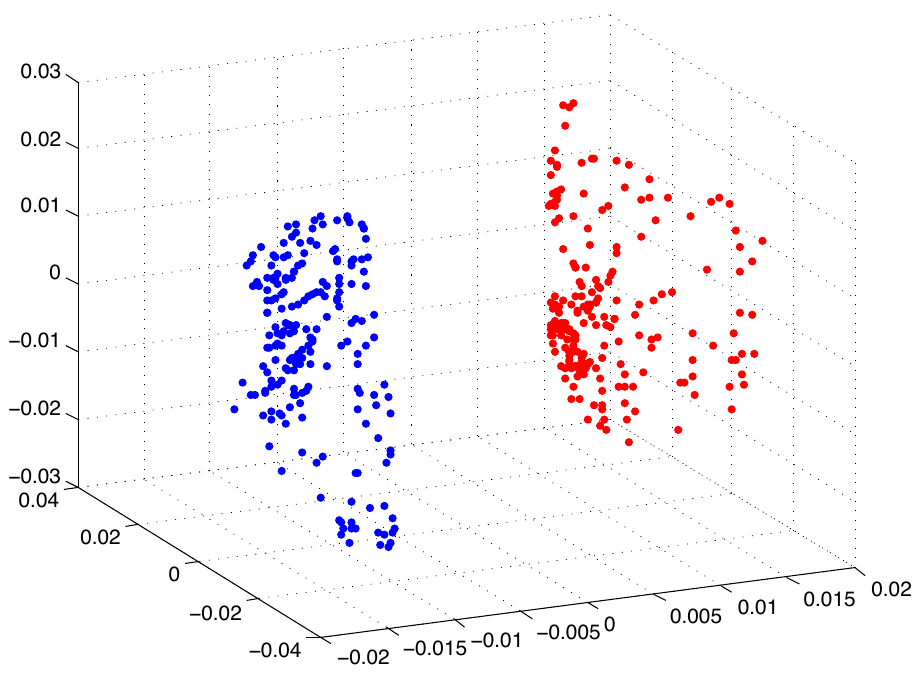}}
 \end{center}
 \caption{Supervised Laplacian embeddings of data sampled from quadratic surfaces.}
 \label{fig:embeddings_two_class}
\end{figure}

\subsection{Separability of embeddings with supervised manifold learning}

We first present results on synthetical data in order to study the embeddings obtained with supervised dimensionality reduction. We test the supervised Laplacian eigenmaps algorithm in a setting with two classes. We generate samples from two nonintersecting and linearly nonseparable surfaces in $\R^3$ that represent two different classes. We experiment on three different types of surfaces; namely, quadratic surfaces, Swiss rolls and spheres. The data sampled from these surfaces are shown in Figure \ref{fig:surfaces_two_class}. We choose $\numsamp=200$ samples from each class. We construct the graph $\Gw$ by connecting each sample to its $K$-nearest  neighbors from the same class, where $K$ is chosen between $20$ and $30$. The graph $\Gb$ is constructed similarly, where each sample is connected to its $K/5$ nearest neighbors from the other class. The graph weights are determined as a Gaussian function of the distance between the samples. The embeddings are then computed by minimizing the objective function in (\ref{eq:supLap_formal}). The one-dimensional, two-dimensional, and three-dimensional embeddings obtained for the quadratic surface are shown in Figure \ref{fig:embeddings_two_class}, where the weight parameter is taken as $\mu=0.57$ (to have a visually clear embedding for the purpose of illustration). Similar results are obtained on the Swiss roll and the spherical surface. One can observe that the data samples that were initially linearly nonseparable become linearly separable when embedded with the supervised Laplacian eigenmaps algorithm. The two classes are mapped to different (positive or negative) sides of the real line in Figure \ref{fig:quademb_1d_twoclass} as predicted by Lemma \ref{lem:twoclass_diffsign}.  The separation in the 2-D and 3-D embeddings in Figure \ref{fig:embeddings_two_class} is close to the separation obtained with the 1-D embedding.

\begin{figure}[]
\begin{center}
     \subfigure[Experimental value of the separation $\mar$]
       {\label{fig:exp_two_class}\includegraphics[height=5cm]{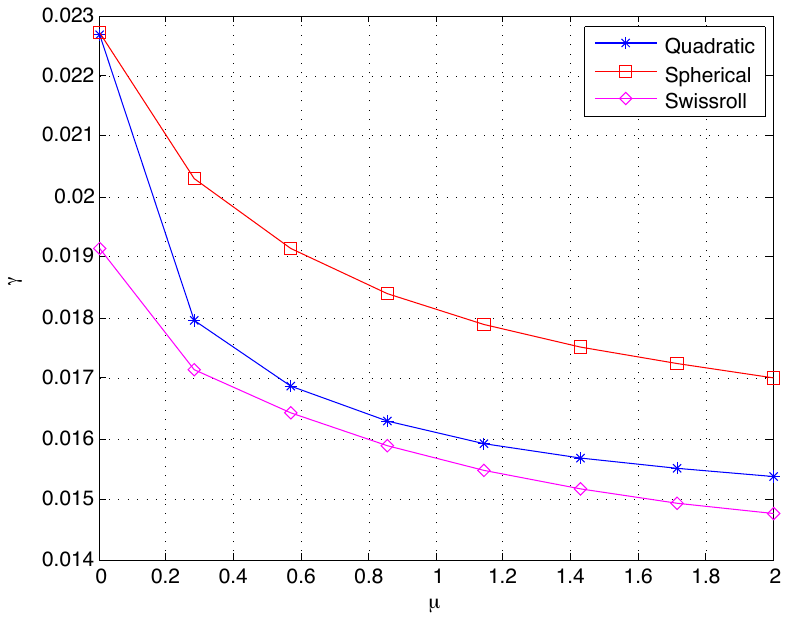}}
     \subfigure[Theoretical upper bound for $\mu$ that guarantees a separation of at least $\mar$]
       {\label{fig:theo_two_class}\includegraphics[height=5cm]{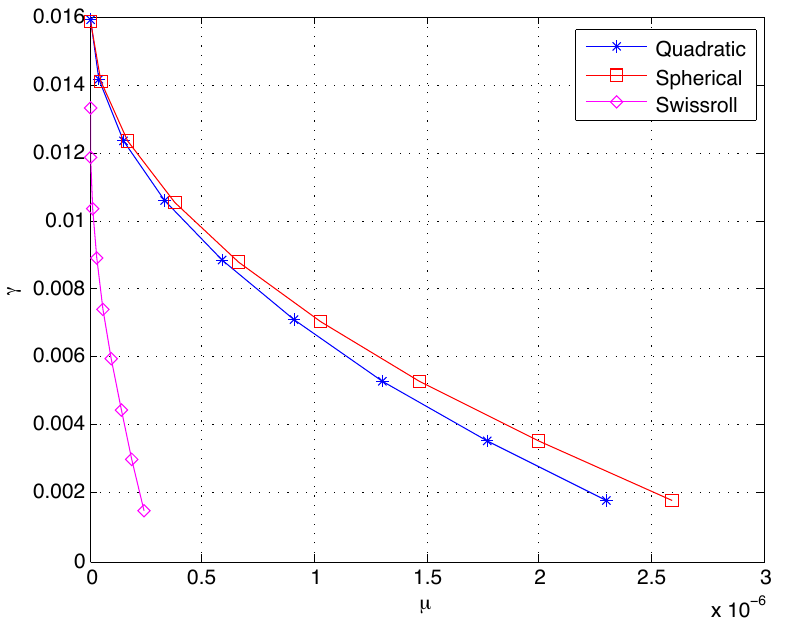}}
 \end{center}
 \caption{Variation of the separation $\mar$ between the two classes with the parameter $\mu$ for the synthetic data sets}
 \label{fig:separation_two_class}
\end{figure}

We then compute and plot the separation obtained at different values of $\mu$. Figure \ref{fig:exp_two_class} shows the experimental value of the separation $\mar = \ztwomin - \zonemax$ obtained with the 1-D embedding for the three types of surfaces. Figure \ref{fig:theo_two_class} shows the theoretical upper bound for $\mu$ in Theorem \ref{thm:sep2class} that guarantees a separation of at least $\mar$. Both the experimental value and the theoretical bound for the separation $\mar$ decrease with the increase in the parameter $\mu$. This is in agreement with  \eqref{eq:lbsep2class}, which predicts a decrease of $O(1-\sqrt{\mu})$ in the separation with respect to $\mu$. The theoretical bound for the separation is seen to decrease at a relatively faster rate with $\mu$ for the Swiss roll data set. This is due to the particular structure of this data set with a nonuniform sampling density where the sampling is sparser away from the spiral center. The parameter $\wminbyD$ then takes a small value, which consequently leads to a fast rate of decrease for the separation due to \eqref{eq:lbsep2class}. Comparing Figures \ref{fig:exp_two_class} and \ref{fig:theo_two_class}, one observes that the theoretical bounds for the separation are numerically more pessimistic than their experimental values, which is a result of the fact that our results are obtained with a worst-case analysis. Nevertheless, the theoretical bounds capture well the actual variation of the separation margin with $\mu$.

\subsection{Classification performance of supervised manifold learning algorithms}
\label{ssec:exp_oos}

\begin{figure}[t]
\begin{center}
     \subfigure[COIL-20 object data set]
       {\label{fig:compare_methods_coil20}\includegraphics[height=5cm]{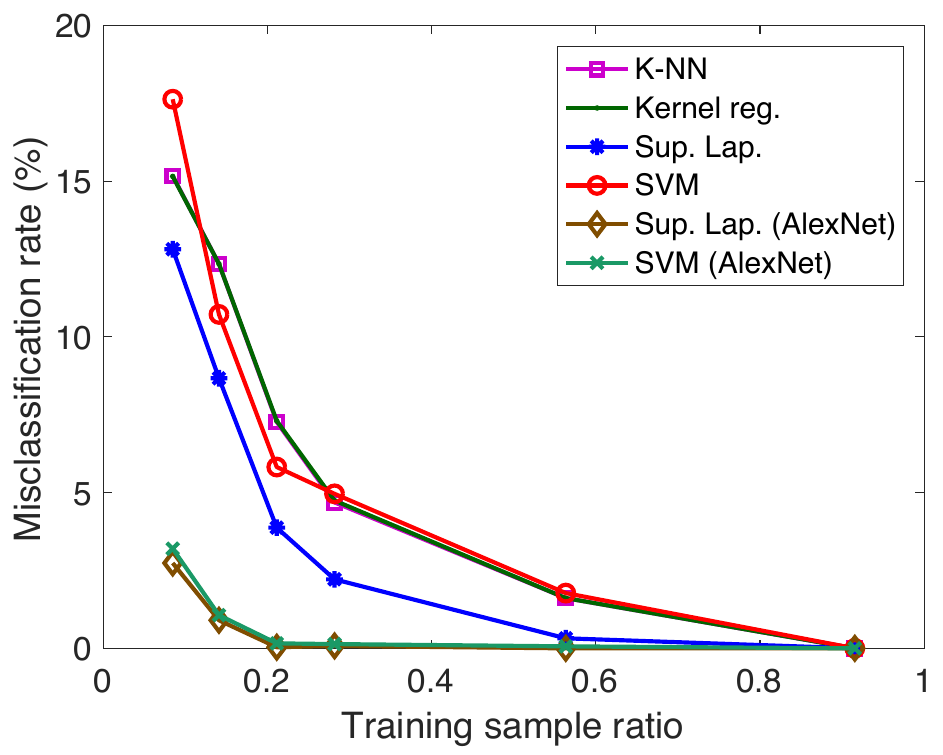}}   
     \subfigure[ETH-80 object data set]
       {\label{fig:compare_methods_eth80}\includegraphics[height=5cm]{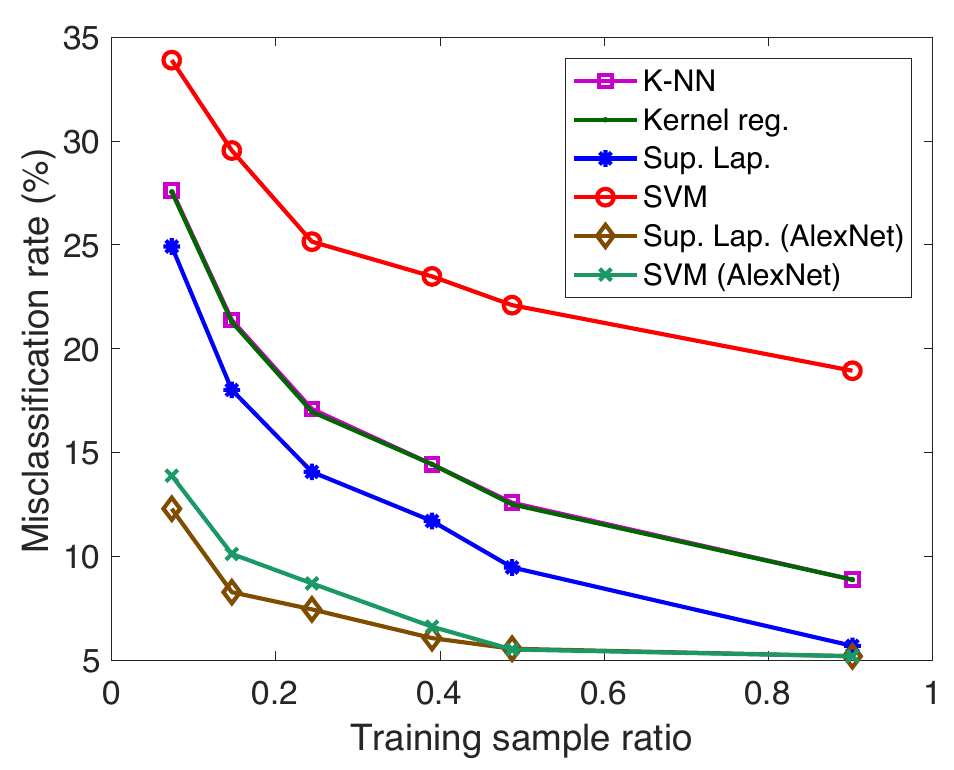}}       
     \subfigure[Yale face data set]
       {\label{fig:compare_methods_yalefull}\includegraphics[height=5cm]{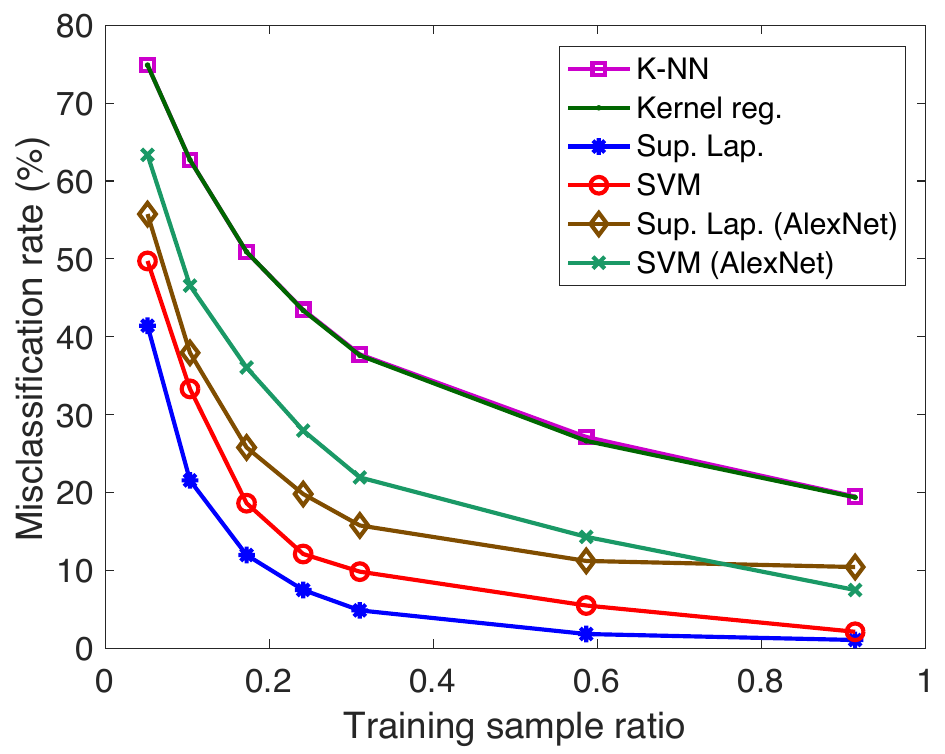}}   
     \subfigure[Reduced Yale face data set]
       {\label{fig:compare_methods_yalered}\includegraphics[height=5cm]{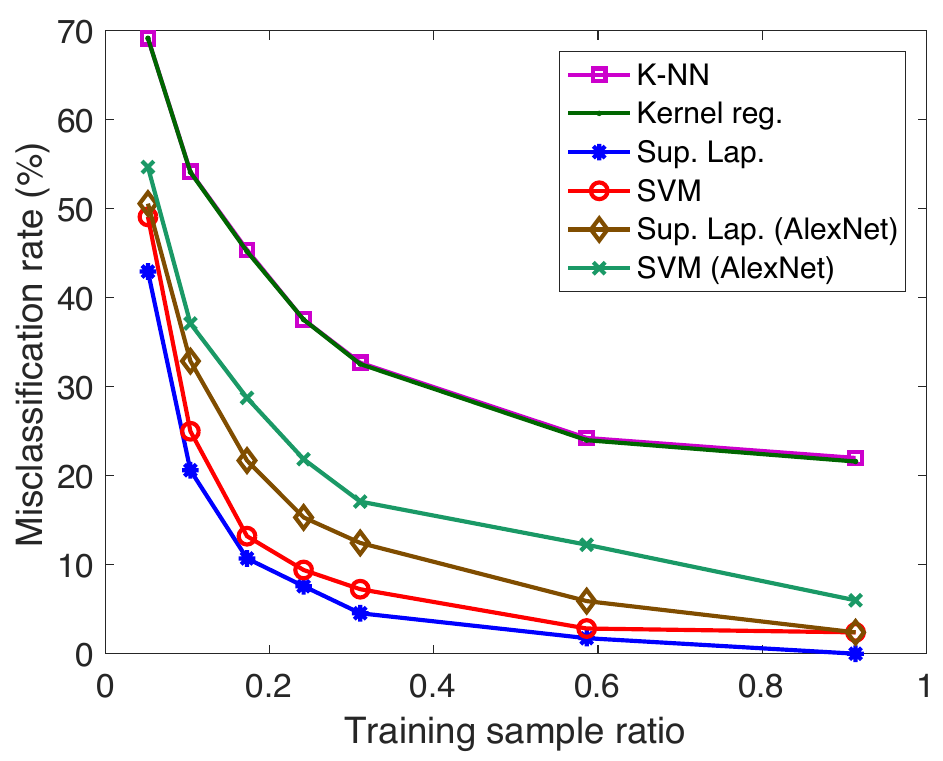}}
     \subfigure[MNIST data set]
       {\label{fig:compare_methods_mnist}\includegraphics[height=5cm]{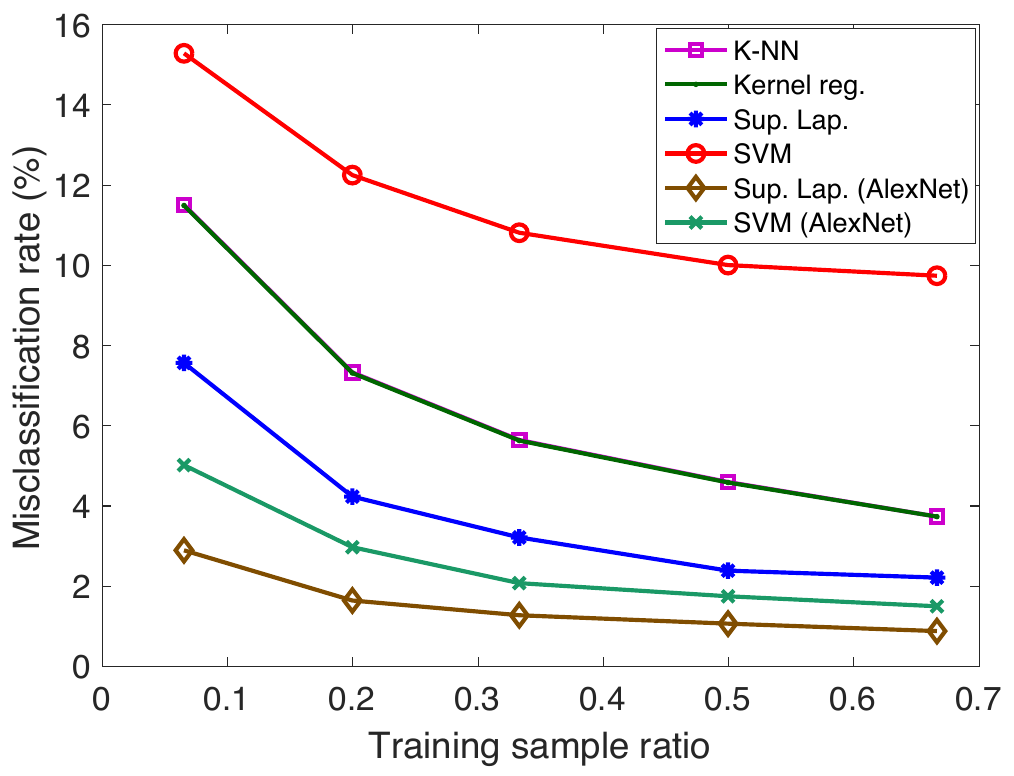}}      
 \end{center}
 \caption{Comparison of the performance of several supervised classification methods}
 \label{fig:compare_methods}
\end{figure}

We now study the overall performance of classification obtained in a setting with supervised manifold learning, where the out-of-sample generalization is achieved with smooth RBF interpolators. We evaluate the theoretical results of Section \ref{sec:class_anly_supml} on several real data sets: the COIL-20 object database \citep{NeneNM96}, the Yale face database \citep{GeBeKr01}, the ETH-80  object database \citep{LeibeS03}, and the MNIST handwritten digit database \citep{LeCunBBH98}. The COIL-20, Yale face, ETH-80, and MNIST databases contain a total of 1420,  2204, 3280, and 70046 images from 20, 38, 8, and 10 image classes respectively. The images in the COIL-20, Yale and ETH-80 data sets are converted to greyscale, normalized, and downsampled to a resolution of respectively $32 \times 32$, $20 \times 17$, and $20 \times 20$ pixels.

\subsubsection{Comparison of supervised manifold learning to baseline classifiers}

We first compare the performance of supervised manifold learning with some reference classification methods. The performances of SVM, K-NN, kernel regression, and the supervised Laplacian eigenmaps methods are evaluated and compared. Figure \ref{fig:compare_methods} reports the results obtained on the COIL-20 data set, the ETH-80 data set, the Yale data set, a subset of the Yale data set consisting  of its first 10 classes (reduced Yale data set), and the MNIST data set. The SVM, K-NN, and kernel regression algorithms are applied in the original domain and their hyperparameters are optimized with cross-validation. In the supervised Laplacian eigenmaps method, the embedding of the training images into a low-dimensional space is computed. Then, an out-of-sample interpolator with Gaussian RBFs is constructed that maps the training samples to their embedded coordinates as described in Section \ref{ssec:oos_rbf}. Test samples are mapped to the low-dimensional domain via the RBF interpolator and the class labels of test samples are estimated via nearest-neighbor classification in the low-dimensional domain. The supervised Laplacian eigenmaps and the SVM methods are also tested over an alternative representation of the image data sets based on deep learning. The images are provided as input to the pretrained AlexNet convolutional neural network proposed in \citep{KrizhevskySH12}, and the activation values at the second fully connected layer are used as the feature representations of the images. The feature representations of training and test images are then provided to the supervised Laplacian eigenmaps and the SVM methods. The plots in Figure \ref{fig:compare_methods} show the variation of the misclassification rate of test samples in percentage with the ratio of the number of training samples in the whole data set. The results are the average of 5 repetitions of the experiment with different random choices for the training and test samples.

The results in Figure \ref{fig:compare_methods} show that the best results are obtained with the supervised Laplacian eigenmaps algorithm in general. The performances of the algorithms improve with the number of training images as expected. In the COIL-20 and ETH-80 object data sets, the supervised Laplacian eigenmaps and the SVM algorithms yield significantly smaller error when applied to the feature representations of the images obtained with deep learning. Meanwhile, in the Yale face data set these two methods perform better on raw image intensity maps. This can be explained with the fact that the AlexNet model may be more successful in extracting useful features for object images rather than face images as it is trained on many common object and animal classes. It is interesting to compare Figures \ref{fig:compare_methods_yalefull} and \ref{fig:compare_methods_yalered}. While the performances of the supervised Laplacian eigenmaps and the SVM methods are closer in the reduced version of the Yale database with 10 classes, the performance gap between the supervised Laplacian eigenmaps method and the other methods is larger for the full data set with 38 classes. This can be explained with the fact that the linear separability of different classes degrades as the number of classes increases, thus causing a degradation in the performance of the classifiers in comparison. Meanwhile, the performance of the supervised Laplacian eigenmaps method is not much affected by the increase in the number of classes. The K-NN and kernel regression classifiers are seen to give almost the same performance in the plots in Figure \ref{fig:compare_methods}. The number of neighbors is set as $K=1$ for the K-NN algorithm in these experiments, where it has been observed to attain its best performance;  and the scale parameter of the kernel regression algorithm is optimized to get the best accuracy, which has turned out to take relatively small values. Hence the performances of these two classifiers practically correspond to that of the nearest-neighbor classifier in the original domain.


\subsubsection{Variation of the error with algorithm parameters and sample size}

\begin{figure}[t]
\begin{center}
     \subfigure[COIL-20 object data set]
       {\label{fig:error_N_coil}\includegraphics[height=5cm]{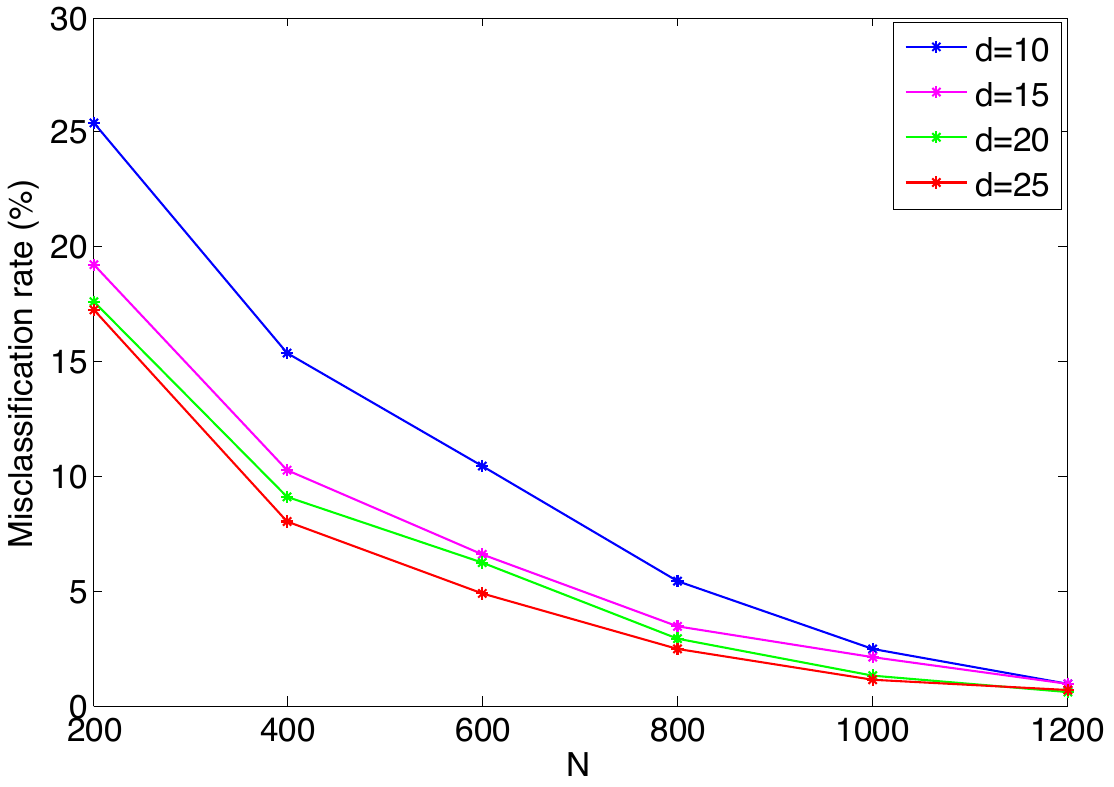}}
     \subfigure[ETH-80 object data set]
       {\label{fig:error_N_eth}\includegraphics[height=5cm]{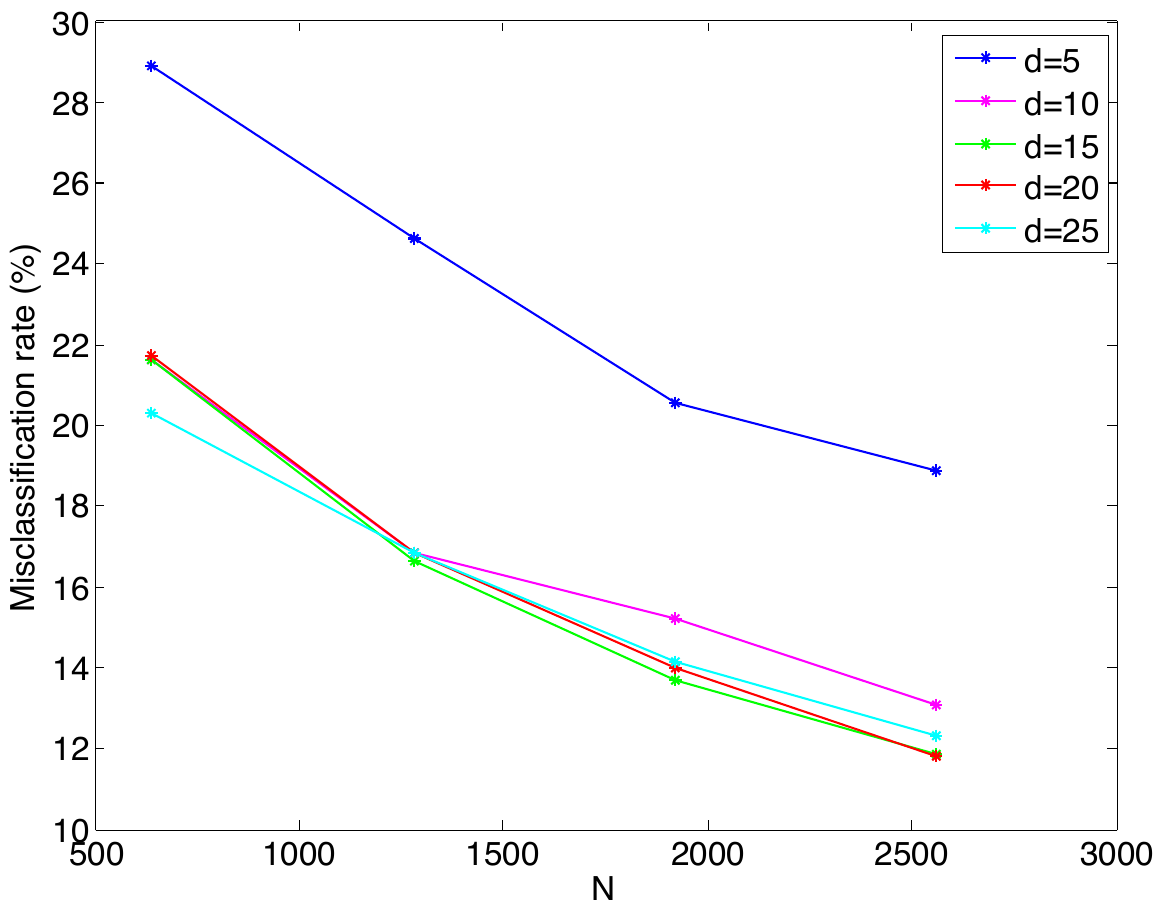}}  
     \subfigure[Yale face data set]
       {\label{fig:error_N_yale}\includegraphics[height=5cm]{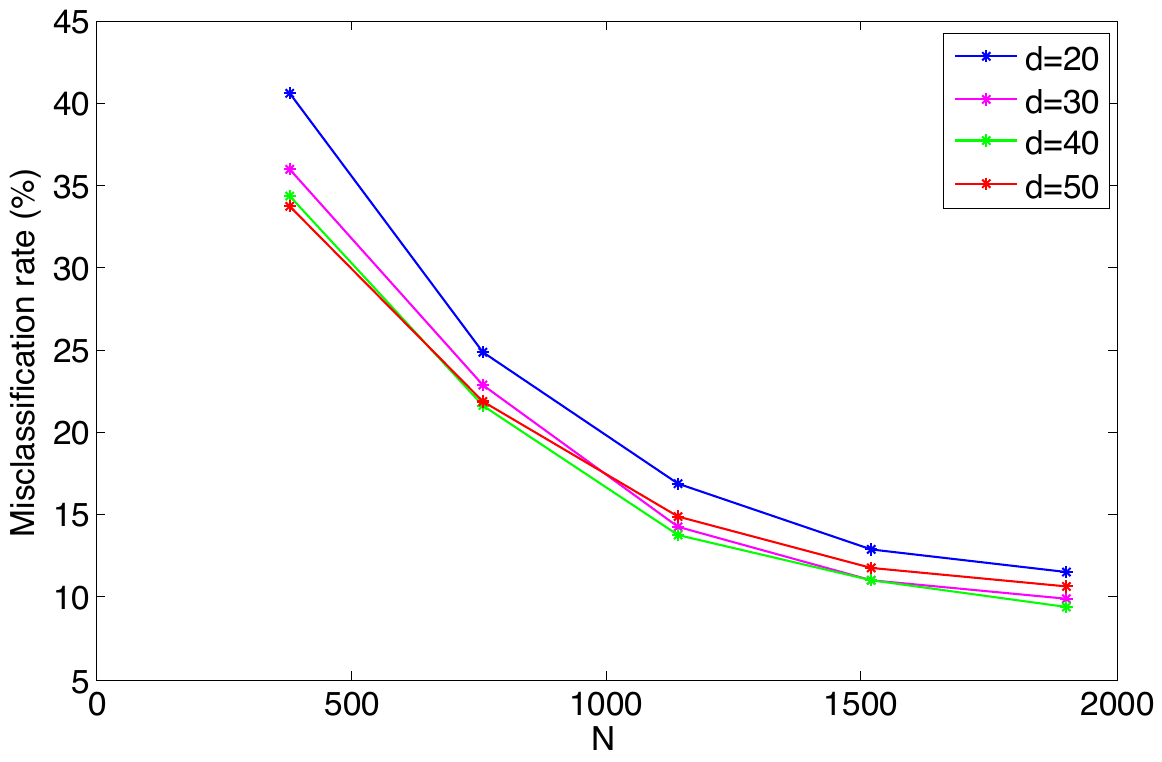}}      
 \end{center}
 \caption{Variation of the misclassification rate with the number of training samples}
 \label{fig:error_N}
\end{figure}

\begin{figure}[]
\begin{center}
     \subfigure[COIL-20 object data set]
       {\label{fig:error_d_coil}\includegraphics[height=5cm]{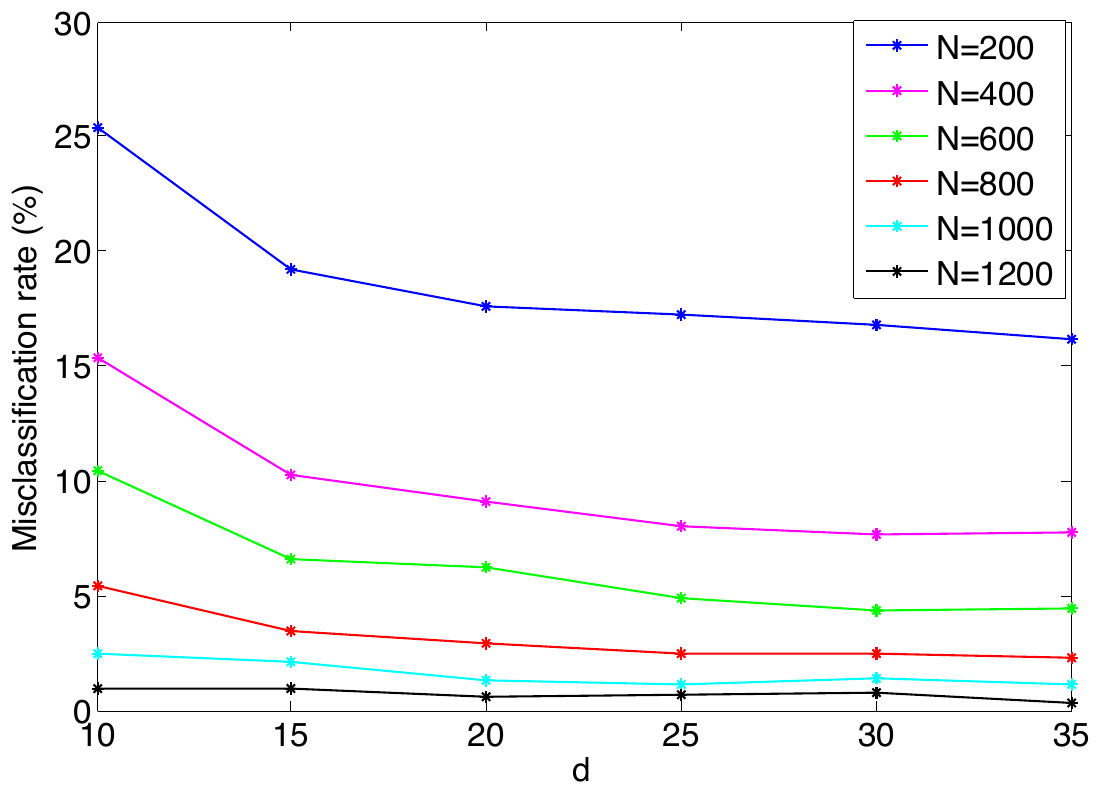}}
     \subfigure[ETH-80 object data set]
       {\label{fig:error_d_eth}\includegraphics[height=5cm]{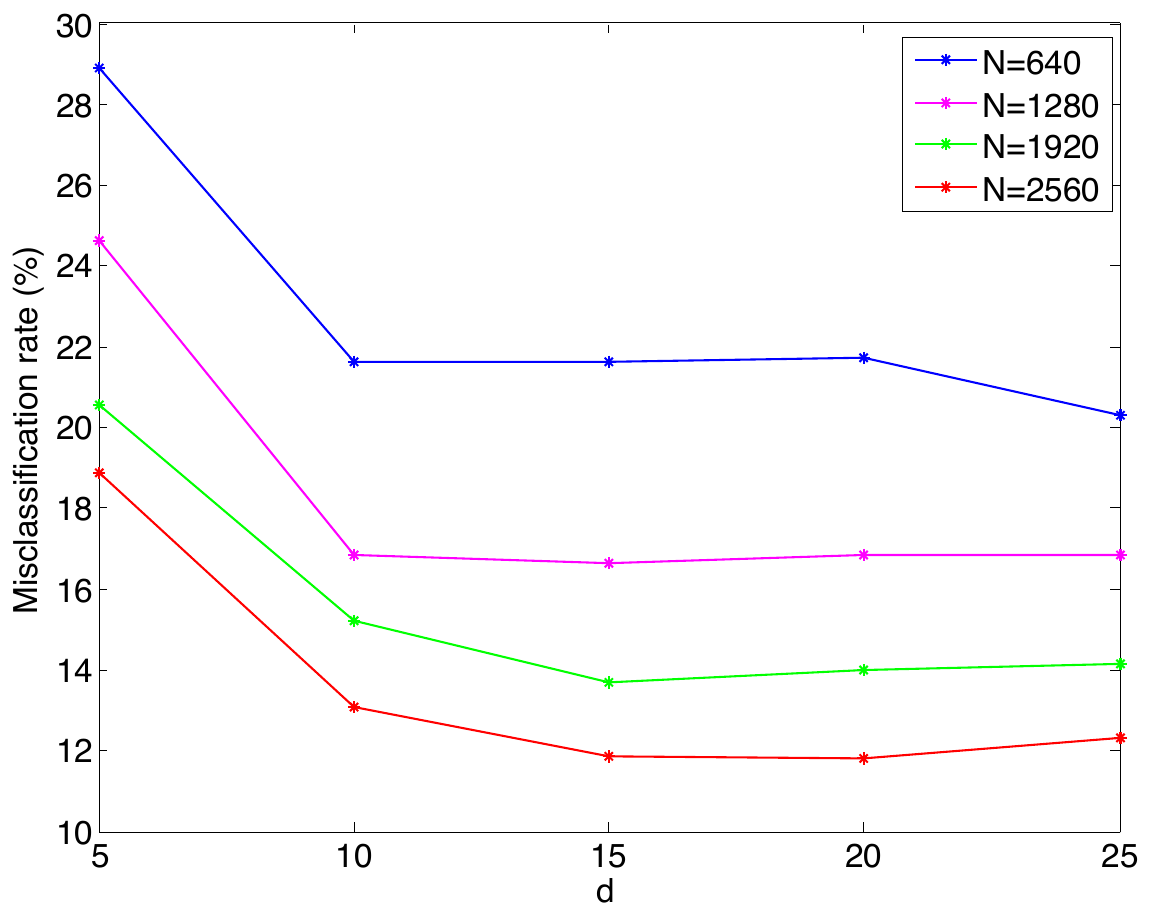}}
     \subfigure[Yale face data set]
       {\label{fig:error_d_yale}\includegraphics[height=5cm]{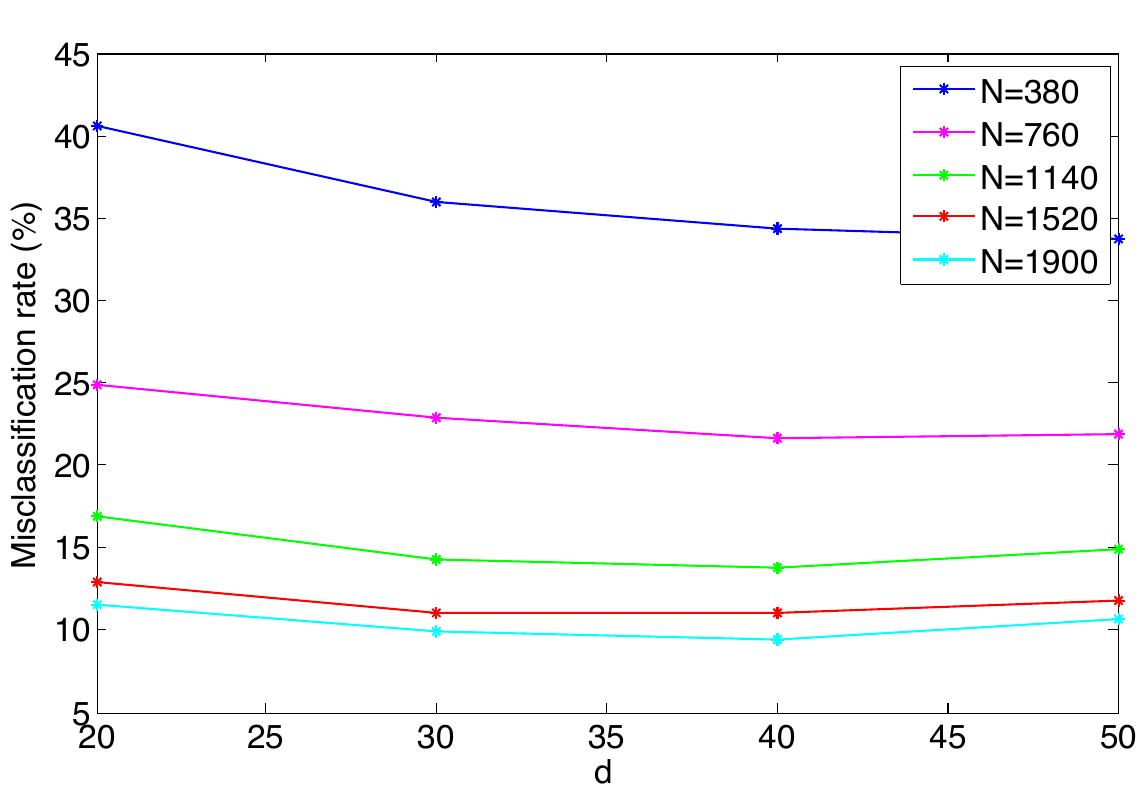}}
 \end{center}
 \caption{Variation of the misclassification rate with the dimension of the embedding}
 \label{fig:error_d}
\end{figure}

We first study the evolution of the classification error with the number of training samples. Figures \ref{fig:error_N_coil}- \ref{fig:error_N_yale} show the variation of the misclassification rate of test samples with respect to the total number of training samples $\numsamp$ for the COIL-20, ETH-80 and Yale data sets. Each curve in the figure shows the errors obtained at a different value of the dimension $d$ of the embedding. The decrease in the misclassification rate with the number of training samples is in agreement with the results in Section \ref{sec:class_anly_supml} as expected. 

The results of Figure \ref{fig:error_N} are replotted in Figure \ref{fig:error_d}, where the variation of the misclassification rate is shown with respect to the dimension $d$ of the embedding at different $\numsamp$ values. It is observed that there may exist an optimal value of the dimension that minimizes the misclassification rate. This can be interpreted in light of the conditions \eqref{eq:cond_sep_interp} and \eqref{eq:cond_sep_nn_interp} in Theorems \ref{thm:acc_cl_rbfint} and \ref{thm:acc_cl_nn_rbfint}, which impose a lower bound on the separability margin $\mar_Q$ in terms of the dimension $d$ of the embedding. In the supervised Laplacian eigenmaps algorithm, the first few dimensions are critical and effective for separating different classes. The decrease in the error with the increase in the dimension for small values of $d$ can be explained with the fact that the separation increases with $d$ at small $d$, thereby satisfying the conditions \eqref{eq:cond_sep_interp},  \eqref{eq:cond_sep_nn_interp}. Meanwhile, the error may stagnate or increase if the dimension $d$ increases beyond a certain value, as the separation does not necessarily increase at the same rate. 

We then examine the variation of the misclassification rate with the separation.  We obtain embeddings at different separation values $\mar$ by changing the parameter $\mu$ of the supervised Laplacian eigenmaps algorithm. Figure \ref{fig:error_gamma} shows the variation of the misclassification rate with the separation $\mar$. Each curve is obtained at a different value of the scale parameter $\sigma$ of the RBF kernels. It is seen that the misclassification rate decreases in general with the separation for small $\mar$ values. This is in agreement with our results, as the conditions \eqref{eq:cond_sep_interp}, \eqref{eq:cond_sep_nn_interp} require the separation to be higher than a threshold. On the other hand, the possible increase in the error at relatively large values of the separation is due to the following. These parts of the plots are obtained at very small $\mu$ values, which typically result in a deformed embedding with a degenerate geometry. The deformation of structure at too small values of $\mu$ may cause the interpolation function to be irregular and hence result in an increase in the error. The tradeoff between the separation and the interpolation function regularity is further studied in Section \ref{ssec:study_cond_exp}.

Finally, Figure \ref{fig:error_sigma} shows the relation between the misclassification error and the scale parameter $\sigma$ of the Gaussian RBF kernels. Each curve is obtained at a different value of the $\mu$ parameter. The optimum value of the scale parameter minimizing the misclassification error can be observed in most experiments. These results confirm the findings of Section \ref{ssec:scale_optim}, suggesting that there exists a unique value of $\sigma$ that minimizes the left hand side of the conditions \eqref{eq:cond_sep_interp}, \eqref{eq:cond_sep_nn_interp}, which probabilistically guarantee the correct classification of data.

\begin{figure}[t]
\begin{center}
     \subfigure[COIL-20 object data set]
       {\label{fig:error_gamma_coil}\includegraphics[height=5cm]{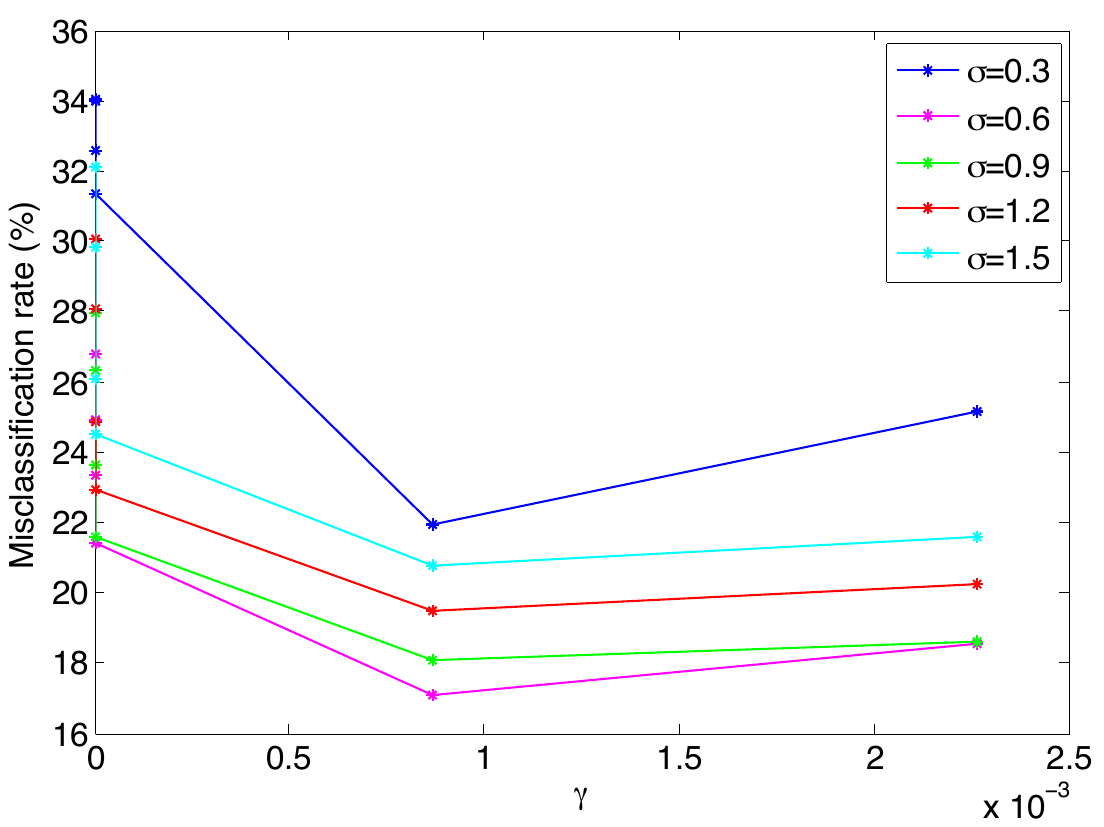}}
     \subfigure[ETH-80 object data set]
       {\label{fig:error_gamma_eth}\includegraphics[height=5cm]{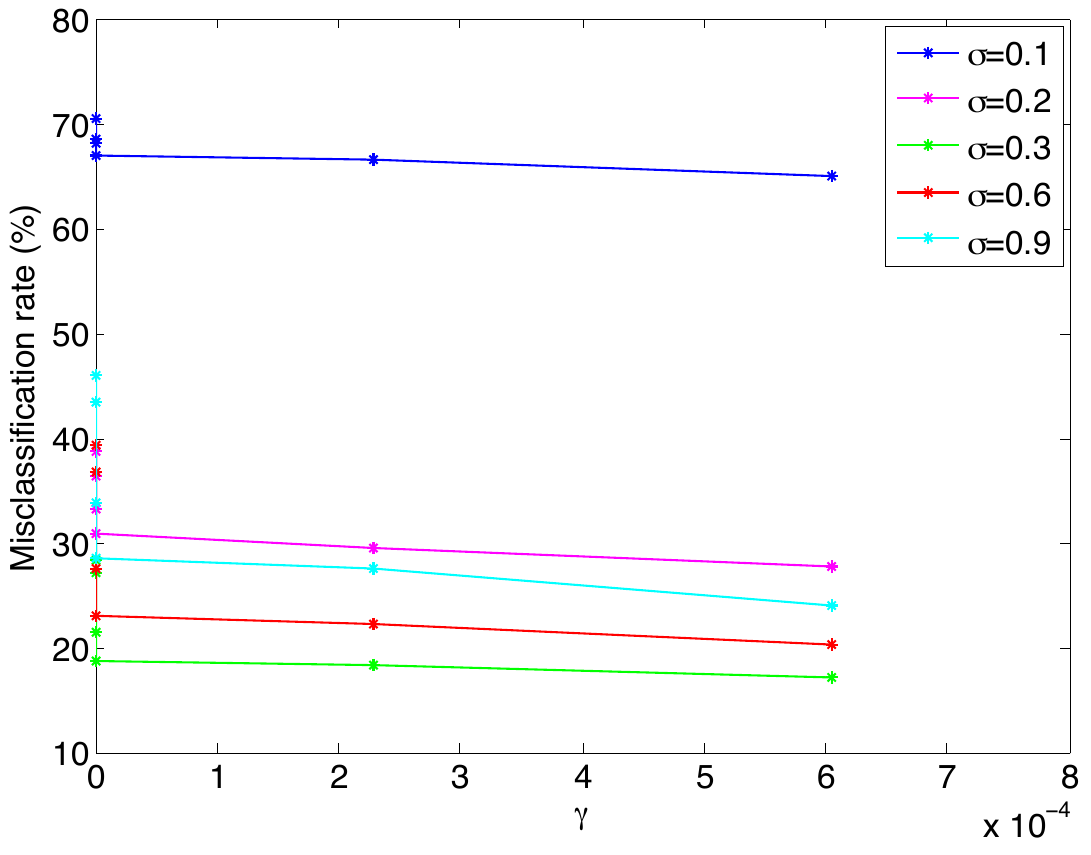}}
     \subfigure[Yale face data set]
       {\label{fig:error_gamma_yale}\includegraphics[height=5cm]{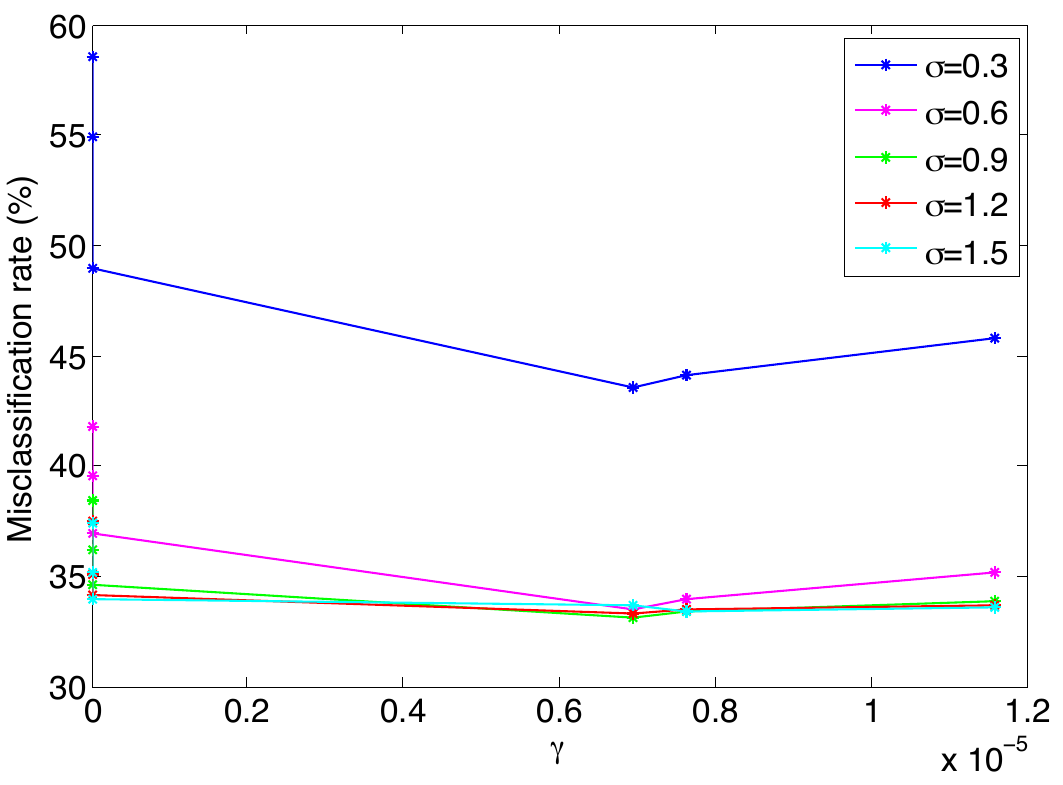}}
 \end{center}
 \caption{Variation of the misclassification rate with the separation}
 \label{fig:error_gamma}
\end{figure}

\begin{figure}[t]
\begin{center}
     \subfigure[COIL-20 object data set]
       {\label{fig:error_sigma_coil}\includegraphics[height=5cm]{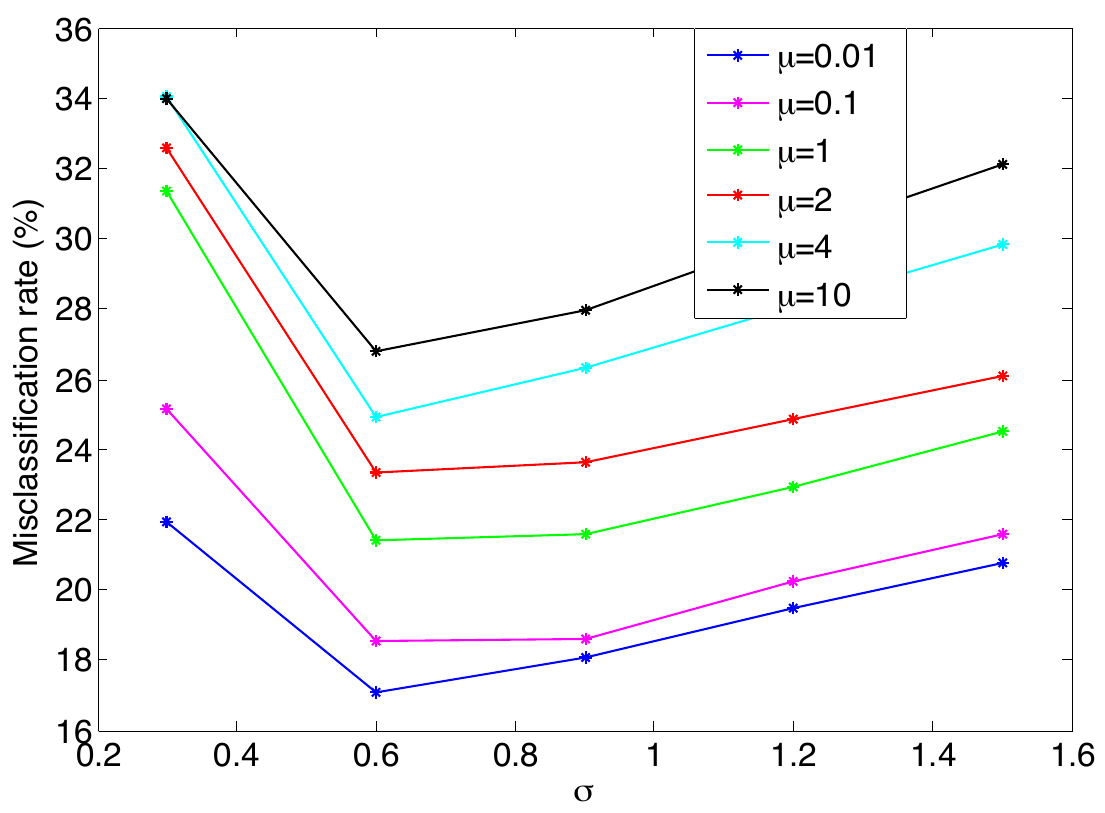}}
     \subfigure[ETH-80 object data set]
       {\label{fig:error_sigma_eth}\includegraphics[height=5cm]{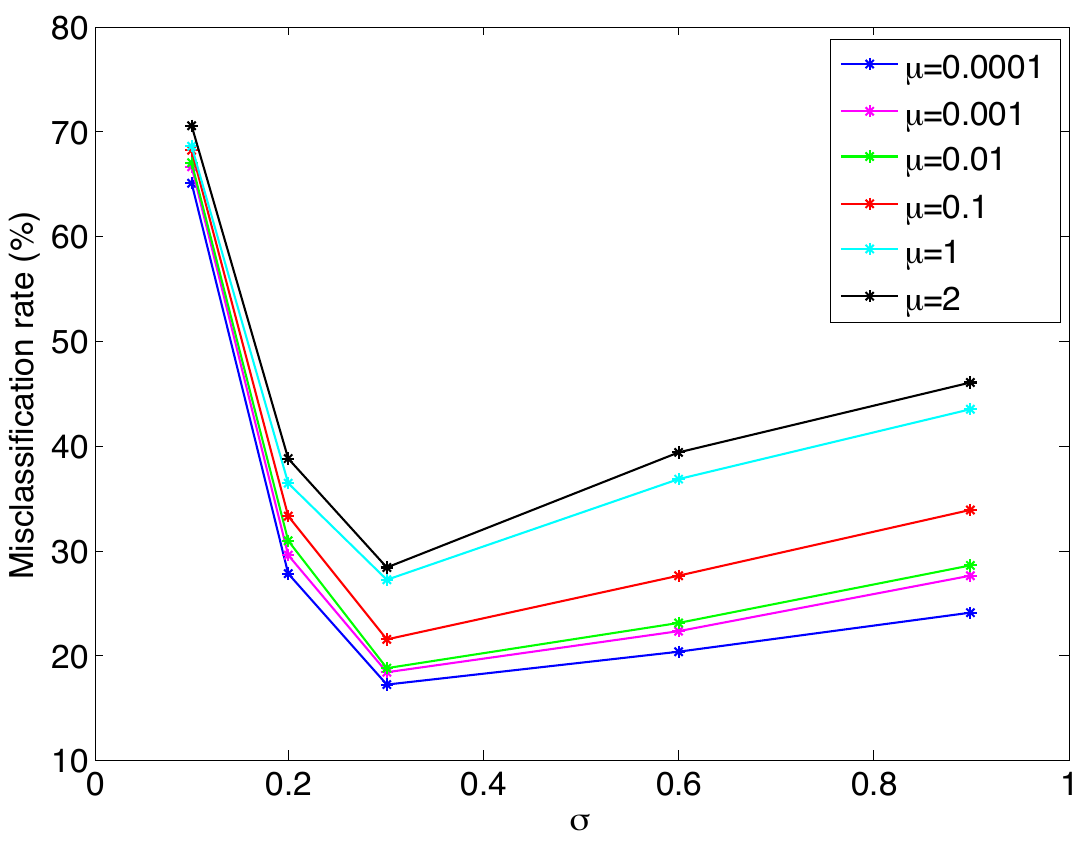}}
     \subfigure[Yale face data set]
       {\label{fig:error_sigma_yale}\includegraphics[height=5cm]{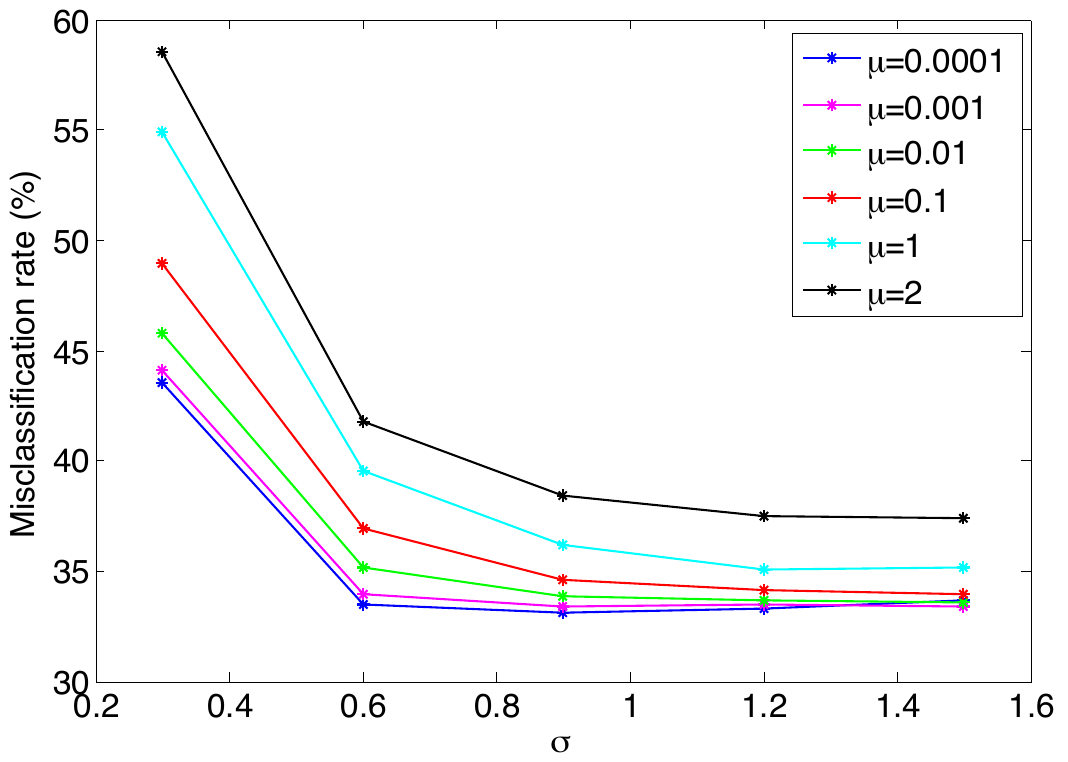}}
 \end{center}
 \caption{Variation of the misclassification rate with the scale parameter}
 \label{fig:error_sigma}
\end{figure}


\subsubsection{Performance analysis of several supervised manifold learning algorithms}
\label{ssec:study_cond_exp}

Next, we compare several supervised manifold learning methods. We aim to interpret the performance differences of different types of embeddings in light of our theoretical results in Section \ref{ssec:oos_rbf}. First, remember from Theorem \ref{thm:acc_cl_rbfint} that the condition 
\begin{equation}
\label{eq:cond_param_condition}
\sqrt{d} \, \cbnd \, (\Lconf \delta + \epsilon) \leq \mar / 2
\end{equation}
needs to be satisfied (or, equivalently the condition \eqref{eq:cond_sep_nn_interp} from Theorem \ref{thm:acc_cl_nn_rbfint}) in order for the generalization bounds to hold. This preliminary condition basically states that a compromise must be achieved between the regularity of the interpolation function, captured via the terms $\cbnd$ and $\Lconf$, and the separation $\mar$ of the embedding of training samples, in order to bound the misclassification error. In other words, increasing the separation too much in the embedding of training samples does not necessarily lead to good classification performance if the interpolation function has poor regularity. 

Hence, when comparing different embeddings in the experiments of this section, we define a condition parameter given by
\[
\frac{\sqrt{d} \cbnd \Lconf}{\mar}
\]
which represents the ratio of the left and right hand sides of \eqref{eq:cond_param_condition} (by fixing the probability parameters $\delta$ and $\epsilon$). Setting the Lipschitz constant of the Gaussian RBF kernel as $\Lconf = \sqrt{2} e^{-\half} \sigma^{-1}$ (see Section \ref{ssec:scale_optim} for details), we can equivalently define the condition parameter as
\begin{equation}
\label{eq:defn_cond_param}
\kappa = \frac{\sqrt{d} \cbnd }{\sigma \mar}
\end{equation}
and study this condition parameter for the supervised dimensionality methods in comparison. Note that a smaller condition parameter means that the necessary conditions of Theorems \ref{thm:acc_cl_rbfint} and \ref{thm:acc_cl_nn_rbfint} are more likely to be satisfied, hence hinting at the expectation of a better classification accuracy.

We compare the following supervised embeddings:

\begin{itemize}
\item Supervised Laplacian eigenmaps embedding obtained by solving \eqref{eq:supLap_formal}:
\[
\min_{Y} \tr(\Y^T \Lw \Y) - \mu \,  \tr(\Y^T \Lb \Y) \text{ subject to } \Y^T \Y = I
\]
\item Fisher embedding\footnote{We use a nonlinear version of the formulation in \citep{Wang09} by removing the constraint that the embedding be given by a linear projection of the data.}, obtained by solving 
\begin{equation}
\label{eq:form_fisher_embed}
\max_{Y} \frac{\tr(Y^T L_b Y)}{\tr(Y^T L_w Y)}.
\end{equation}
\item Label encoding, which maps each data sample to its label vector of the form $[ 0 \ 0 \dots 1 \dots 0]$, where the only nonzero entry corresponds to its class.
\end{itemize}

The label encoding method is included in the experiments to provide a reference, which can also be regarded as a degenerate supervised manifold learning algorithm that provides maximal separation between data samples from different classes. In all of the above methods the training samples are embedded into the low-dimensional domain, and test samples are mapped via Gaussian RBF interpolators and assigned labels via nearest neighbor classification in the low-dimensional domain. The scale parameter $\sigma$ of the RBF kernel is set to a reference value in each dataset within the typical range $[0.5, 1]$ where the best accuracy is attained. We have fixed the weight parameter as $\mu=0.01$ in all setups, and set the dimension of the embedding as equal to the number of classes. In order to study the properties of the interpolation function in relation with the condition parameter in \eqref{eq:defn_cond_param}, we also test the supervised Laplacian eigenmaps and the label encoding methods under RBF interpolators with high scale parameters, which are chosen as a few times the reference $\sigma$ value giving the best results. Finally, we also include in the comparisons a regularized version of the supervised Laplacian eigenmaps embedding by controlling the magnitude of the interpolation function.

The results obtained on the COIL-20, ETH-80, Yale and reduced Yale data sets are reported respectively in Figures \ref{fig:error_cond_COIL20}-\ref{fig:error_cond_yalered}. In each figure, panel (a) shows the misclassification rates of the embeddings and panel (b) shows the condition parameters of the embeddings at different total number of training samples ($N$). The logarithm of the condition parameter is plotted for ease of visualization.

\begin{figure}[t]
\begin{center}
     \subfigure[]
       {\label{fig:error_embeddings_coil20}\includegraphics[height=5cm]{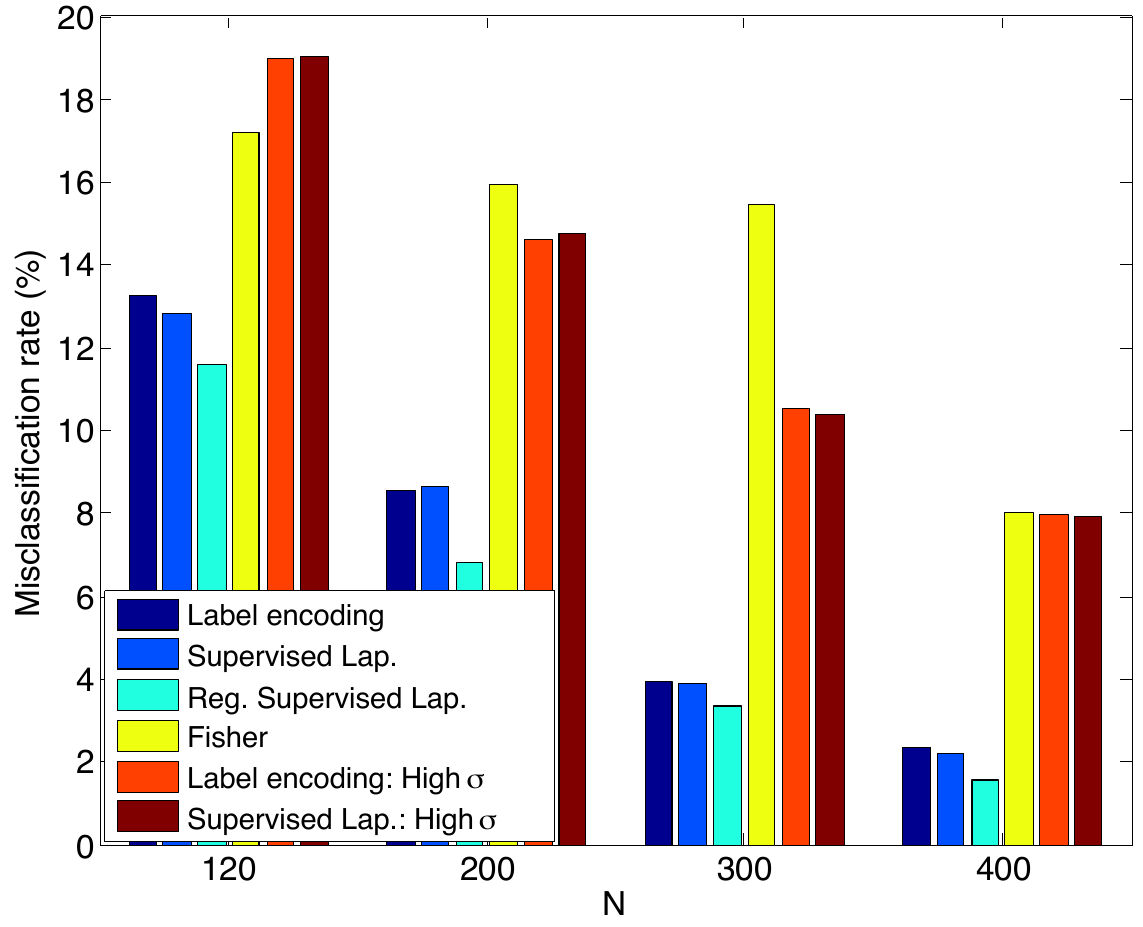}}
     \subfigure[]
       {\label{fig:cond_embeddings_coil20}\includegraphics[height=5cm]{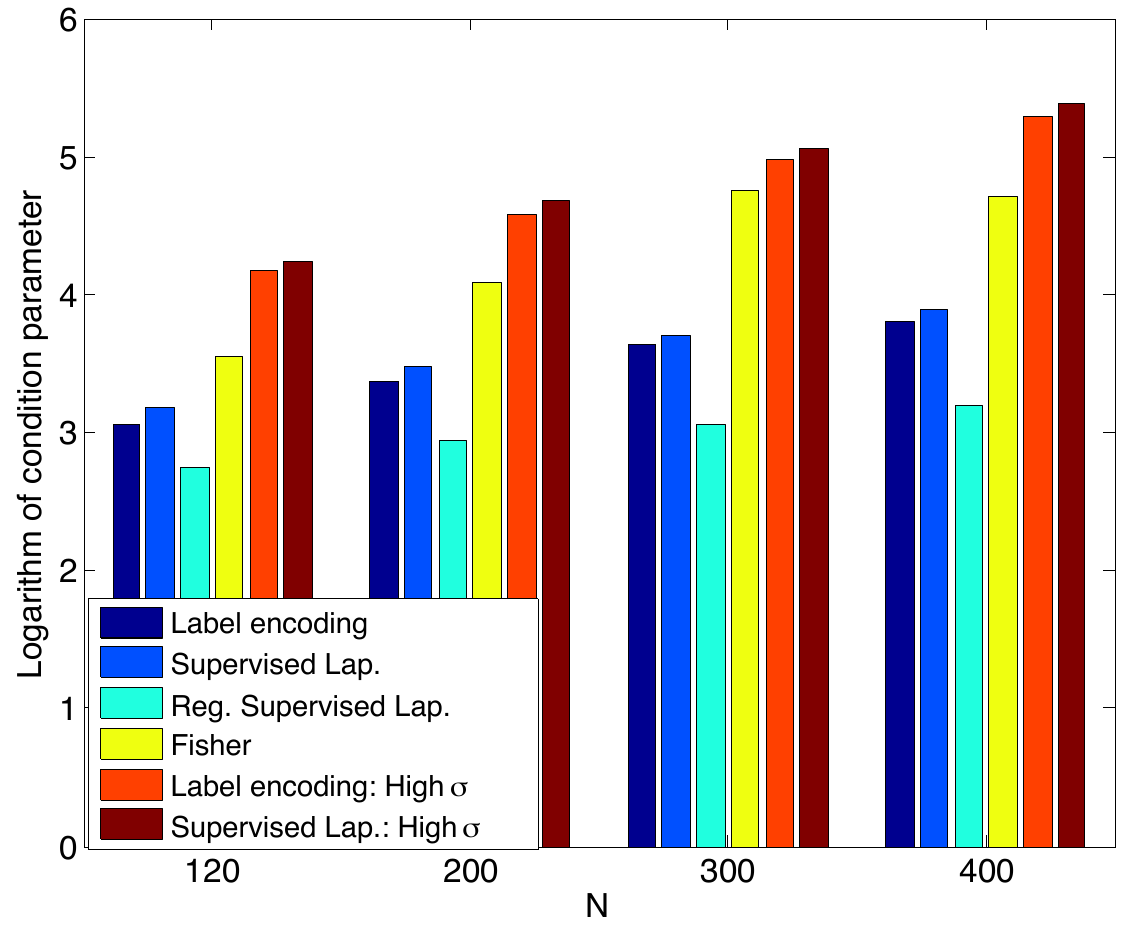}}
 \end{center}
 \caption{Misclassification rates and the condition parameters of the embeddings for the COIL-20 object data set}
 \label{fig:error_cond_COIL20}
\end{figure}

\begin{figure}[]
\begin{center}
     \subfigure[]
       {\label{fig:error_embeddings_eth80}\includegraphics[height=5cm]{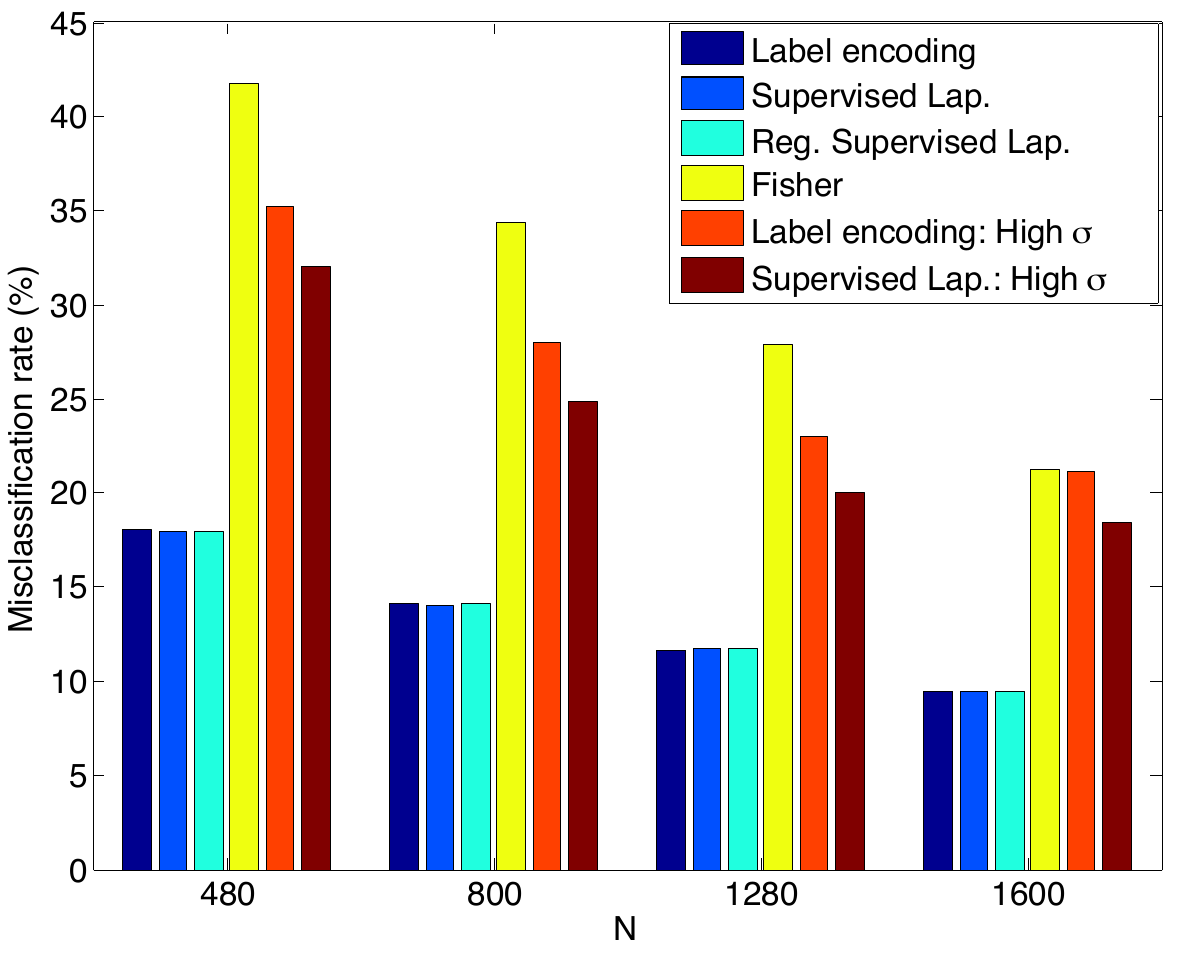}}
     \subfigure[]
       {\label{fig:cond_embeddings_eth80}\includegraphics[height=5cm]{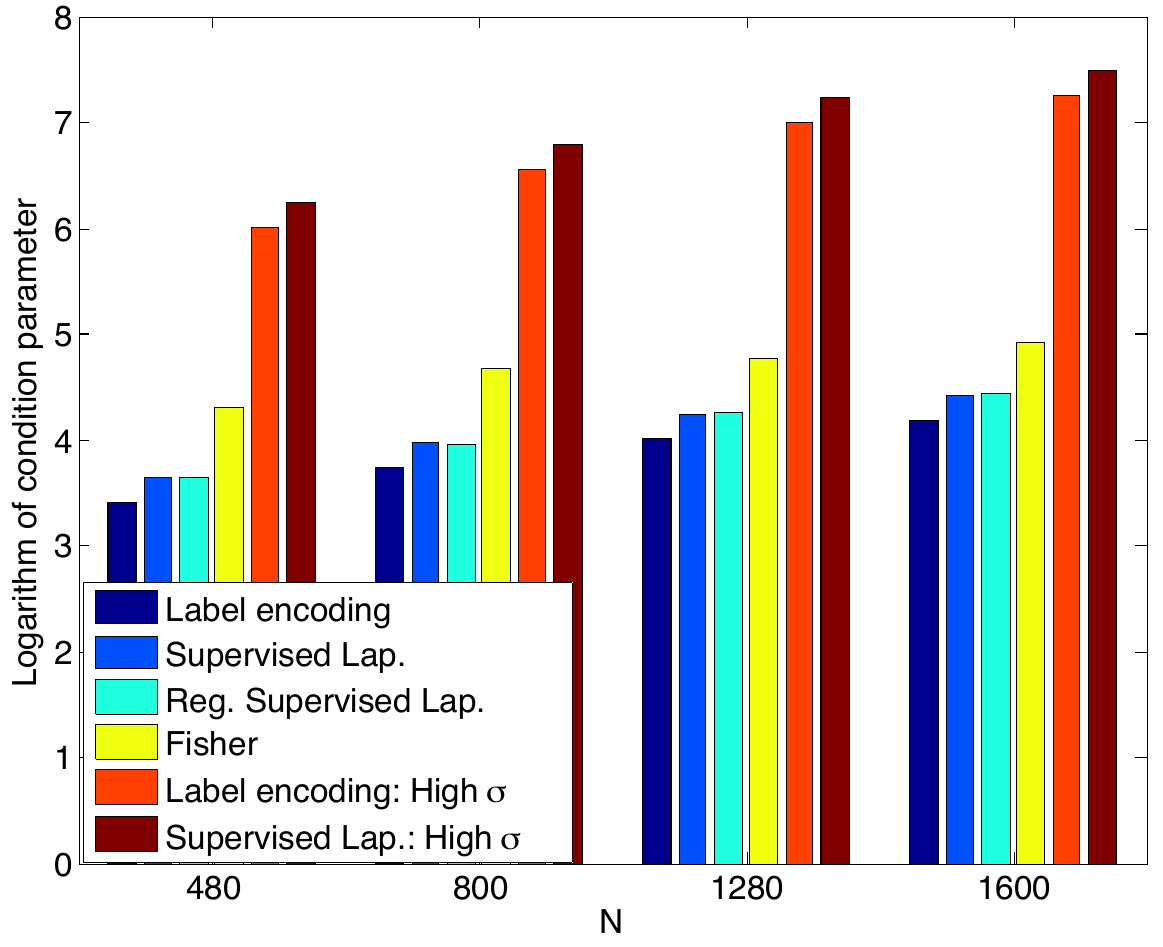}}
 \end{center}
 \caption{Misclassification rates and the condition parameters of the embeddings for the ETH-80 object data set}
 \label{fig:error_cond_ETH80}
\end{figure}

\begin{figure}[]
\begin{center}
     \subfigure[]
       {\label{fig:error_embeddings_yalefull}\includegraphics[height=5cm]{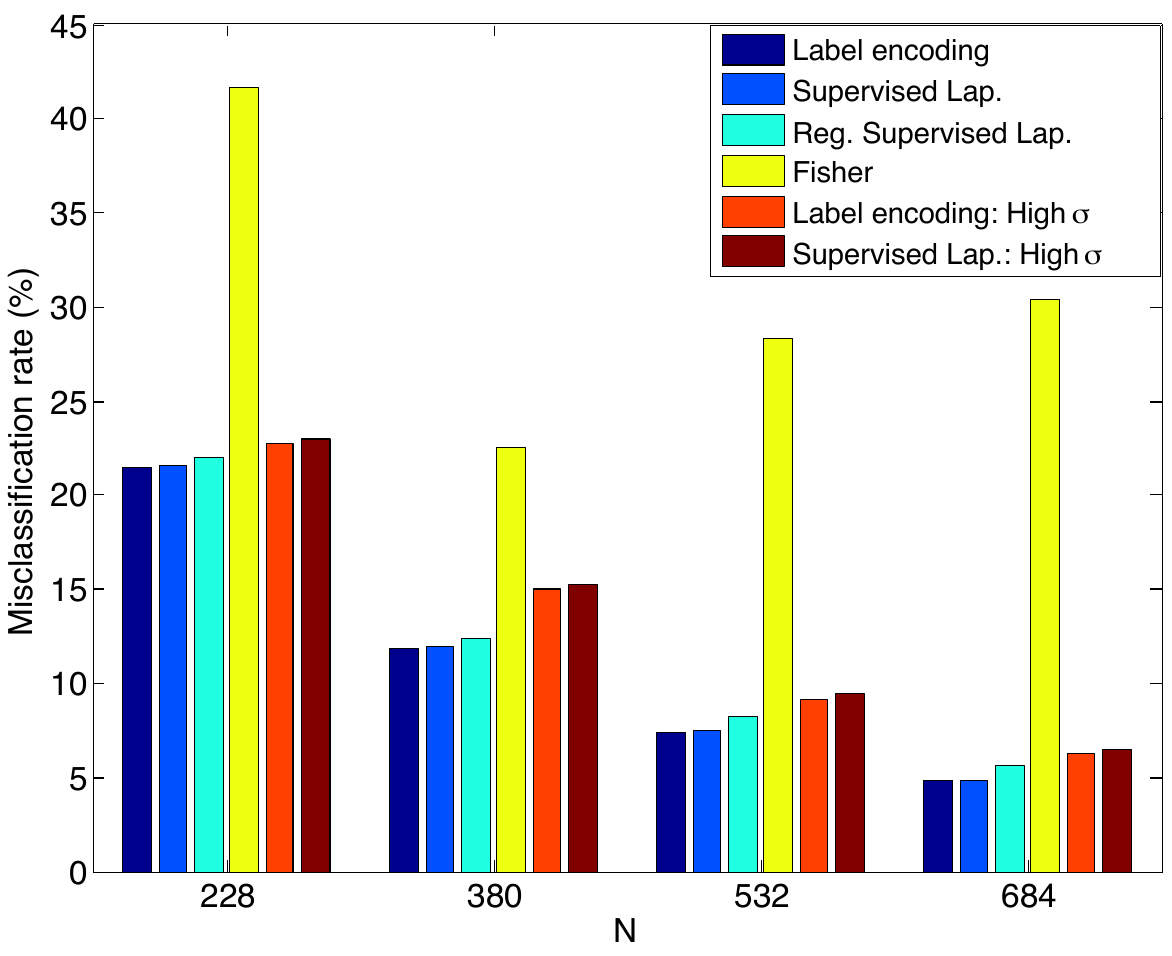}}
     \subfigure[]
       {\label{fig:cond_embeddings_yalefull}\includegraphics[height=5cm]{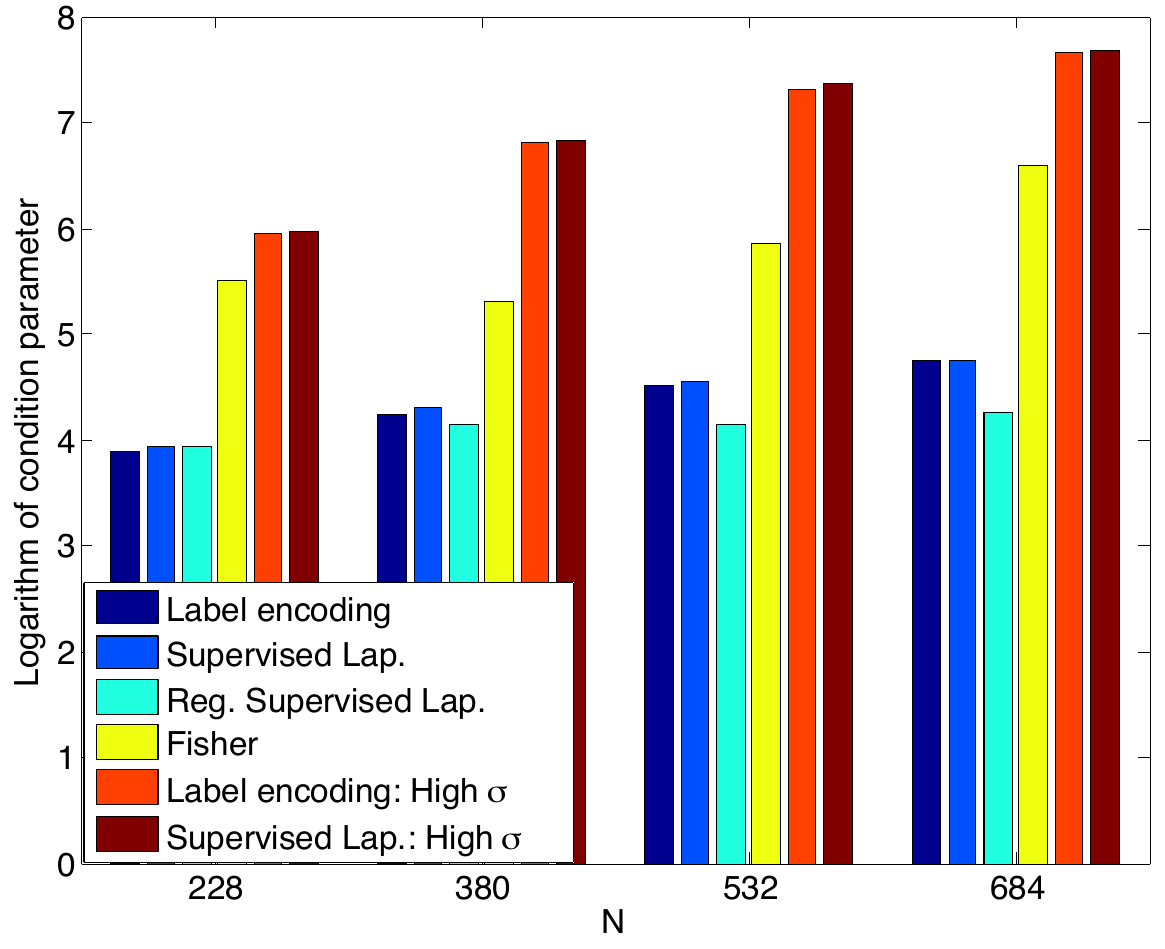}}
 \end{center}
 \caption{Misclassification rates and the condition parameters of the embeddings for the Yale face data set}
 \label{fig:error_cond_yalefull}
\end{figure}

\begin{figure}[]
\begin{center}
     \subfigure[]
       {\label{fig:error_embeddings_yalered}\includegraphics[height=5cm]{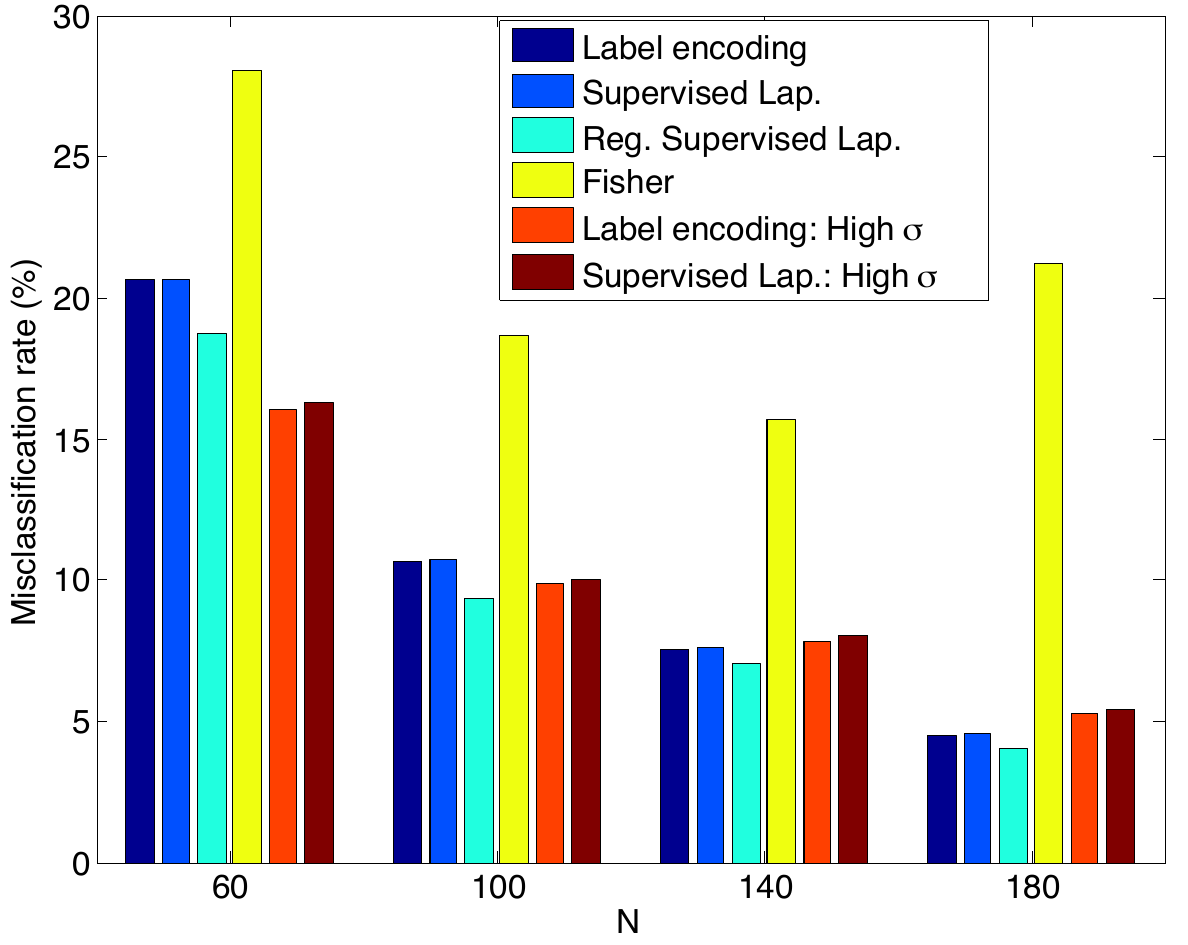}}
     \subfigure[]
       {\label{fig:cond_embeddings_yalered}\includegraphics[height=5cm]{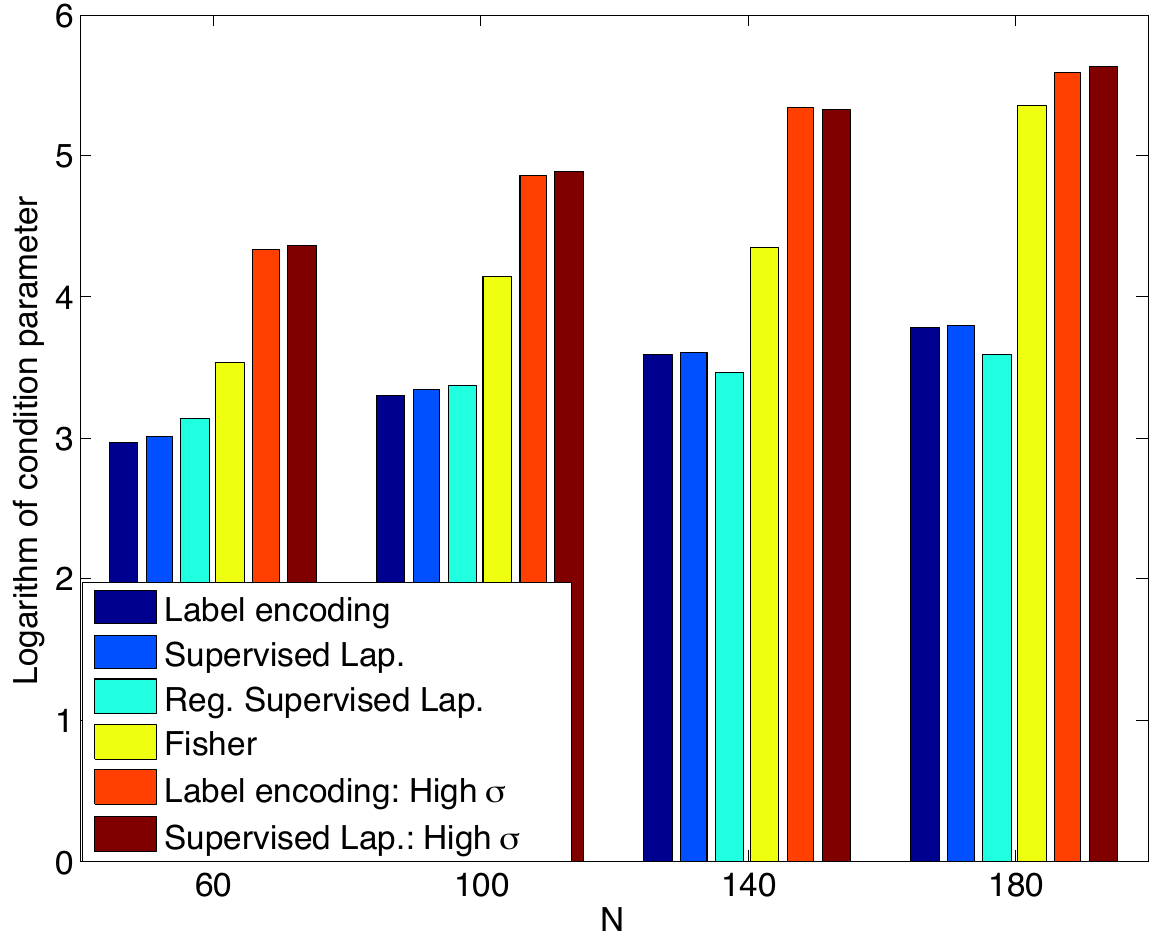}}
 \end{center}
 \caption{Misclassification rates and the condition parameters of the embeddings for the reduced Yale face data set}
 \label{fig:error_cond_yalered}
\end{figure}

The plots in Figures \ref{fig:error_cond_COIL20}-\ref{fig:error_cond_yalered} show that the label encoding, supervised Laplacian eigenmaps, and the regularized supervised Laplacian eigenmaps embeddings yield better classification accuracy than the other three methods (supervised Fisher, and the embeddings with high scale parameters) in all experiments, with the only exception of the cases $N=60$ and $N=100$ for the reduced Yale data set. Meanwhile, examining the condition parameters of the embeddings, we observe that label encoding, supervised Laplacian eigenmaps, and the regularized supervised Laplacian eigenmaps embeddings always have a smaller condition parameter than the other three methods. This observation confirms the intuition provided by the necessary conditions of Theorems \ref{thm:acc_cl_rbfint} and \ref{thm:acc_cl_nn_rbfint}: A compromise between the separation and the interpolator regularity is required for good classification accuracy. The increase in the condition parameter as $N$ increases is since the coefficient bound  $\cbnd$ involves a summation over all training samples. The reason why the embeddings with high $\sigma$ parameters yield better classification accuracy than the other ones in the cases $N=60$ and $N=100$ for the reduced Yale data set is that a larger RBF scale helps better cover up the ambient space when the number of training samples is particularly low.

In the COIL-20 and the reduced Yale data sets, the best classification accuracy is obtained with the regularized supervised Laplacian eigenmaps method, while this is also the method having the smallest condition number, except for the smallest two values of $N$ in the reduced Yale data set. In the ETH-80 and Full Yale data sets, the classification accuracy of label encoding attains that of the supervised Laplacian eigenmaps method. The condition parameter of the label encoding embedding is relatively small in these two data sets; in fact, in ETH-80 the label encoding embedding has the smallest condition number among all methods. This may be useful for explaining why this simple classification method has quite favorable performance in this data set. Likewise, if we leave aside the versions of the methods with high-scale interpolators, the Fisher embedding has the highest misclassification rate compared to label encoding, the supervised Laplacian, and the regularized supervised Laplacian embeddings, while it also has the highest condition parameter among these methods. \footnote{The formulation in \eqref{eq:form_fisher_embed} has been observed to give highly polarized embeddings in \citep{VuralG16}, where the samples of only few classes stretch out along each dimension and all the other classes are mapped close to zero.} 

To conclude, the results in this section suggest that the experimental findings are in agreement with the main results in Section \ref{ssec:oos_rbf}, justifying the pertinence of the conditions \eqref{eq:cond_sep_interp} and \eqref{eq:cond_sep_nn_interp} to classification accuracy, hence suggesting that a balance must be sought between the separability margin of the embedding and the regularity of the interpolation function in supervised manifold learning.

\section{Conclusions}
\label{sec:conc}

Most of the current supervised manifold learning algorithms focus on learning representations of training data, while the generalization properties of these representations have not been understood well yet. In this work, we have proposed a theoretical analysis of the performance of supervised manifold learning methods. We have presented generalization bounds for nonlinear supervised manifold learning algorithms and explored how the classification accuracy relates to several setup parameters such as the linear separation margin of the embedding, the regularity of the interpolation function, the number of training samples, and the intrinsic dimensions of the class supports (manifolds). Our results suggest that embeddings of training data with good generalization capacities must allow the construction of sufficiently regular interpolation functions that extend the mapping to new data. We have then examined whether the assumption of linear separability is easy to satisfy for structure-preserving supervised embedding algorithms. We have taken the supervised Laplacian eigenmaps algorithms as reference, and showed that these methods can yield linearly separable embeddings. Providing insight about the generalization capabilities of supervised dimensionality reduction algorithms, our findings can be helpful in the classification of low-dimensional data sets.



\acks{We would like to thank Pascal Frossard and Alhussein Fawzi for the helpful discussions that contributed to this study.}


\appendix


\section{Proof of the results in Section \ref{sec:class_anly_supml} }

\subsection{Proof of Theorem \ref{thm:emb_smooth_intp}}
\label{pf:thm:emb_smooth_intp}

\begin{proof}
Given $\x$, let $\x_i \in \X$ be the nearest neighbor of $\x$ in $\X$ that is sampled from $\pmes_m$
\begin{equation*}
i = \arg \min_{j} \| \x - \x_j  \|  \, \text{ s.t. } \, \x_j \sim \pmes_m.
\end{equation*}
Due to the separation hypothesis, 
\begin{equation*}
\hyp_{mk}^T \, \yi + b_{mk}  > \mar/2,  \quad  \forall k=1, \dots, M-1. 
\end{equation*}
We have 
\begin{equation*}
\begin{split}
\hyp_{mk}^T \, f(\x) + b_{mk} &= \hyp_{mk}^T \, f(\x_i) + b_{mk} +  \hyp_{mk}^T \, (f(\x) - f(\x_i)) \\
  & \geq \hyp_{mk}^T \, \y_i + b_{mk} - \left |  \hyp_{mk}^T \, (f(\x) - f(\x_i))  \right | \\
  & >  \mar/2  - \| f(\x) - f(\x_i)  \| 
     \, \, \geq \, \,  \mar/2  - \Lcon \| \x - \x_i \|.
\end{split}
\end{equation*}
Then if the condition $\Lcon \| \x - \x_i \| \leq \mar/2$ is satisfied, from the above inequality we have $\hyp_{mk}^T \, f(\x) + b_{mk} > 0 $ for all $k=1, \dots, M-1$. This gives $\hat{\class}(\x) = m$ and thus ensures that $\x$ is classified correctly.

In the sequel, we lower bound the probability that the distance $\| \x - \x_i \|$ between $\x$ and its nearest neighbor from the same class is smaller than $\mar/2$. We employ the following result by \citet{KulkarniP95}. It is demonstrated in the proof of Theorem 1 in \citep{KulkarniP95} that, if $\X$ contains at least $\numsamp_m$ samples drawn i.i.d.~from $\pmes_m$ such that $\numsamp_m  \geq \cover(\epsilon/2, \M_m) $ for some $\epsilon >0$, then the probability of $\| \x - \x_i \|$ being larger than $\epsilon$ can be upper bounded in terms of the covering number of $\M_m$ as
\begin{equation*}
P(\| \x - \x_i \| > \epsilon ) \leq \frac{\cover(\epsilon/2, \M_m)}{2 \numsamp_m}.
\end{equation*}
Therefore, for any $\epsilon$ such that $\epsilon \leq \mar/(2L)$ and $\numsamp_m  \geq \cover(\epsilon/2, \M_m) $, with probability at least $1 - \cover(\epsilon/2, \M_m) / (2 \numsamp_m)$, we have
\begin{equation*}
\| \x - \x_i \| \leq \epsilon  \leq \mar/(2L)
\end{equation*}
thus, the class label of $\x$ is correctly estimated as $\hat{\class}(\x) = m$ due to the above discussion.
\end{proof}

\subsection{Proof of Lemma \ref{lem:fdev_neigh_genf}}
\label{pf:lem_fdev_neigh_genf}
\begin{proof}
We first bound the deviation of $f(x)$ from the sample average of $f$ in the neighborhood of $\x$ as
\begin{equation}
\label{eq:f_dev_decom_genf}
\left \| f(x) - \frac{1}{\Knb} \sum_{\x_j \in \nbdm} f(\x_j) \right \|   
\leq \left  \|  f(x) - m_f   \right \|  
+ \left \|   \frac{1}{\Knb} \sum_{\x_j \in \nbdm} f(\x_j) - m_f  \right \|
\end{equation}
where $m_f$ is the conditional expectation of $f(u)$, given $u \in B_{\delta}(x)$
\begin{equation*}
m_f = \E_{u} \big[ f(u) \given u \in \Bx \big ] 
 =  \frac{1}{\pmes_m \big(\Bx \big)}  \int_{\Bx}  f(u) \, \,  d \pmes_m(u ) .
\end{equation*}

The first term in \eqref{eq:f_dev_decom_genf} can be bounded as
\begin{equation}
\label{eq:first_term_lem_genf}
\begin{split}
\|  f(x) - m_f    \|  &= \left \|   \frac{1}{\pmes_m \big(\Bx \big)}  \int_{\Bx} \big( f(x) - f(u) \big) \, d \pmes_m(u )  \right  \| \\
&\leq   \frac{1}{\pmes_m \big(\Bx \big)}    \int_{\Bx} \| f(x) - f(u) \|  \, d \pmes_m(u ) 
\leq  \frac{1}{\pmes_m \big(\Bx \big)}    \int_{\Bx}  \Lcon \| x - u \| \, d \pmes_m(u ) \\
&\leq  \frac{1}{\pmes_m \big(\Bx \big)}  \int_{\Bx}  \Lcon \delta  \, d \pmes_m(u )  = \Lcon \delta
\end{split}
\end{equation}
where the second inequality follows from the fact that $f$ is Lipschitz continuous on the support $\M_m$, where the measure $\pmes_m$ is nonzero.

The second term in \eqref{eq:f_dev_decom_genf} is given by
\begin{equation}
\label{eq:norm_fdev_mean_genf}
\left \|    \frac{1}{\Knb} \sum_{\x_j \in \nbdm} f(\x_j)  - m_f  \right \| 
= \left(    \sum_{k=1}^d   \bigg |     
\frac{1}{\Knb}  \sum_{\x_j \in \nbdm} f^k(\x_j)  - m_f^k
   \bigg |^2   \right)^{1/2}
\end{equation}
where $m_f^k$ denotes the $k$-th component of $m_f$, for $k=1, \dots, d$. Consider the random variables $f^k(\x_j)$. Defining  
\begin{equation*}
f_{\min}^k = \inf_{u \in \Bx } f^k(u), 
\qquad \quad
f_{\max}^k = \sup_{u \in \Bx } f^k(u), 
\end{equation*}
it follows that $f_{\max}^k - f_{\min}^k  \leq 2 \Lcon  \delta$ due to the Lipschitz continuity of $f$. Then from Hoeffding's inequality, we have
\begin{equation*}
P \left(   \bigg |  \frac{1}{\Knb}  \sum_{\x_j \in \nbdm} f^k(\x_j)  - m_f^k    \bigg |  \geq \epsilon   \right)
\leq 2 \exp \left(     -\frac{2 \Knb \epsilon^2}{ ( f_{\max}^k - f_{\min}^k)^2}   \right)
\leq 2 \exp \left(     -\frac{\Knb \epsilon^2}{ 2 \Lcon^2  \delta^2}   \right).
\end{equation*}
From the union bound, we get that with probability at least $ 1- 2 d \exp \left(     -\frac{\Knb \epsilon^2}{ 2 \Lcon^2  \delta^2}   \right)$, for all $k$
\begin{equation*}
 \bigg |  \frac{1}{\Knb}  \sum_{\x_j \in \nbdm} f^k(\x_j)  - m_f^k    \bigg |  \leq \epsilon, 
\end{equation*}
which yields from \eqref{eq:norm_fdev_mean_genf}
\begin{equation*}
\left \|    \frac{1}{\Knb} \sum_{\x_j \in \nbdm} f(\x_j)  - m_f  \right \| 
\leq \sqrt{d} \epsilon.
\end{equation*}
Combining this result with the bound in \eqref{eq:first_term_lem_genf}, we conclude that with probability at least $ 1- 2 d \exp \left(     -\frac{\Knb \epsilon^2}{ 2 \Lcon^2  \delta^2}   \right)$
\begin{equation*}
\left \| f(x) - \frac{1}{\Knb} \sum_{\x_j \in \nbdm} f(\x_j) \right \|   
\leq \Lcon \delta + \sqrt{d} \epsilon.
\end{equation*}
\end{proof}

\subsection{Proof of Theorem \ref{thm:linclassif_genf}}
\label{pf:thm_linclassif_genf}

\begin{proof}
Given the test sample $\x$ and a training sample $\xii$ drawn i.i.d.~with respect to $\pmes_m$, the probability that $\xii$ lies within a $\delta$-neighborhood of $\x$ is given by
\begin{equation*}
P(\xii \in \Bx) = \pmes_m (\Bx) \geq \minm.
\end{equation*}
Then, among the $\numsamp_m$ samples drawn with respect to $\pmes_m$, the probability that $\Bx$ contains at least $Q$ samples is given by
\begin{equation*}
\begin{split}
P(| \nbdm | \geq Q) &= \sum_{q=Q}^{\numsamp_m} \binom {\numsamp_m}{q}  
\bigg( \pmes_m (\Bx) \bigg)^q  \bigg(1 - \pmes_m (\Bx) \bigg)^{\numsamp_m - q} \\
&\geq
\sum_{q=Q}^{\numsamp_m} \binom {\numsamp_m}{q}  
(\minm)^q \, (1-\minm)^{\numsamp_m - q}
\end{split}
\end{equation*}
where the set $\nbdm$ is defined as in \eqref{eq:defn_NNx_lem_genf}. The last expression above is the probability of having at least $Q$ successes out of $\numsamp_m$  realizations of a Bernoulli random variable with probability parameter $\minm$. This probability can be lower bounded using a tail bound for binomial distributions. We thus have
\begin{equation*}
P(| \nbdm | \geq Q) \geq 1 - \exp \left( \frac{-2 \, (\numsamp_m \, \minm - Q)^2 }{\numsamp_m} \right)
\end{equation*}
which simply follows from interpreting $|\nbdm|$ as the sum of of $\numsamp_m$ i.i.d.~observations of a Bernoulli distributed random variable and then applying Hoeffding's inequality as shown by \citet{Herbrich99}, under the hypothesis that $\numsamp_m > Q/{\minm}$.

Assuming that $\Bx$ contains at least $Q$ samples, Lemma \ref{lem:fdev_neigh_genf} states that with probability at least
\begin{equation*}
 1- 2 d \exp \left(     -\frac{ | \nbdm | \epsilon^2}{ 2 \Lcon^2  \delta^2}   \right) \geq  1- 2 d \exp \left(     -\frac{ \Knb \epsilon^2}{ 2 \Lcon^2  \delta^2}   \right) 
\end{equation*}
the deviation between $f(x)$ and the sample average of its neighbors is bounded as
\begin{equation*}
\left \| f(x) - \frac{1}{  | \nbdm | } \sum_{\x_j \in \nbdm} f(\x_j) \right \|   
\leq \Lcon \delta + \sqrt{d} \epsilon.
\end{equation*}
Hence, with probability at least
\begin{equation*}
\begin{split}
\left(  1 - \exp \left( \frac{-2 \, (\numsamp_m \, \minm - Q)^2 }{\numsamp_m} \right) \right)
\left ( 1- 2 d \exp \left(     -\frac{\Knb \epsilon^2}{ 2 \Lcon^2  \delta^2}   \right) \right)  \\
\geq 1 - \exp \left( \frac{-2 \, (\numsamp_m \, \minm - Q)^2 }{\numsamp_m} \right) 
- 2 d \exp \left(     -\frac{\Knb \epsilon^2}{ 2 \Lcon^2  \delta^2}   \right)
\end{split}
\end{equation*}
we have 
\begin{equation}
\label{eq:ovrbnd_genfdev_mn_lincl}
\left \| f(x) - \frac{1}{| \nbdm |} \sum_{\x_j \in \nbdm} f(\x_j) \right \|   
\leq \Lcon \delta + \sqrt{d} \epsilon.
\end{equation}

The class label of a test sample $\x$ drawn from $\pmes_m$ is correctly estimated with respect to the classifier \eqref{eq:lin_classifier} if  
\begin{equation*}
\hyp_{mk}^T \, f(x) + b_{mk}  > 0,  \quad  \forall k=1, \dots, M-1, \, k\neq m.
\end{equation*}
If the condition in \eqref{eq:ovrbnd_genfdev_mn_lincl} is satisfied, for all $k\neq m$, we have
\begin{equation*}
\begin{split}
\hyp_{mk}^T \, f(x) + b_{mk}  
&= \hyp_{mk}^T \frac{1}{| \nbdm |} 
\sum_{\x_j \in \nbdm} f(\x_j)  
+ b_{mk} 
+\hyp_{mk}^T \, \left( f(x) - \frac{1}{| \nbdm |} 
\sum_{\x_j \in \nbdm} f(\x_j) 
\right) \\   
&\geq
\hyp_{mk}^T \frac{1}{| \nbdm |} 
\sum_{\x_j \in \nbdm} f(\x_j)  
+ b_{mk}
- \|  f(x) - \frac{1}{| \nbdm |} 
\sum_{\x_j \in \nbdm} f(\x_j) 
  \| \\
&>
\mar_Q/2 - \|  f(x) - \frac{1}{| \nbdm |} 
\sum_{\x_j \in \nbdm} f(\x_j)  \|
\geq 
\mar_Q/2 -   \Lcon \delta - \sqrt{d} \epsilon
\geq 0.
\end{split}
\end{equation*}
Here, we obtain the second inequality from the hypothesis that the embedding is $Q$-mean separable with margin larger than $\mar_Q$, which implies that the embedding is also $R$-mean separable with margin larger than $\mar_Q$, for $R>Q$. Then the last inequality is due to the condition \eqref{eq:cond_sep_genf_lin} on the interpolation function in the theorem. We thus get that with probability at least 
\begin{equation*}
1 - \exp \left( \frac{-2 \, (\numsamp_m \, \minm - Q)^2 }{\numsamp_m} \right) 
- 2 d \exp \left(     -\frac{\Knb \epsilon^2}{ 2 \Lcon^2  \delta^2}   \right)
\end{equation*}
$\hyp_{mk}^T \, f(x) + b_{mk}  > 0$ for all $k\neq m$, hence, the sample $\x$ is correctly classified. This concludes the proof of the theorem.
\end{proof}

\subsection{Proof of Theorem \ref{thm:error_genf_nnclass}}
\label{pf:thm_error_genf_nnclass}

\begin{proof}
%
Remember from the proof of Theorem \ref{thm:linclassif_genf} that with probability at least 
\begin{equation*}
1 - \exp \left( \frac{-2 \, (\numsamp_m \, \minm - Q)^2 }{\numsamp_m} \right) 
- 2 d \exp \left(     -\frac{\Knb \epsilon^2}{ 2 \Lcon^2  \delta^2}   \right)
\end{equation*}
the $\delta$-neighborhood $\Bx$ of a test sample $x$ from class $m$ contains at least $Q$ samples from class $m$, and
\begin{equation}
\label{eq:genf_fdev_nn_pfthm}
\left \| f(x) - \frac{1}{| \nbdm |} \sum_{\x_j \in \nbdm} f(\x_j) \right \|   
\leq \Lcon \delta + \sqrt{d} \epsilon
\end{equation}
where $\nbdm$ is the set of training samples in $\Bx$ from class $m$.

Let $\x_i, \x_j \in \nbdm$ be two training samples from class $m$ in $\Bx$. As $\| \x_i -  \x_j \|  \leq 2 \delta$, by the hypothesis on the embedding, we have $\| \y_i - \y_j  \| = \| f(\x_i) - f(\x_j)  \| \leq \Demb_{2 \delta} $, which gives
\[
\| f(\x_i )-  \frac{1}{| \nbdm |}  \sum_{\x_j \in \nbdm} f(\x_j)  \|  = \left \|    \frac{1}{| \nbdm |}  \sum_{\x_j \in \nbdm} \big( f(\x_i )-  f(\x_j) \big) \right \|
\leq    \frac{1}{| \nbdm |}   \sum_{\x_j \in \nbdm}  \| f(\x_i )-  f(\x_j) \|  \leq  \Demb_{2 \delta}.
\]
Then, for any $\x_i \in \Bx$, 
\begin{equation*}
\begin{split}
\|  f(\x) - f(\x_i)  \| = \|  f(\x) -  \frac{1}{| \nbdm |}  \sum_{\x_j \in \nbdm} f(\x_j) + \frac{1}{| \nbdm |}  \sum_{\x_j \in \nbdm} f(\x_j)  -   f(\x_i)  \|  \\
\leq   \|  f(\x) -  \frac{1}{| \nbdm |}  \sum_{\x_j \in \nbdm} f(\x_j) \|  + \Demb_{2 \delta}.
\end{split}
\end{equation*}
Combining this with \eqref{eq:genf_fdev_nn_pfthm}, we get that with probability at least 
\begin{equation*}
1 - \exp \left( \frac{-2 \, (\numsamp_m \, \minm - Q)^2 }{\numsamp_m} \right) 
- 2 d \exp \left(     -\frac{\Knb \epsilon^2}{ 2 \Lcon^2  \delta^2}   \right)
\end{equation*}
$\Bx$ will contain at least $Q$ samples $\x_i$ from class $m$ such that 
\begin{equation}
\label{eq:bnd_genf_dist_to_clneig}
\|  f(\x) - f(\x_i)  \| \leq  \Lcon \delta + \sqrt{d} \epsilon + \Demb_{2 \delta}.
\end{equation}

Now, assuming \eqref{eq:bnd_genf_dist_to_clneig},  let  $\x_i'$ be a training sample from another class (other than $m$). We have
\begin{equation*}
\| f(\x) - f(\x_i') \| \geq  \| f(\x_i)  - f(\x_i')  \|  - \|  f(\x) - f(\x_i)  \|  > \mar - ( \Lcon \delta + \sqrt{d} \epsilon + \Demb_{2 \delta})
\end{equation*}
which follows from \eqref{eq:bnd_genf_dist_to_clneig} and the hypothesis on the embedding that $ \| f(\x_i) - f(\x_i') \| > \mar$.

It follows from the condition \eqref{eq:cond_sep_genf_nn} that $\mar \geq 2 \Lcon \delta + 2 \sqrt{d} \epsilon +   2 \Demb_{2 \delta} $. Using this in the above equation, we get
\begin{equation*}
\| f(\x) - f(\x_i') \| >  \Lcon \delta + \sqrt{d} \epsilon + \Demb_{2 \delta}.
\end{equation*}
This means that the distance of $f(\x)$ to the embedding of any other sample from another class is more than $ \Lcon \delta + \sqrt{d} \epsilon + \Demb_{2 \delta}$, while there are samples from its own class within a distance of  $\Lcon \delta + \sqrt{d} \epsilon + \Demb_{2 \delta}$ to $f(\x)$. Therefore, $x$ is classified correctly with nearest-neighbor classification in the low-dimensional domain of embedding.

\end{proof}

\subsection{Proof of Lemma \ref{lem:bnd_dev_f_nborhd}}
\label{pf:lem:bnd_dev_f_nborhd}

\begin{proof}
The deviation of each component $f^k(\x)$ of the interpolator from the sample average in the neighborhood of $\x$ is given by
\begin{equation}
\label{eq:form1_fdiff}
\begin{split}
\left | f^k(\x) - \frac{1}{\Knb} \sum_{\x_j \in \nbdm} f^k(\x_j) \right | 
	= \left |  \sum_{i=1}^\numsamp \cik
	\left( 
	 \phi(\| \x - \x_i  \|) 
		-   \frac{1}{\Knb}  \sum_{\x_j \in \nbdm} \phi(\| \x_j - \x_i  \|)   
	\right)
	\right |.
\end{split}
\end{equation}
%
We thus proceed by studying the term
\begin{equation}
\label{eq:phi_dev_term}
 \phi(\| \x - \x_i  \|) -   \frac{1}{\Knb}  \sum_{\x_j \in \nbdm} \phi(\| \x_j - \x_i  \|)   
\end{equation}
which will then be used in the above expression to arrive at the stated result.

Now let $\xii \in \X$ be any training sample. In order to study the term in \eqref{eq:phi_dev_term}, we first look at 
\begin{equation*}
\bigg | \phi(\| \x - \xii \|) - \E_{u} \big[ \phi(\| u - \xii \|) \, | \, u \in B_{\delta}(x) \big ] \bigg |
\end{equation*}
where $\E_{u} \big[ \phi(\| u - \xii \|) \, | \, u \in B_{\delta}(x) \big ] $ denotes the conditional expectation of $\phi(\| u - \xii \|)$ over $u$, given $u \in B_{\delta}(x)$. The conditional expectation is given by
\begin{equation*}
\begin{split}
\E_{u} \big[ \phi(\| u - \xii \|) \given u \in \Bx \big ] 
 &=  \frac{1}{\pmes_m \big(\Bx \big)}  \int_{\Bx}  \phi(\| u - \xii \|) \, \,  d \pmes_m(u ) .
\end{split}
\end{equation*}
We have
\begin{equation*}
\begin{split}
\bigg | \phi(\| \x - \xii \|) &- \E_{u} \big[ \phi(\| u - \xii \|) \, | \, u \in B_{\delta}(x) \big ] \bigg | \\
 &=  \frac{1}{\pmes_m \big(\Bx \big)} \ 
  \bigg| 
  \int_{\Bx} \big( \phi(\| \x - \xii \|)  -  \phi(\| u - \xii \|)  \big)  \, \, d \pmes_m(u )
 \bigg| \\
 & \leq  \frac{1}{\pmes_m \big(\Bx \big)} 
  \int_{\Bx} \big| \phi(\| \x - \xii \|)  -  \phi(\| u - \xii \|)  \big|  \, \,  d \pmes_m(u ). 
\end{split}
\end{equation*}
The term in the integral is bounded as
\begin{equation*}
 \big| \phi(\| \x - \xii \|)  -  \phi(\| u - \xii \|)  \big|  
 \leq \Lconf  \, 
 \big|  \| \x - \xii \| -  \| u - \xii \|  \big|
 \leq \Lconf  \,  \| \x - u \|.
\end{equation*}
Using this in the above term, we get
\begin{equation}
\label{eq:bnd_dev_phix_mean}
\begin{split}
\bigg | \phi(\| \x - \xii \|) &- \E_{u} \big[ \phi(\| u - \xii \|) \, | \, u \in B_{\delta}(x) \big ] \bigg | \\
&\leq \frac{\Lconf}{\pmes_m \big(\Bx \big)}  
   \int_{\Bx}  \| \x - u \|  \, d \pmes_m(u )
 =   \Lconf  \, \E_{u} \big[ \| u - \x \|  \given  u \in B_{\delta}(x) \big ]  \\
& \leq \Lconf \, \delta.
\end{split}
\end{equation}

We now analyze the term in \eqref{eq:phi_dev_term} for a given $\xii$ for two different cases, i.e., for $\xii \notin \Bx$ and $\xii \in \Bx$. We first look at the case $\xii \notin \Bx$. For $\x_j \in \Bx$, let
\begin{equation*}
\zeta_j := \phi(\| \x_j - \x_i  \|).
\end{equation*}
The observations $\zeta_j$ are i.i.d.~(since $\xj$ are i.i.d.) with mean 
$
m_\zeta= \E_{u} \big[ \phi(\| u - \xii \|) \, | \, u \in B_{\delta}(x) \big ] 
$
and take values in the interval $\zeta_{\min} \leq \zeta_j \leq \zeta_{\max}$, where 
\begin{equation*}
\zeta_{\min} := \inf_{u \in \Bx} \phi (\| u - \x_i  \|), 
\qquad \qquad
\zeta_{\max} := \sup_{u \in \Bx} \phi (\| u - \x_i  \|).
\end{equation*}
Since for any $u_1, u_2 \in \Bx$, $\|u_1-u_2\| \leq 2 \delta$, it follows from the Lipschitz continuity of $\phi$ that  
$\zeta_{\max} - \zeta_{\min} \leq 2 \Lconf \delta$. Using this together with the Hoeffding's inequality, we get
\begin{equation}
\label{eq:hoeff_dev_zeta}
P \bigg(  \bigg|   \frac{1}{\Knb} \sum_{\x_j \in \nbdm} \zeta_j - m_\zeta  \bigg |  
\geq \epsilon
\bigg)
\leq  2 \exp \left( - \frac{2 \, \Knb \, \epsilon^2 }{(\zeta_{\max}  - \zeta_{\min} )^2 } \right)
\leq 2 \exp \left( - \frac{\Knb \, \epsilon^2 }{ 2 \Lconf^2 \delta^2} \right).
\end{equation}
We have 
\begin{equation*}
\begin{split}
\bigg|  \phi(\| \x - \x_i  \|) -   \frac{1}{\Knb}  \sum_{\x_j \in \nbdm} \phi(\| \x_j - \x_i  \|)  \bigg| 
\leq 
\big|  \phi(\| \x - \x_i  \|) -  m_\zeta \big| 
+
\bigg| m_\zeta -  \frac{1}{\Knb}  \sum_{\x_j \in \nbdm} \phi(\| \x_j - \x_i  \|)  \bigg| .
\end{split} 
\end{equation*}
Using \eqref{eq:bnd_dev_phix_mean} and \eqref{eq:hoeff_dev_zeta} in the above equation, it holds with probability at least 
\begin{equation*}
1- 2 \exp \left( - \frac{\Knb \, \epsilon^2 }{ 2 \Lconf^2 \delta^2} \right)
\end{equation*}
that
\begin{equation*}
\bigg|  \phi(\| \x - \x_i  \|) -   \frac{1}{\Knb}  \sum_{\x_j \in \nbdm} \phi(\| \x_j - \x_i  \|)  \bigg| 
\leq \Lconf \delta + \epsilon.
\end{equation*}
Next, we study the case $\xii \in \Bx$. For any fixed $\xii \in \Bx$, hence $\xii \in \nbdm$, we have: 
\begin{equation*}
\begin{split}
&\bigg|  \phi(\| \x - \x_i  \|) -   \frac{1}{\Knb}  \sum_{\x_j \in \nbdm} \phi(\| \x_j - \x_i  \|)  \bigg|  \\
&= 
\bigg| \frac{1}{\Knb} \phi(\| \x - \x_i  \|) +  \frac{\Knb - 1}{\Knb} \phi(\| \x - \x_i  \|) -  \frac{1}{\Knb} \phi(\| \x_i - \x_i \|)  - \frac{1}{\Knb}  \sum_{\x_j \in \nbdm \setminus \{ \x_i \} } \phi(\| \x_j - \x_i  \|)  \bigg|  \\
&\leq 
 \frac{1}{\Knb}  \bigg|  \phi(\| \x - \x_i  \|) - \phi(\| \x_i - \x_i \|)   \bigg| 
 +
  \frac{\Knb - 1}{\Knb}  \bigg|  \phi(\| \x - \x_i  \|)  - \frac{1}{\Knb -1 }  \sum_{\x_j \in \nbdm \setminus \{ \x_i \} } \phi(\| \x_j - \x_i  \|)  \bigg| 
\end{split} 
\end{equation*}
The first term above is bounded as 
\begin{equation*}
 \frac{1}{\Knb}  \bigg|  \phi(\| \x - \x_i  \|) - \phi(\| \x_i - \x_i \|)   \bigg| 
 \leq    \frac{\Lconf \delta}{\Knb}. 
\end{equation*}
Next, similarly to the analysis of the case $\xii \neq \Bx$, we get that for $\xii \in \Bx$ with probability at least 
\begin{equation*}
1- 2 \exp \left( - \frac{(\Knb-1) \, \epsilon^2 }{ 2 \Lconf^2 \delta^2} \right)
\end{equation*}
it holds that
\begin{equation*}
\bigg|  \phi(\| \x - \x_i  \|)  - \frac{1}{\Knb -1 }  \sum_{\x_j \in \nbdm \setminus \{ \x_i \} } \phi(\| \x_j - \x_i  \|)  \bigg| 
\leq \Lconf \delta + \epsilon,
\end{equation*}
hence
\begin{equation*}
\begin{split}
\bigg|  \phi(\| \x - \x_i  \|) -   \frac{1}{\Knb}  \sum_{\x_j \in \nbdm} \phi(\| \x_j - \x_i  \|)  \bigg|  
\leq 
\frac{\Lconf \delta}{\Knb} +   \frac{\Knb - 1}{\Knb} ( \Lconf \delta + \epsilon)
\leq 
\Lconf \delta + \epsilon.
\end{split}
\end{equation*}
Combining the analyses of the cases $\xii \neq \Bx$ and $\xii \in \Bx$, we conclude that for any given $\xii \in \X$, 
\begin{equation*}
P \left( 
\bigg|  \phi(\| \x - \x_i  \|) -   \frac{1}{\Knb}  \sum_{\x_j \in \nbdm} \phi(\| \x_j - \x_i  \|)  \bigg|  
\leq
\Lconf \delta + \epsilon
\right)
\geq
1- 2 \exp \left( - \frac{(\Knb-1) \, \epsilon^2 }{ 2 \Lconf^2 \delta^2} \right).
\end{equation*}
Therefore, applying the union bound on all $\numsamp$ samples $\xii$ in $\X$, we get that with probability at least
\begin{equation*}
1- 2 \numsamp \exp \left( - \frac{(\Knb-1) \, \epsilon^2 }{ 2 \Lconf^2 \delta^2} \right)
\end{equation*}
it holds that
\begin{equation}
\label{eq:cond_phidev_allxi}
\bigg|  \phi(\| \x - \x_i  \|) -   \frac{1}{\Knb}  \sum_{\x_j \in \nbdm} \phi(\| \x_j - \x_i  \|)  \bigg|  
\leq
\Lconf \delta + \epsilon
\end{equation}
for all $\xii \in \X$.

We can now use this in \eqref{eq:form1_fdiff} to bound the deviation of $f^k(\x)$ from the empirical mean of $f^k$ in the neighbourhood of $\x$. Assuming that the condition \eqref{eq:cond_phidev_allxi} holds for all $\xii \in \X$, we obtain
\begin{equation*}
\begin{split}
\left | f^k(\x) - \frac{1}{\Knb} \sum_{\x_j \in \nbdm} f^k(\x_j) \right | 
	&= \left |  \sum_{i=1}^\numsamp \cik
	\left( 
	 \phi(\| \x - \x_i  \|) 
		-   \frac{1}{\Knb}  \sum_{\x_j \in \nbdm} \phi(\| \x_j - \x_i  \|)   
	\right)
	\right | \\
&\leq
(\Lconf \delta + \epsilon)  \sum_{i=1}^\numsamp | \cik |	
\leq    \cbnd (\Lconf \delta + \epsilon),
\end{split}
\end{equation*}
which gives
\begin{equation*}
\begin{split}
 \| f(x) -  \frac{1}{\Knb} \sum_{\x_j \in \nbdm} f(\x_j)  \|
 = \left(
 \sum_{k=1}^d
 \bigg( f^k(\x) - \frac{1}{\Knb} \sum_{\x_j \in \nbdm} f^k(\x_j) \bigg)^2
 \right)^{1/2}
 \leq \sqrt{d}  \cbnd (\Lconf \delta + \epsilon).
\end{split}
\end{equation*}
We thus get
\begin{equation*}
P\left( \| f(x) -  \frac{1}{\Knb} \sum_{\x_j \in \nbdm} f(\x_j)  \|  \leq \sqrt{d}  \cbnd (\Lconf \delta + \epsilon) \right) 
\geq
1- 2 \numsamp \exp \left( - \frac{(\Knb-1) \, \epsilon^2 }{ 2 \Lconf^2 \delta^2} \right)
\end{equation*}
which completes the proof.
\end{proof}

\subsection{Proof of Theorem \ref{thm:acc_cl_rbfint}}
\label{pf:thm:acc_cl_rbfint}

\begin{proof} 
%
Remember from the proof of Theorem \ref{thm:linclassif_genf} that 
\begin{equation*}
P(| \nbdm | \geq Q) \geq 1 - \exp \left( \frac{-2 \, (\numsamp_m \, \minm - Q)^2 }{\numsamp_m} \right).
\end{equation*}

Lemma \ref{lem:bnd_dev_f_nborhd} states that, if $\Bx$ contains at least $Q$ samples from class $m$, i.e., $| \nbdm | \geq Q$, then
\begin{equation*}
\begin{split}
P\left( \| f(x) -  \frac{1}{|\nbdm|} \sum_{\x_j \in \nbdm} f(\x_j)  \|  \leq \sqrt{d}  \cbnd (\Lconf \delta + \epsilon) \right) 
&\geq
1- 2 \numsamp \exp \left( - \frac{(|\nbdm|-1) \, \epsilon^2 }{ 2 \Lconf^2 \delta^2} \right) 
\\
&\geq 
1- 2 \numsamp \exp \left( - \frac{(\Knb-1) \, \epsilon^2 }{ 2 \Lconf^2 \delta^2} \right).
\end{split}
\end{equation*}


Hence, combining these two results (multiplying both probabilities), we get that with probability at least 
\begin{equation*}
\begin{split}
\left( 
1 - \exp \left( \frac{-2 \, (\numsamp_m \, \minm - Q)^2 }{\numsamp_m} \right)
 \right) 
\left(
1- 2 \numsamp \exp \left( - \frac{(\Knb-1) \, \epsilon^2 }{ 2 \Lconf^2 \delta^2} \right)
\right) \\
\geq
1 - \exp \left( \frac{-2 \, (\numsamp_m \, \minm - Q)^2 }{\numsamp_m} \right)
 - 2 \numsamp \exp \left( - \frac{(\Knb-1) \, \epsilon^2 }{ 2 \Lconf^2 \delta^2} \right)
\end{split}
\end{equation*}
it holds that
\begin{equation}
\label{eq:ovrbnd_fdev_mean}
\| f(x) -  \frac{1}{|\nbdm|} \sum_{\x_j \in \nbdm} f(\x_j)  \|  \leq \sqrt{d}  \cbnd (\Lconf \delta + \epsilon).
\end{equation}

A test sample $\x$ drawn from $\pmes_m$ is classified correctly with the linear classifier if  
\begin{equation*}
\hyp_{mk}^T \, f(x) + b_{mk}  > 0,  \quad  \forall k=1, \dots, M-1, \, k\neq m.
\end{equation*}
If the condition in \eqref{eq:ovrbnd_fdev_mean} is satisfied, for all $k\neq m$, we have
\begin{equation*}
\begin{split}
\hyp_{mk}^T \, f(x) + b_{mk}  
&= \hyp_{mk}^T \frac{1}{| \nbdm |} 
\sum_{\x_j \in \nbdm} f(\x_j)  
+ b_{mk} 
+\hyp_{mk}^T \, \left( f(x) - \frac{1}{| \nbdm |} 
\sum_{\x_j \in \nbdm} f(\x_j) 
\right) \\   
&\geq
\hyp_{mk}^T \frac{1}{| \nbdm |} 
\sum_{\x_j \in \nbdm} f(\x_j)  
+ b_{mk}
- \|  f(x) - \frac{1}{| \nbdm |} 
\sum_{\x_j \in \nbdm} f(\x_j) 
  \| \\
&>
\mar_Q/2 - \|  f(x) - \frac{1}{| \nbdm |} 
\sum_{\x_j \in \nbdm} f(\x_j)  \|
\geq 
\mar_Q/2 -  \sqrt{d}  \cbnd (\Lconf \delta + \epsilon)
\geq 0.
\end{split}
\end{equation*}
%
We thus conclude that with probability at least 
\begin{equation*}
1 - \exp \left( \frac{-2 \, (\numsamp_m \, \minm - Q)^2 }{\numsamp_m} \right)
 - 2 \numsamp \exp \left( - \frac{(\Knb-1) \, \epsilon^2 }{ 2 \Lconf^2 \delta^2} \right)
\end{equation*}
$\hyp_{mk}^T \, f(x) + b_{mk}  > 0$ for all $k\neq m$, hence, the class label of $\x$ is estimated correctly.
\end{proof}

\subsection{Proof of Theorem \ref{thm:acc_cl_nn_rbfint}}
\label{pf:thm:acc_cl_nn_rbfint}

\begin{proof}
First, recall from the proof of Theorem \ref{thm:acc_cl_rbfint} that, with probability at least 
\begin{equation*}
1 - \exp \left( \frac{-2 \, (\numsamp_m \, \minm - Q)^2 }{\numsamp_m} \right)
 - 2 \numsamp \exp \left( - \frac{(\Knb-1) \, \epsilon^2 }{ 2 \Lconf^2 \delta^2} \right)
\end{equation*}
the $\delta$-neighborhood $\Bx$ of a test sample $x$ from class $m$ contains at least $Q$ samples from class $m$, and
\begin{equation}
\label{eq:recap_lemma_avg}
\| f(x) -  \frac{1}{|\nbdm|} \sum_{\x_j \in \nbdm} f(\x_j)  \|  \leq \sqrt{d}  \cbnd (\Lconf \delta + \epsilon)
\end{equation}
where $\nbdm$ is the set of training samples in $\Bx$ from class $m$. 


%
Then it is easy to show that (as in the proof of Theorem \ref{thm:error_genf_nnclass}), with probability at least
\begin{equation*}
1 - \exp \left( \frac{-2 \, (\numsamp_m \, \minm - Q)^2 }{\numsamp_m} \right)
 - 2 \numsamp \exp \left( - \frac{(\Knb-1) \, \epsilon^2 }{ 2 \Lconf^2 \delta^2} \right)
\end{equation*}
$\Bx$ will contain at least $Q$ samples $\x_i$ from class $m$ such that 
\begin{equation}
\label{eq:bnd_dist_to_clneig}
\|  f(\x) - f(\x_i)  \| \leq \sqrt{d}  \cbnd (\Lconf \delta + \epsilon) + \Demb_{2 \delta}.
\end{equation}

Hence, for a training sample  $\x_i'$ from another class (other than $m$), we have
\begin{equation*}
\| f(\x) - f(\x_i') \| \geq  \| f(\x_i)  - f(\x_i')  \|  - \|  f(\x) - f(\x_i)  \|  > \mar - (\sqrt{d}  \cbnd (\Lconf \delta + \epsilon) + \Demb_{2 \delta})
\end{equation*}
which follows from \eqref{eq:bnd_dist_to_clneig} and the hypothesis on the embedding that $ \| f(\x_i) - f(\x_i') \| > \mar$.

Due to the condition \eqref{eq:cond_sep_nn_interp}, we have $\mar \geq 2\sqrt{d} \, \cbnd \, (\Lconf \delta + \epsilon) +   2 \Demb_{2 \delta} $. Using this above equation, we obtain
\begin{equation*}
\| f(\x) - f(\x_i') \| >  \sqrt{d}  \cbnd (\Lconf \delta + \epsilon) + \Demb_{2 \delta} .
\end{equation*}
Therefore, the distance of $f(\x)$ to the embedding of the samples from other classes is more than $ \sqrt{d} \cbnd (\Lconf \delta + \epsilon) + \Demb_{2 \delta}$, while there are samples from its own class within a distance of  $ \sqrt{d} \cbnd (\Lconf \delta + \epsilon) + \Demb_{2 \delta}$ to $f(\x)$. We thus conclude that the class label of $x$ is estimated correctly with nearest-neighbor classification in the low-dimensional domain of embedding.
\end{proof}

\section{Proof of the results in Section \ref{sec:sep_analysis}}

\subsection{Proof of Lemma \ref{lem:twoclass_diffsign}}
\label{pf:lem_twoclass_diffsign}

\begin{proof}
The coordinate vector $\y$ is the eigenvector of the matrix $\Lw - \mu \Lb$ corresponding to its minimum eigenvalue. Hence,
\begin{equation*}
\y= \arg \min_{\begin{subarray}{c} \yv \\  \| \yv\|=1  \end{subarray} }
\yv^T ( \Lw - \mu \Lb) \yv. 
\end{equation*}
Equivalently, defining the degree-normalized coordinates $\z = \Dw^{-1/2} \y $, and thus replacing the above $\yv$ by $\Dw^{1/2} \yv$, we have 
%
%
%
%
\begin{equation}
\label{eq:z1d_probform}
\begin{split}
\z &=\arg \min_{\begin{subarray}{c} \yv \\   \yv^T \Dw \yv=1  \end{subarray}} 
N(\yv)\\
N(\yv)&= \yv^T \Dw^{1/2} ( \Lw - \mu \Lb) \Dw^{1/2} \yv \\
	 &= \yv^T (\Dw-\Ww) \yv - \mu \, \yv^T  (\Dw \Db^{-1})^{1/2} \, (\Db - \Wb) \, (\Db^{-1} \Dw)^{1/2} \yv.
\end{split}
\end{equation}
Then, denoting $\rdeg_i=\dw(i)/\db(i)$, the term $N(\yv)$ can be rearranged as
\begin{equation*}
\begin{split}
N(\yv) &= \sum_{\ii} \yv_i \bigg( \dw(i) \, \yv_i - \sum_{\jj \simw \ii}  \yv_j \, \wij \bigg)
	- \mu \sum_{\ii} \yv_i \bigg( \dw(i) \, \yv_i - \sum_{\jj \simb \ii}  \yv_j \, \wij \,\sqrt{\rdeg_i \rdeg_j}  \bigg) \\
	&=\sum_{\ii} \yv_i \sum_{\jj \simw \ii} (\yv_i- \yv_j) \wij
	- \mu \sum_{\ii} \yv_i \sum_{\jj \simb \ii} (\rdeg_i \yv_i- \sqrt{\rdeg_i \rdeg_j} \, \yv_j) \wij
\\
	&= \sum_{\ii} \sum_{\jj \simw \ii} (\yv_i^2- \yv_i \yv_j) \wij
	- \mu  \sum_{\ii} \sum_{\jj \simb \ii} ( \rdeg_i \yv_i^2-  \sqrt{\rdeg_i \rdeg_j} \yv_i \yv_j) \wij
\end{split}
\end{equation*}
which gives \footnote{In our notation, the terms $\ii \simw \jj$ and $\ii \simb \jj$ in the summation indices as in (\ref{defn:Nyv}) refer to edges rather than neighboring $(i,j)$-pairs; i.e., each pair is counted only once in the summation.}
\begin{equation}
\label{defn:Nyv}
N(\yv)=
\sum_{\ii \simw \jj} (\yv_i -\yv_j)^2 \, \wij  
- \mu \sum_{\ii \simb \jj} \left( \sqrt{\rdeg_i} \yv_i - \sqrt{\rdeg_j} \yv_j \right)^2 \, \wij 
\end{equation}
by grouping the neighboring $(i,j)$ pairs in the inner sums. Now, for any $\yv \in \R^{\numsamp \times 1}$ such that $\yv^T \Dw \yv=1 $, we define $\yv^*$ as follows.
\begin{equation}
\label{eq:form_mnmzer_N}
\yv_i^* = \bigg\{ 
\begin{array} {l}
 - | \yv_i | \, \, \text{    if   }  \classi=1 \\ 
\, \, \, \, \,  | \yv_i | \, \, \text{    if   }  \classi=2 \end{array}
\end{equation}
Clearly, $\yv^*$ also satisfies $(\yv^*)^T \Dw \yv^*=1 $. From (\ref{defn:Nyv}), it can be easily checked that $N(\yv^*) \leq N(\yv)$ for any $\yv$, Then, a  minimizer $\z$ of the problem (\ref{eq:z1d_probform}) has to be of the separable form defined in \eqref{eq:form_mnmzer_N}; otherwise $\z^*$ would yield a smaller value for the function $N$, which would contradict the fact that $\z$ is a minimizer. Note that the equality $N(\z^*) = N(\z)$ holds only if $\z=\z^*$ or $\z=-\z^*$, thus when $\z$ is separable. Therefore, the embedding $\z$ 
satisfies the condition
\begin{equation*}
\zi \leq 0 \, \, \text{    if   }  \classi=1, 
\qquad \qquad
\zi \geq 0 \, \, \text{    if   }  \classi=2
\end{equation*}
or the equivalent condition 
\begin{equation*}
\zi \leq 0 \, \, \text{    if   }  \classi=2,
\qquad \qquad
\zi \geq 0 \, \, \text{    if   }  \classi=1.
\end{equation*}
Finally, since $\yi=  \sqrt{\dw(i)} \, \zi$, the same property also holds for the embedding $\y$.


\end{proof}

\subsection{Proof of Theorem \ref{thm:sep2class}}
\label{pf:thm_sep2class}

\begin{proof}
From (\ref{eq:z1d_probform}) and (\ref{defn:Nyv}), we have
\begin{equation}
\label{eq:solnobj2class}
\z=\arg \min_{\begin{subarray}{c} \yv \\   \yv^T \Dw \yv=1  \end{subarray}} 
\sum_{\ii \simw \jj} (\yv_i -\yv_j)^2 \, \wij  
- \mu \sum_{\ii \simb \jj} \left( \sqrt{\rdeg_i} \yv_i - \sqrt{\rdeg_j} \yv_j \right)^2 \, \wij.
\end{equation}
Thus, at the optimal solution $\z$ the objective function takes the value
\begin{equation}
\label{eq:obj2class}
N(\z)= \sum_{\ii \simw \jj} (\z_i -\z_j)^2 \, \wij  
- \mu \sum_{\ii \simb \jj} \left( \sqrt{\rdeg_i} \z_i - \sqrt{\rdeg_j} \z_j \right)^2 \, \wij.
\end{equation}

In the following, we derive a lower bound for the first sum and an upper bound for the second sum in (\ref{eq:obj2class}). We begin with the first sum. Let $\ionemin$, $\ionemax$, $\itwomin$ and $\itwomax$ denote the indices of the data samples in class 1 and class 2 that are respectively mapped to the extremal coordinates $\zonemin$, $\zonemax$, $\ztwomin$, $\ztwomax$, where
\begin{equation*}
\z_{k, min} = \min_{i: \, \classi=k  } \zi \, , 
\qquad \qquad
\z_{k, max} = \max_{i: \, \classi=k  } \zi  \, .
\end{equation*}
Let $P_1=\{ (\x_{k_{i-1}}, \x_{k_{i}}) \}_{i=1}^{L_1}$ be a shortest path of length $L_1$ joining $\x_{\ionemin}$ and $\x_{\ionemax}$ and $P_2=\{ (\x_{n_{i-1}}, \x_{n_{i}}) \}_{i=1}^{L_2}$ be a shortest path of length $L_2$ joining $\x_{\itwomin}$ and $\x_{\itwomax}$. We have 
\begin{equation}
\label{eq:term1obj2class}
\begin{split}
\sum_{\ii \simw \jj} (\z_i -\z_j)^2 \, \wij  
\geq \sum_{i=1}^{L_1}  (\z_{k_i} -\z_{k_{i-1}})^2 \, \w_{k_{i-1} k_i}
+ \sum_{i=1}^{L_2}   (\z_{n_i} -\z_{n_{i-1}})^2 \, \w_{n_{i-1} n_i} \\
\geq \wminone \sum_{i=1}^{L_1} (\z_{k_i} -\z_{k_{i-1}})^2 \, 
+  \wmintwo  \sum_{i=1}^{L_2}  (\z_{n_i} -\z_{n_{i-1}})^2 \, 
\end{split}
\end{equation}
where the first inequality simply follows from the fact that the set of edges making up $P_1 \cup P_2$ are contained in the set of all edges in $\Gw$. For a sequence $\{a_i\}_{i=0}^L$, the following inequality holds.
\begin{equation*}
\begin{split}
(a_L - a_0)^2 &= \sum_{i=1}^{L} (a_i - a_{i-1})^2 + 
\sum_{\begin{subarray}{c} i,j=1 \\  i \neq j  \end{subarray}}^L
 (a_i - a_{i-1})  (a_j - a_{j-1})\\
 & \leq  \sum_{i=1}^{L} (a_i - a_{i-1})^2 + 
\half \sum_{\begin{subarray}{c} i,j=1 \\  i \neq j  \end{subarray}}^L
\big( (a_i - a_{i-1})^2 +  (a_j - a_{j-1})^2 \big)
 = L \sum_{i=1}^{L} (a_i - a_{i-1})^2 
 \end{split}
\end{equation*}
Hence, 
\begin{equation*}
 \sum_{i=1}^{L} (a_i - a_{i-1})^2  \geq \frac{1}{L} (a_L - a_0)^2.
\end{equation*}
Using this inequality in (\ref{eq:term1obj2class}), we get
\begin{equation*}
\begin{split}
\sum_{\ii \simw \jj} (\z_i -\z_j)^2 \, \wij  
\geq \frac{\wminone}{L_1} (\zonemax - \zonemin)^2 \, 
+  \frac{\wmintwo}{L_2}   (\ztwomax -\ztwomin)^2 . 
\end{split}
\end{equation*}
Since the path lengths $L_1$ and $L_2$ are upper bounded by the diameters $\diam_1$ and $\diam_2$, we finally obtain the lower bound
\begin{equation}
\label{eq:ub_term1_obj2class}
\sum_{\ii \simw \jj} (\z_i -\z_j)^2 \, \wij  
\geq \frac{\wminone}{\diam_1} (\zonemax - \zonemin)^2 \, 
+  \frac{\wmintwo}{\diam_2}   (\ztwomax -\ztwomin)^2 . 
\end{equation}
Next, we find an upper bound for the second sum in (\ref{eq:obj2class}). Using Lemma \ref{lem:twoclass_diffsign}, we obtain the following inequality. 
\begin{equation}
\label{eq:ub_term2_obj2class}
\begin{split}
\sum_{\ii \simb \jj} \left( \sqrt{\rdeg_i} \z_i - \sqrt{\rdeg_j} \z_j \right)^2 \, \wij
& \leq \sum_{\ii \simb \jj} (\ztwomax - \zonemin)^2 \, \rdegmax \, \wij \\
& = \half (\ztwomax - \zonemin)^2 \, \rdegmax \volbetmax
\end{split}
\end{equation}

Now, since the solution $\z$ in (\ref{eq:solnobj2class}) minimizes the objective function $N(\yv)$, we have
\begin{equation*}
N(\z) = \lambdamin(\Lw - \mu \Lb)
\end{equation*}
where $\lambdamin(\cdot)$ and $\lambdamax(\cdot)$ respectively denote the minimum and the maximum eigenvalues of a matrix. For two Hermitian matrices $A$ and $B$, the inequality $\lambdamin(A+B) \leq \lambdamin(A) + \lambdamax(B)$ holds. As $\Lw$ and $\Lb$ are graph Laplacian matrices, we have  $\lambdamin(\Lw) = \lambdamin(\Lb) = 0$ and thus
\begin{equation*}
N(\z) = \lambdamin(\Lw - \mu \Lb) \leq \lambdamin(\Lw) + \lambdamax(-\mu \Lb) = \lambdamin(\Lw) - \mu \lambdamin(\Lb) = 0.
\end{equation*}
Using in (\ref{eq:obj2class}) the above inequality and the lower and upper bounds in (\ref{eq:ub_term1_obj2class}) and (\ref{eq:ub_term2_obj2class}), we obtain
\begin{equation*}
\begin{split}
0 \geq 
N(\z) &= \sum_{\ii \simw \jj} (\z_i -\z_j)^2 \, \wij  
- \mu \sum_{\ii \simb \jj} \left( \sqrt{\rdeg_i} \z_i - \sqrt{\rdeg_j} \z_j \right)^2 \, \wij \\
& \geq  \frac{\wminone}{\diam_1} (\zonemax - \zonemin)^2 \,  
+  \frac{\wmintwo}{\diam_2}   (\ztwomax -\ztwomin)^2  \\
& \quad -   \half  \mu  (\ztwomax - \zonemin)^2 \, \rdegmax \volbetmax.
\end{split}
\end{equation*}
Hence
\begin{equation}
\label{eq:ineqsupp_2class}
\frac{\wminone}{\diam_1} (\zonemax - \zonemin)^2 \,  
+  \frac{\wmintwo}{\diam_2}   (\ztwomax -\ztwomin)^2  \leq  \half  \mu  (\ztwomax - \zonemin)^2 \, \rdegmax \volbetmax.
\end{equation}
The RHS of the above inequality is related to the overall support $\ztwomax - \zonemin$ of the data, whereas the terms on the LHS are related to the individual supports $\zonemax - \zonemin$ and $\ztwomax - \ztwomin$ of the two classes in the learnt embedding. Meanwhile, the separation $\ztwomin - \zonemax$ between the two classes is given by the gap between the overall support and the sums of the individual supports. In order to use the above inequality in view of this observation, we first derive a lower bound on the RHS term. Since $\z^T \Dw \z =1$, we have
\begin{equation*}
\begin{split}
1 = \sum_{i} \zi^2 \, \dw(i) 
& = \sum_{i: \, \classi =1 } \zi^2 \, \dw(i)  
+ \sum_{i: \, \classi =2 } \zi^2 \, \dw(i) \\
& \leq \zonemin^2 \sum_{i: \, \classi =1 } \dw(i)  + \ztwomax^2 \sum_{i: \, \classi =2 }  \dw(i) 
 = \zonemin^2 \vol_1 + \ztwomax^2 \vol_2.
\end{split}
\end{equation*}
This gives 
\begin{equation*}
\zonemin^2  + \ztwomax^2  \geq \frac{1}{\volmax}.
\end{equation*}
Hence, we obtain the following lower bound on the overall support
\begin{equation}
\label{eq:lb_suppovr_2class}
(\ztwomax - \zonemin)^2 \geq \ztwomax^2 + \zonemin^2 \geq \frac{1}{\volmax}.
\end{equation}
Denoting the supports of class 1 and class 2 and the overall support as
\begin{equation*}
\suppone = \zonemax - \zonemin,
\qquad \qquad
\supptwo = \ztwomax - \ztwomin,
\qquad \qquad
\suppovr = \ztwomax - \zonemin,
\end{equation*}
we have from (\ref{eq:ineqsupp_2class})
\begin{equation*}
\wminbyD (\suppone^2 + \supptwo^2)
 \leq  \half \, \mu \, \suppovr^2 \, \rdegmax \volbetmax
\end{equation*}
which yields the following upper bound on the total support of the two classes
\begin{equation*}
\suppone + \supptwo \leq \sqrt{ 2(\suppone^2 + \supptwo^2) } 
\leq \suppovr \sqrt{ \frac{ \mu \rdegmax \volbetmax }{\wminbyD}}.
\end{equation*}
We can thus lower bound the separation $\ztwomin - \zonemax$ as
\begin{equation*}
\ztwomin - \zonemax = \suppovr - (\suppone + \supptwo)
\geq \suppovr \left(  1 -  \sqrt{ \frac{ \mu \rdegmax \volbetmax }{\wminbyD}}  \right)
\end{equation*}
provided that $\mu < \wminbyD / (\rdegmax \volbetmax)$. From the lower bound on the overall support in (\ref{eq:lb_suppovr_2class}), we lower bound the separation as follows
\begin{equation*}
\ztwomin - \zonemax  \geq  \frac{1}{\sqrt{\volmax}} \left(  1 -  \sqrt{ \frac{ \mu \rdegmax \volbetmax }{\wminbyD}}  \right).
\end{equation*}
Finally, since the separation of any embedding with dimension $d\geq 1$ is at least as much as the separation $\ztwomin - \zonemax$ of the embedding of dimension $d=1$, the above lower bound holds for any $d \geq 1$ as well.
\end{proof}

\subsection{Proof of Corollary \ref{cor:dist_orig_2class}}
\label{pf:cor_dist_orig_2class}

\begin{proof}
The one-dimensional embedding $\z$ is given as the solution of the constrained optimization problem 
\begin{equation*}
\z =\arg \min N(\yv)  \text{ s.t. } D(\yv)=1
\end{equation*}
where
\begin{equation*}
N(\yv)= \yv^T \Dw^{1/2} ( \Lw - \mu \Lb) \Dw^{1/2} \yv, 
\qquad
D(\yv)= \yv^T \Dw \yv.
\end{equation*}
Defining the Lagrangian function
\begin{equation*}
\Lambda(\yv, \lambda) = N(\yv) + \lambda (D(\yv) -1)
\end{equation*}
at the optimal solution $\z$, we have
\begin{equation*}
\nabla_{\yv} \Lambda = \nabla_{\lambda} \Lambda = 0
\end{equation*}
where $\nabla_{\yv}$ and $\nabla_{\lambda} $ respectively denote the derivatives of $\Lambda$ with respect to $\yv$ and $\lambda$. Thus, at $\yv = z$,
\begin{equation*}
\frac{\partial \Lambda}{ \partial \yv_\ii} 
= \frac{\partial N(\yv)}{ \partial \yv_\ii} +  \lambda \frac{\partial D(\yv)}{ \partial \yv_\ii} =0
\end{equation*}
for all $\ii=1, \dots, \numsamp$. From (\ref{defn:Nyv}), the derivatives of $N(\yv)$ and $D(\yv)$ at $\z$ are given by
\begin{equation*}
\begin{split}
\frac{\partial N(\yv)}{ \partial \yv_\ii} \bigg |_{\yv=\z} 
&= \sum_{\jj \simw \ii} 2 (\zi -\zj)  \wij 
- \mu \sum_{\jj  \simb \ii } 2 \left( \sqrt{\rdeg_i} \zi - \sqrt{\rdeg_j} \zj \right) \,  \sqrt{\rdeg_i} \, \wij \\
\frac{\partial D(\yv)}{ \partial \yv_\ii} \bigg |_{\yv=\z} &= 2  \, \dw(i) \, \zi
\end{split}
\end{equation*}
which yields 
\begin{equation}
\label{eq:lag_eq_z_2class}
\sum_{\jj \simw \ii}  (\zi -\zj)  \wij 
- \mu \sum_{\jj  \simb \ii }  \left( \sqrt{\rdeg_i} \zi - \sqrt{\rdeg_j} \zj \right) \,  \sqrt{\rdeg_i} \, \wij 
+ \lambda  \, \dw(i) \, \zi = 0
\end{equation}
for all $\ii$. At $\ii = \ionemax$, as $\z$ attains its maximal value $\zonemax$ for class 1, we have 
\begin{equation*}
\begin{split}
\lambda  \, \dw(\ionemax) \, \zonemax &= \sum_{\jj \simw \ionemax}  (\zj -\zonemax)  \w_{\ionemax j} \\
& \qquad + \mu \sum_{\jj  \simb \ionemax}  \left( \sqrt{\rdeg_{\ionemax}} \zonemax - \sqrt{\rdeg_j} \zj \right)  \sqrt{\rdeg_{\ionemax}} \, \w_{\ionemax j}  \\
&\leq - \mu \, \rdegmin \, (\ztwomin - \zonemax) \, \db(\ionemax).
\end{split}
\end{equation*}
Hence
\begin{equation}
\label{eq:lb_zonemax_v1}
| \zonemax | = - \zonemax
 \geq \frac{\mu \, \rdegmin \, (\ztwomin - \zonemax) \db(\ionemax)}{\lambda  \, \dw(\ionemax) }
\geq \frac{\mu \, \rdegmin \, (\ztwomin - \zonemax) }{\lambda  \, \rdegmax }.
\end{equation}
We proceed by deriving an upper bound for $\lambda$. The gradients of $N(\yv)$ and $D(\yv)$ are given by
\begin{equation*}
\nabla_{\yv} N= 2 \Dw^{1/2} ( \Lw - \mu \Lb) \Dw^{1/2} \yv, 
\qquad
\nabla_{\yv} D = 2 \Dw \yv.
\end{equation*}
 From the condition $\nabla_{\yv} \Lambda=0$ at $\yv = \z$, we have
\begin{equation*}
\begin{split}
\Dw^{1/2} ( \Lw - \mu \Lb) \Dw^{1/2} \z + \lambda  \Dw \z &=0 \\
( \Lw - \mu \Lb) \y  + \lambda \y &= 0. 
\end{split}
\end{equation*}
Since $\y = \Dw^{1/2} \z$ is the unit-norm eigenvector of $ \Lw - \mu \Lb$ corresponding to its smallest eigenvalue, the Lagrangian multiplier $\lambda$ is given by
\begin{equation*}
\lambda = - \lambdamin( \Lw - \mu \Lb).
\end{equation*}
We can lower bound the minimum eigenvalue as
\begin{equation*}
\lambdamin( \Lw - \mu \Lb) \geq  \lambdamin(\Lw) + \lambdamin(-\mu \Lb)
= 0 - \mu \lambdamax(\Lb) \geq - 2 \mu
\end{equation*}
since the eigenvalues of a graph Laplacian are upper bounded by $2$. This gives $\lambda \leq 2\mu$. Using this upper bound on $\lambda$ in (\ref{eq:lb_zonemax_v1}), we obtain 
\begin{equation*}
| \zonemax | 
\geq \half \frac{ \rdegmin} {  \rdegmax }  \, (\ztwomin - \zonemax) .
\end{equation*}
Repeating the same steps for $i=\itwomin$ following (\ref{eq:lag_eq_z_2class}), one can similarly show that
\begin{equation*}
\ztwomin
\geq \half \frac{ \rdegmin} {  \rdegmax }  \, (\ztwomin - \zonemax) .
\end{equation*}

\end{proof}


\subsection{Proof of Theorem \ref{thm:sep_mult_class_categ}}
\label{pf:thm:sep_mult_class_categ}

We first present two lemmas that will be useful for proving Theorem \ref{thm:sep_mult_class_categ}. 

\begin{lemma}
\label{lem:corr_eigvec_pert}

Let $A \in \R^{\numsamp \times \numsamp}$ be a symmetric matrix with eigenvalue decomposition $A = U \Lambda U^T$, where $U$ is an orthogonal matrix and $\Lambda$ is a diagonal matrix consisting of the eigenvalues $\lambda_1, \dots, \lambda_\numsamp$. Consider a symmetric perturbation $\Delta A$ on $A$. Let the perturbed matrix $\tA = A + \Delta A$ have the eigenvalue  decomposition $\tA = \tU \tLambda \tU^T$. 

Assume that the eigenvalues $\lambda_i$ have a separation of at least $\eta$, i.e., for all distinct $i, j$, one has $| \lambda_i - \lambda_j | \geq \eta$. Then the inner products of the corresponding eigenvectors of $A$ and $\tilde A$ are lower bounded as
\begin{equation*}
| \tilde u_j^T u_j | \geq \sqrt{1 - \frac{4 \, \| \Delta A \|^2}{\eta^2} }
\end{equation*}
for all $j=1, \dots, \numsamp$, where $u_j$ denotes the $j$-th column of $U$.
\end{lemma}

\begin{proof}
Defining $R=\tU^T U$, we look for a lower bound on the diagonal entries of $R$. It will be helpful to examine the term
\begin{equation}
\label{eq:noncommut_lamR}
\| \Lambda R - R \Lambda \| = \|  \tLambda R -  (\Delta \Lambda) R - R \Lambda \|
		\leq  \|  \tLambda R -  R \Lambda  \|  +  \|  \Delta \Lambda  \|
		\leq  \|  \tLambda R -  R \Lambda  \|  +  \|  \Delta A  \|
\end{equation}
where $\Delta \Lambda = \tLambda - \Lambda$ and the last inequality follows from the fact that the variation in the eigenvalues is upper bounded by the norm of the perturbation matrix.

We proceed by bounding the term $ \|  \tLambda R -  R \Lambda  \|$. First observe that
\begin{equation*}
(\Delta A) U = (\tA - A) U = (\tU \tLambda \tU^T - U \Lambda U^T) U 
		= \tU \tLambda R - U \Lambda  
\end{equation*}
which gives
\begin{equation*}
\| \Delta A  \| = \| \tU \tLambda R - U \Lambda  \| 
		= \| \tU^T (\tU \tLambda R - U \Lambda)  \|
		= \| \tLambda R  - R \Lambda \|.
\end{equation*}
Using this in \eqref{eq:noncommut_lamR}, we get
\begin{equation}
\label{eq:bnd_deltaR_noncom}
\| \Lambda R - R \Lambda \|  \leq 2 \| \Delta A  \| .
\end{equation}

Since each column of $R$ is given by the projection of a unit norm vector on an orthogonal basis, it is unit-norm. Denoting the $(i,j)$-th entry of $R$ by $R_{ij}$, we have
\begin{equation}
\label{eq:R_jj_form}
| R_{jj} |= ( 1- \sum_{i \neq j} R_{ij}^2 )^{1/2}.
\end{equation}
We proceed by bounding the sum of the entries $R_{ij}^2$. Notice that the  $(i,j)$-th entry of $\Lambda R - R \Lambda$ is given by $(\lambda_i - \lambda_j) R_{ij}$. For each $j$,
\begin{equation*}
\sum_{i \neq j} (\lambda_i - \lambda_j)^2 R_{ij}^2 \leq  \|  \Lambda R - R \Lambda \|^2 \leq 4 \| \Delta A \|^2
\end{equation*}
where the first inequality follows from the fact that the norm of the $j$-th column of a matrix can be upper bounded by its operator norm, and the second inequality is due to \eqref{eq:bnd_deltaR_noncom}. Due to the eigenvalue separation hypothesis, the first term above can be lower bounded as
\begin{equation*}
\sum_{i \neq j} (\lambda_i - \lambda_j)^2 R_{ij}^2 \geq  \eta^2  \sum_{i \neq j}  R_{ij}^2
\end{equation*}
which gives
\begin{equation*}
 \sum_{i \neq j}  R_{ij}^2 \leq  \frac{4 \| \Delta A \|^2}{ \eta^2}.
\end{equation*}
From \eqref{eq:R_jj_form}, we arrive at the stated result, i.e., for each $j$
\begin{equation*}
| \tilde u_j^T u_j | = | R_{jj} |= ( 1- \sum_{i \neq j} R_{ij}^2 )^{1/2} \geq   \left( 1- \frac{4 \| \Delta A \|^2}{ \eta^2} \right)^{1/2}.
\end{equation*}
\end{proof}

\begin{lemma}
\label{lem:pert_row_col}

Let $U, \tU \in \R^{\numsamp \times \numsamp}$ be two orthogonal matrices such that the difference between the corresponding columns $u_i$, $\tilde u_i$ of $U$ and $\tU$ are upper bounded as $\| u_i - \tilde u_i \|^2 \leq \delta$ for some $\delta < 2$. Let $V=U^T$ and $\tV=\tU^T$. Then the difference between the corresponding columns $v_i$, $\tilde v_i$ of $V$ and $\tV$ are upper bounded as 
\begin{equation*}
\| v_i - \tilde v_i \|^2 \leq \delta + 2 \sqrt{\numsamp} \sqrt{1- \left(1-\frac{\delta}{2} \right)^2}.
\end{equation*}
\end{lemma}

\begin{proof}
Let $R= \tU^T U$. Since $u_i$ and $\tu_i$ are unit-norm vectors, we have
\begin{equation*}
\| u_i - \tu_i \|^2 = 2 - 2   \tu_i^T u_i  \leq \delta
\end{equation*}
therefore
\begin{equation}
\label{eq:R_ii_bnd}
R_{ii} =  \tu_i^T u_i   \geq 1 - \frac{ \delta}{2} > 0
\end{equation}
where $R_{ii}$ denotes the $i$-th diagonal entry of $R$. From $\tv_i = R v_i$ it follows that
\begin{equation}
\label{eq:vi_vitilde}
\begin{split}
v_i^T \tv_i = v_i^T R v_i  = v_i^T R^d v_i + v_i^T R^{nd} v_i
	\geq  v_i^T R^d v_i  - | v_i^T R^{nd} v_i |
\end{split}
\end{equation}
%
%
where $R^d$ and $R^{nd}$ denote the components of $R$ consisting respectively of the diagonal and the nondiagonal terms. From the condition \eqref{eq:R_ii_bnd} on the diagonal entries of $R$, it follows that the first term is lower bounded as
\begin{equation*}
v_i^T R^d v_i \geq 1 - \delta/2.
\end{equation*}
Also, from \eqref{eq:R_ii_bnd}, the $\ell_2$-norm of each row and each column of $R^{nd}$ is upper bounded by $\sqrt{1-(1-\delta/2)^2}$.  Bounding the operator norm of $R^{nd}$ in terms of the maximal $\ell_1$-norms of the rows and columns, we get 
\begin{equation*}
| v_i^T R^{nd} v_i  | \leq \| R^{nd} \|  \leq \sqrt{\numsamp} \sqrt{1-(1-\delta/2)^2}.
\end{equation*}
Using this together with the inequality \eqref{eq:R_ii_bnd} in \eqref{eq:vi_vitilde}, we get 
\begin{equation*}
v_i^T \tv_i  \geq 1- \frac{\delta}{2} - \sqrt{\numsamp} \sqrt{1-(1-\delta/2)^2}
\end{equation*}
which gives the stated result
\begin{equation*}
\| v_i - \tilde v_i  \|^2 = 2 - 2 v_i^T \tilde v_i \leq \delta + 2 \sqrt{N}  \sqrt{1-(1-\delta/2)^2}.
\end{equation*}
\end{proof}

We are now ready to prove Theorem \ref{thm:sep_mult_class_categ}.

\begin{proof}
We first look at the separation of the embedding obtained with $\Lcat$ for the reduced data graph with all between-category edges removed. The data graph corresponding to $\Lcat$ has $\numcat$ connected components; therefore, $\Lcat$ is a block diagonal matrix consisting of $\numcat$ blocks. Each $\icat$-th block is given by the objective matrix $\Licat = \Lwicat - \mu \Lbicat $ where $\Lwicat$ and $\Lbicat$ are the within-class and the between-class Laplacian matrices of the data graph restricted to only the category $\icat$. As $\Lcat$ is a block-diagonal matrix, its eigenvalues and eigenvectors are given by the union of the eigenvalues  and the eigenvectors of the block components $\Licat$ (i.e., their inclusions in $\R^\numsamp$ by zero-padding).

Let $\Y^\icat = [ \y_1^\icat  \dots \y_{\numsamp_\icat}^\icat]^T$ be the $d^\icat$-dimensional embedding of the $\numsamp_\icat$ samples in category $\icat$, whose columns are the eigenvectors of $\Licat$. The embedding $\Y^\icat $ is assumed to be separable with a margin of $\marc$ by the theorem hypothesis. Consider the embeddings $\Y^\icat$, $\Y^\jcat$ of two different categories $\icat$ and $\jcat$, and two classes $k$, $l$ respectively from these two categories. By the separation hypothesis  \eqref{eq:cond_sep_nooff} within each category, there exist hyperplanes $\hyp_k^\icat$ and $\hyp_l^\jcat$ with $\| \hyp_k^\icat \| =1$,  $\| \hyp_l^\jcat \| =1$, such that for the embedding of any sample $\y_i^\icat \in \R^{d^\icat}$ from class $k$ and any sample $\y_j^\jcat \in  \R^{d^\jcat}$ from class $l$ it holds that
\begin{equation}
\label{eq:sep_hyp_kl_ij}
\begin{split}
(\hyp_k^\icat)^T \y_i^\icat &\geq \marc/2 \\
(\hyp_l^\jcat)^T \y_j^\jcat & \leq -\marc/2.
\end{split}
\end{equation}
Now considering an ordering of all $\numcat$ categories, we can define the inclusion $ \overline{\y}_i^\icat \in \R^d $ of each sample ${\y}_i^\icat \in \R^{d^\icat}$ from each category $\icat$, where $d=\sum_\icat d^\icat$ and the nonzero entries of $ \overline{\y}_i^\icat = [0 \dots 0 \ ({\y}_i^\icat)^T \ 0 \dots 0]^T$ are located at the support of category $\icat$. Note that each $ ( \overline{\y}_i^\icat )^T$ corresponds to a row of the coordinate matrix $\Ycat $, whose columns are the eigenvectors of $\Lcat$.

Consider the hyperplane 
\begin{equation*}
\hyp_{k,l}^{\icat,\jcat} = \frac{1}{\sqrt{2}} [ 0 \dots 0 \ (\hyp_k^\icat)^T \ 0 \dots 0 \ (\hyp_l^\jcat)^T  ]^T \in \R^d
\end{equation*}
with $\| \hyp_{k,l}^{\icat,\jcat} \|=1$, formed by the inclusion of  $(\hyp_k^\icat)^T$ and $(\hyp_l^\jcat)^T$ in $\R^d$ over the entries corresponding respectively to the categories $\icat$ and $\jcat$. From \eqref{eq:sep_hyp_kl_ij}, we get that the hyperplane $\hyp_{k,l}^{\icat,\jcat}$ separates these two classes as
\begin{equation}
\label{eq:sep_Lc_catform}
\begin{split}
(\hyp_{k,l}^{\icat,\jcat})^T  \ \overline{\y}_i^\icat  &\geq \frac{ \marc}{2\sqrt{2}} \\
(\hyp_{k,l}^{\icat,\jcat})^T  \ \overline{\y}_j^\jcat  & \leq - \frac{ \marc}{2\sqrt{2}} .
\end{split}
\end{equation}
This shows that there exists a $d$-dimensional embedding given by the eigenvectors of $\Lcat$ that separates any pair of classes with a margin of at least $\marc/\sqrt{2}$.

Now observe from Lemma \ref{lem:corr_eigvec_pert} that the correlation between the $i$-th eigenvector $u^c_i$ of $\Lcat$ and the corresponding eigenvector $u_i$ of $\Lall$ is upper bounded as
\begin{equation*}
|  u_i^T u^c_i | \geq \xi= \sqrt{1 - \frac{4 \| \PerL \|^2}{\eta^2}}.
\end{equation*}
This implies either of the conditions $u_i^T u^c_i \geq \xi$ or $u_i^T (-u^c_i) \geq \xi$. Therefore, the eigenvector $u_i$ of the perturbed objective matrix $\Lall$ has a correlation of at least $\xi$ with either $u^c_i$ or its opposite $-u^c_i$. Meanwhile, the separability of an embedding is invariant to the negation of one of the eigenvectors. This corresponds simply to changing the sign of one of the coordinates of all data samples (i.e., taking the symmetric of the embedding with respect to one axis); therefore, the linear separability remains the same. For this reason, it suffices to treat the case $u_i^T u^c_i \geq \xi$ for analyzing the separability without loss of generality.

The condition $u_i^T u^c_i \geq \xi$ implies 
\begin{equation}
\label{eq:eigv_pert_norsq}
\| u_i - u^c_i \|^2 = 2 - 2 u_i^T u^c_i  \leq 2 - 2 \xi.
\end{equation}
While this upper bounds the difference between the corresponding eigenvectors of $\Lall$ and $\Lcat$, we need to upper bound the variation between the rows of $\Lall$ and $\Lcat$, as we are interested in the separation obtained with the embedded data coordinates given by the rows of $\Lall$. Denoting the $i$-th rows of $\Lall$ and $\Lcat$ respectively as $\y_i^T$ and $\overline{\y}_i^T$, from the condition in \eqref{eq:eigv_pert_norsq} and Lemma \ref{lem:pert_row_col}, the difference between the corresponding rows of these matrices can be bounded as
\begin{equation}
\label{eq:bnd_yi_yibar}
\|  \y_i^T - \overline{\y}_i^T \|^2 \leq 2 - 2\xi + 2 \sqrt{\numsamp (1-\xi^2) }. 
\end{equation}
As the separability condition in \eqref{eq:sep_Lc_catform} is general and valid for any two categories, we can reformulate it as follows. For any pair of classes $k,l \in \{ 1, \dots, \numclass \}$, there exists a hyperplane $\hyp_{k,l}$ such that
\begin{equation}
\label{eq:sep_Lc_genform}
\begin{split}
\hyp_{k,l}^T  \ \overline{\y}_i  &\geq \frac{ \marc}{2\sqrt{2}}  \quad \,\,\,\, \text{     if  } \classi=k\\ 
\hyp_{k,l}^T  \ \overline{\y}_i  & \leq - \frac{ \marc}{2\sqrt{2}}  \quad \,\,\,\, \text{     if  } \classi=l.
\end{split}
\end{equation}
Then, from \eqref{eq:bnd_yi_yibar} and \eqref{eq:sep_Lc_genform} we have
\begin{equation*}
\hyp_{k,l}^T \y_i = \hyp_{k,l}^T \overline{\y_i} + \hyp_{k,l}^T (\y_i - \overline{\y_i}) 
	\geq  \hyp_{k,l}^T \overline{\y_i}  - \| \y_i - \overline{\y_i} \|
	\geq   \frac{ \marc}{2\sqrt{2}}  -  \left( 2 - 2\xi + 2 \sqrt{\numsamp (1-\xi^2) } \right)^{1/2}
\end{equation*}
if $\classi=k$; and 
\begin{equation*}
\hyp_{k,l}^T \y_i = \hyp_{k,l}^T \overline{\y_i} + \hyp_{k,l}^T (\y_i - \overline{\y_i}) 
	\leq  \hyp_{k,l}^T \overline{\y_i}  + \| \y_i - \overline{\y_i} \|
	\leq  -  \frac{ \marc}{2\sqrt{2}}  + \left( 2 - 2\xi + 2 \sqrt{\numsamp (1-\xi^2) } \right)^{1/2}
\end{equation*}
if $\classi = l$. Hence, the embedding $\Y$ given by the eigenvectors of the overall objective matrix $\Lall$ is linearly separable with a margin of 
\begin{equation*}
\mar = \frac{ \marc}{\sqrt{2}}  -  2 \left( 2 - 2\xi + 2 \sqrt{\numsamp (1-\xi^2) } \right)^{1/2}
\end{equation*}
if
\begin{equation*}
\marc > 4 \left( 1 - \xi +  \sqrt{\numsamp (1-\xi^2) } \right)^{1/2}.
\end{equation*}
We thus arrive at the result stated in the theorem.

\end{proof}

\vskip 0.2in
\bibliography{refs}

\end{document}